\documentclass[a4paper, 10.5pt]{article}

\pdfoutput=1

\usepackage{amsmath}
\usepackage{amsfonts}
\usepackage{amssymb}
\usepackage{mathtools}
\usepackage{amsthm}
\usepackage{bm}
\usepackage{hyperref}
\usepackage{titlesec}
\usepackage{subfigure}
\usepackage{array}
\usepackage{graphicx,epstopdf}
\usepackage{authblk}
\usepackage{fullpage}
\usepackage{setspace}
\usepackage{algorithm}
\usepackage{algpseudocode}
\usepackage{url}

\ifpdf
  \DeclareGraphicsExtensions{,.pdf,.png,.jpg}
\else
  \DeclareGraphicsExtensions{}
\fi

\newcommand{\tensor}[1]{\mathbf{#1}}
\newcommand{\mexp}[1]{\mathbb{E}\left\{#1\right\}}
\newcommand{\trans}{\top}
\newcommand{\cov}{\text{Cov}}

\newcommand{\half}{\frac{1}{2}}

\numberwithin{equation}{section}

\newtheorem{thm}{Theorem}[section]
\newtheorem{lem}[thm]{Lemma}

\makeatletter
\def\mathcenterto#1#2{\mathclap{\phantom{#1}\mathclap{#2}}\phantom{#1}}
\let\old@widetilde\widetilde
\def\widetildeto#1#2{\mathcenterto{#2}{\old@widetilde{\mathcenterto{#1}{#2\,}}}}
\makeatother

\newcommand{\wtl}{\widetildeto{I}{L}}

\newcommand{\wtX}{\widetildeto{X}{\bm X}}
\newcommand{\wty}{\widetildeto{y}{\bm y}}
\newcommand{\wtm}{\widetildeto{y}{\bm\mu}}
\newcommand{\wtc}{\widetildeto{y}{\bm c}}
\newcommand{\wtC}{\widetildeto{C}{\tensor C}}

\newcommand{\argmax}{\operatornamewithlimits{argmax}}

\newcommand{\PreserveBackslash}[1]{\let\temp=\\#1\let\\=\temp}
\newcolumntype{C}[1]{>{\PreserveBackslash\centering}p{#1}}

\title{When Bifidelity Meets CoKriging: An Efficient Physics-Informed Multifidelity Method}
\date{}

\author[1]{Xiu Yang\footnote{xiu.yang@pnnl.gov}}
\author[2]{Xueyu Zhu\footnote{xueyu-zhu@uiowa.edu}}
\author[1]{Jing Li\footnote{jing.li@pnnl.gov}}
\affil[1]{Advanced Computing, Mathematics and Data Division, Pacific Northwest
  National Laboratory, Richland, WA 99352}
\affil[2]{Department of Mathematics, The University of Iowa, Iowa City, IA 52242}

\allowdisplaybreaks[1]

\begin{document}
\maketitle

\begin{abstract}
In this work, we propose a framework that combines the approximation-theory-based
multifidelity method and Gaussian-process-regression-based multifidelity
method to achieve data-model convergence when stochastic simulation models and
sparse accurate observation data are available. Specifically, the two types of 
multifidelity methods we use are the \emph{bifidelity} and \emph{CoKriging} 
methods. The new approach uses the bifidelity method to efficiently
estimate the empirical mean and covariance of the stochastic simulation outputs,
then it uses these statistics to construct a Gaussian process (GP) representing 
low-fidelity in CoKriging. We also combine the bifidelity method with Kriging,
where the approximated empirical statistics are used to construct the GP as 
well. We prove that the resulting posterior mean by the new physics-informed 
approach preserves linear physical constraints up to an error bound. 
By using this method, we can obtain an accurate construction of a state of
interest based on a partially correct physical model and a few accurate
observations. We present numerical examples to demonstrate performance of the method. 

\end{abstract}


\section{Introduction}
\label{sec:intro}

Gaussian process (GP), a widely used tool in statistics, and machine learning
~\cite{forrester2008engineering,sacks1989design,stein2012interpolation}, has
become popular in probabilistic scientific computing.
GP regression (GPR), also known as \emph{Kriging} in geostatistics, constructs a
statistical model of a partially observed process by assuming that its observations
are a realization of a GP. A GP is uniquely described by its mean and
covariance function (also known as \emph{kernel}). 
Its variant, CoKriging, was originally formulated to compute predictions of 
sparsely observed states of physical systems by leveraging observations of other
states or parameters of the system~\cite{stein1991universal,knotters1995comparison}.
Recently, it has been employed for constructing multi-fidelity 
models~\cite{kennedy2000predicting,le2014recursive,perdikaris2015multi}, and has
been applied in various fields, 
e.g.,~\cite{laurenceau2008building,brooks2011multi,pan2017optimizing}. In the
widely used stationary Kriging/CoKriging method, usually parameterized forms of
mean and covariance functions are assumed, and the hyperparameters of these 
functions (e.g., variance and correlation length) are estimated 
by  maximizing the log marginal likelihood of the data.

Recently a new framework, physics-informed Kriging (PhIK)/physics-informed
CoKriging (CoPhIK), was developed for those applications where partially
correct physical models along with sparse observation data are
available~\cite{YangTT18, YangBTT18}. These physical models are constructed 
based on domain knowledge, and they include random variables or random fields to
represent the lack of knowledge (e.g., unknown physical law, uncertain
parameters, etc). The PhIK/CoPhIK framework combines the realizations of the 
stochastic physical model with the observation data to provide an accurate 
reconstruct of the state of interest on the entire computational domain. These
realizations are then used to approximate mean and covariance in the PhIK/CoPhIK
framework. The most popular approach of obtaining model realizations is the 
Monte Carlo (MC) simulation, and it can be replaced by more efficient approaches
to estimate mean and covariance, e.g., quasi-Monte 
Carlo~\cite{niederreiter1992random}, probabilistic 
collocation~\cite{tatang1995direct,XiuH05}, Analysis Of Variance
(ANOVA)~\cite{ma2010adaptive, YangCLK12}, compressive 
sensing~\cite{doostan2011non, YangK13}, and the moment equation 
method~\cite{Tart2017WRR}.

It is worthy to note that the aforementioned approaches
rely on single a fidelity solver. For large scale applications, the computational
cost can still be prohibitive if a large number of samples are required. 
In many practical problems, low-fidelity models for the underlying problem are 
often available, and it is much less expensive to obtain the realizations of these
models. Even though their accuracy is not high, they can still capture some
important physics of the underlying models with low computational cost. 
Therefore, it is highly desirable to utilize the computational efficiency of
low-fidelity models to reduce the overall computational cost. Many multifidelity
algorithms have been developed based on different principles in different 
contexts. These include (a) multi-level Monte Carlo 
(MLMC)~\cite{giles2008multilevel,cliffe2011multilevel}, which is already used in
PhIK and CoPhIK to reduce the computational cost~\cite{YangTT18, YangBTT18}; (b) meta-models through
GP, i.e., CoKriging~\cite{kennedy2000predicting,tuo2014surrogate,perdikaris2017nonlinear};
(c) variance reduced based approaches, i.e., control-variate based 
approach~\cite{peherstorfer2016optimal}, importance 
sampling~\cite{peherstorfer2016multifidelity}; and (d) model discrepancy based 
approaches~\cite{ng2012multifidelity,eldred2017multifidelity}. Another trend  
in the context of uncertainty quantification is to explore the parameter space
by using a large number of low-fidelity samples to identify a small set of important
basis, then learn the ``best" approximation rule of the target high-fidelity 
solution or its statistics based on the selected 
basis~\cite{narayan2014stochastic,zhu2014computational,zhu2017multi,fairbanks2017low}.
Previous works demonstrated its potential to significantly reduce the
computational cost for various applications by utilizing $\mathcal{O}(10)$ 
high-fidelity simulations, including combustion 
modeling~\cite{munipalli2018multifidelity}, orbit-state uncertainty 
propagation~\cite{jones2018multi}, molecular dynamics 
simulations~\cite{razi2018fast}, and turbulence modeling~\cite{jofre2018multi}, 
to name a few.  

In this work, we propose to employ the multifidelity approaches presented 
in~\cite{narayan2014stochastic,zhu2014computational} to accelerate the 
computation of PhIK and CoPhIK. For demonstration purpose, we consider the
bifidelity model in this work, and the proposed framework can be implemented in
multifidelity (more than two fidelity) models.


\section{Methodology}
\label{sec:method}

In this section, we begin by reviewing the general GPR 
framework \cite{williams2006gaussian}, the Kriging and CoKriging methods with
stationary kernel \cite{forrester2008engineering}, the PhIK \cite{YangTT18} and
CoPhIK \cite{YangBTT18} methods, and the bi-fidelity approach \cite{zhu2014computational}. 
Then, we introduce the bi-fidelity-aided PhIK and CoPhIK methods.

\subsection{GPR framework}
\label{subsec:gpr}

We denote the observation locations as $\bm X = \{\bm x^{(i)}\}_{i=1}^N$
($\bm x^{(i)}$ are $d$-dimensional vectors in $D\subseteq\mathbb{R}^d$) and the 
observed state values at these locations as 
$\bm y=(y^{(1)}, y^{(2)},\dotsc, y^{(N)})^\trans$
($y^{(i)}\in\mathbb{R}$). For simplicity, we assume that $y^{(i)}$ are scalars. 
We aim to predict $y$ at any new location $\bm x^*\in D$. The GPR method
assumes that the observation vector $\bm y$ is a realization of the following 
$N$-dimensional random vector that satisfies multivariate Gaussian distribution:
\[
\bm Y = \left(Y(\bm x^{(1)}), Y(\bm x^{(2)}), \dotsc, Y(\bm x^{(N)})\right)^\trans,
\]
where $Y(\bm x^{(i)})$ is the concise notation of $Y(\bm x^{(i)}; \omega)$, 
and $Y(\bm x^{(i)}; \omega)$ is a Gaussian random variable defined on a 
probability space $(\Omega, \mathcal{F}, P)$ with $\omega\in\Omega$.
Of note, $\bm x^{(i)}$ can be considered as parameters for the GP
$Y(\cdot,\cdot):D\times\Omega\rightarrow\mathbb{R}$, such that 
$Y(\bm x^{(i)},\cdot):\Omega\rightarrow \mathbb{R}$ is a Gaussian random 
variable for any $\bm x^{(i)}$ in the set $D$. Usually, the GP $Y(\bm x)$ is 
denoted as
\begin{equation}
  \label{eq:gp0}
Y(\bm x) \sim \mathcal{GP}\left(\mu(\bm x), k(\bm x, \bm x')\right),
\end{equation}
where $\mu(\cdot):D\rightarrow\mathbb{R}$ and 
$k(\cdot,\cdot):D\times D\rightarrow\mathbb{R}$ are the mean 
and covariance functions:
  \begin{align}
    \mu(\bm x) & = \mexp{Y(\bm x)},\\ 
    k(\bm x,\bm x') & = \cov\left\{Y(\bm x), Y(\bm x')\right\}
                      = \mexp{(Y(\bm x)-\mu(\bm x))(Y(\bm x')-\mu(\bm x'))}.
  \end{align}
The variance of $Y(\bm x)$ is $k(\bm x, \bm x)$, and its standard deviation is
$\sigma(\bm x)=\sqrt{k(\bm x,\bm x)}$. 
The covariance matrix of random vector $\bm Y$, denoted as $\tensor C$, is 
defined as $C_{ij}=k(\bm x^{(i)}, \bm x^{(j)})$.
Functions $\mu(\bm x)$ and $k(\bm x,\bm x')$ are obtained by identifying their
hyperparameters via maximizing the log marginal 
likelihood~\cite{williams2006gaussian}:
\begin{equation}
  \label{eq:lml}
  \ln L=-\dfrac{1}{2}(\bm y-\bm\mu)^\trans \tensor C^{-1} (\bm
  y-\bm\mu)-\dfrac{1}{2}\ln |\tensor C|-\dfrac{N}{2}\ln 2\pi,
\end{equation}
where $\bm\mu=(\mu(\bm x^{(1)}),\dotsc,\bm x^{(N)})^\trans$. The result of GPR 
is a posterior distribution 
$y(\bm x^*)\sim\mathcal{N}(\hat y(\bm x^*), \hat s^2(\bm x^*))$ for any $\bm
x^*\in D$, where
  \begin{align}
\label{eq:krig}
\hat y(\bm x^*) & = \mu(\bm x^*) + 
  \bm c(\bm x^*)^\trans\tensor{C}^{-1}(\bm y-\bm\mu),  \\
  \hat s^2(\bm x^*) & = 
   \sigma^2(\bm x^*)-\bm c(\bm x^*)^\trans \tensor{C}^{-1}\bm c(\bm x^*),
 \end{align}
and $\bm c(\bm x^*)$ is a vector of covariance, i.e.,
$(\bm c(\bm x^*))_i=k(\bm x^{(i)},\bm x^*)$.
In practice, it is common to use $\hat y(\bm x^*)$ as the prediction, and
$\hat s^2(\bm x^*)$ is also called the mean squared error (MSE) of the prediction
because $\hat s^2(\bm x^*)=\mexp{(\hat y(\bm x^*)-Y(\bm x^*))^2}$
\cite{forrester2008engineering}. Consequently, $\hat s(\bm x^*)$ is the root
mean squared error (RMSE). 
Moreover, to account for the observation noise, one can assume that the noise is
independent and identically distributed (i.i.d.) Gaussian random variables with
zero mean and variance $\delta^2$, and replace $\tensor C$ with 
$\tensor C+\delta^2\tensor I$. In this study, we assume that observations 
$\bm y$ are noiseless. If $\tensor C$ is not invertible or its condition number
is very large, one can add a small regularization term $\alpha\tensor I$
($\alpha$ is a small positive real number) to $\tensor C$, 
which is equivalent to assuming there
is an observation noise. In addition, $\hat s$ can be used in global
optimization, or in the greedy algorithm to identify locations of additional
observations. Specifically, in the greedy algorithm, the new observations can be
added at the maxima of $\hat s$, see Appendix~\ref{sec:append_act} for details.


\subsection{Kriging and CoKriging with stationary kernel}
\label{subsec:stationary_gpr}
In the widely used ordinary Kriging method, a stationary GP is assumed
\cite{kitanidis1997introduction}. Specifically, $\mu$ is set as a constant 
$\mu(\bm x)\equiv\mu$, and $k(\bm x, \bm x')=k(\bm\tau)$, where 
$\bm\tau=\bm x-\bm x'$. Consequently, 
$\sigma^2(\bm x)=k(\bm x,\bm x)=k(\bm 0)=\sigma^2$ is a constant.
Popular forms of kernels include polynomial, exponential, Gaussian
(squared-exponential), and Mat\'{e}rn functions. For example, the Gaussian 
kernel can be written as
$k(\bm\tau)=\sigma^2\exp\left(-\frac{1}{2}\Vert \bm x-\bm x'\Vert^2_w\right)$,
where the weighted norm is defined as 
$\displaystyle\Vert \bm x-\bm x'\Vert^2_w=\sum_{i=1}^d \left(\dfrac{x_i-x'_i}{l_i}\right)^2$. 
Here, $l_i$ ($i=1,\dotsc, d$), the correlation lengths of $\bm y$ in
the $i$ direction, are constants. Given a stationary covariance function, the
covariance matrix $\tensor C$ of $\bm Y$ can be written as
$\tensor C=\sigma^2\tensor\Psi$, where
$\Psi_{ij}=\exp(-\frac{1}{2}\Vert\bm x^{(i)}-\bm x^{(j)}\Vert_w^2)$. In the MLE
framework, the estimators of $\mu$ and $\sigma^2$, denoted as $\hat\mu$ and
$\hat\sigma^2$, are
\begin{equation}
\label{eq:krig_est}
   \hat\mu=\dfrac{\bm 1^\trans \tensor\Psi^{-1}\bm y}
                  {\bm 1^\trans\tensor\Psi^{-1}\bm 1}, \qquad
   \hat\sigma^2=\dfrac{(\bm y-\bm 1\hat\mu)^\trans\tensor\Psi^{-1}(\bm y-\bm 1\hat\mu)}{N},
\end{equation}
where $\bm 1$ is a constant vector consisting of 
$1$~\cite{forrester2008engineering}. The hyperparameters $\sigma$ and $l_i$ are 
estimated by maximizing the log marginal likelihood in Eq.~\eqref{eq:lml}. The 
$\hat y(\bm x^*)$ and $\hat s^2(\bm x^*)$ in Eq.~\eqref{eq:krig} take the 
following form:
\begin{align}
\label{eq:stat_krig}
\hat y(\bm x^*) & = \hat\mu + \bm\psi^\trans\tensor\Psi^{-1}(\bm y-\bm 1\hat\mu), \\
\hat s^2(\bm x^*) & = \hat\sigma^2\left(1-\psi^\trans\tensor\Psi^{-1}\bm\psi\right),
\end{align}
where $\bm\psi=\bm\psi(\bm x^*)$ is a vector of correlations between the observed data and the
prediction, i.e., $\psi_i=\frac{1}{\sigma^2}k(\bm x^{(i)},\bm x^*)$.

Next, we briefly review the formulation of CoKriging for two-level multifidelity 
modeling. Suppose that we have high-fidelity data 
(e.g., accurate measurements of states) 
$\bm y_{H}=\left(y_{H}^{(1)}, \dotsc,y_{H}^{(N_H)}\right)^\trans$ at 
locations $\bm X_{H} = \left\{\bm x_{H}^{(i)}\right\}_{i=1}^{N_H}$, and 
low-fidelity data (e.g., simulation results) 
$\bm y_{L}=\left(y_{L}^{(1)}, \dotsc,y_{L}^{(N_L)}\right)^\trans$ at 
locations $\bm X_{L} = \left\{\bm x_{L}^{(i)}\right\}_{i=1}^{N_L}$, where 
$y_{H}^{(i)}, y_{L}^{(i)}\in\mathbb{R}$ and 
$\bm x_{H}^{(i)}, \bm x_{L}^{(i)}\in D\subseteq\mathbb{R}^d$.
We denote $\wtX=\{\bm X_{L}, \bm X_{H}\}$ and 
$\wty = \left(\bm y_{L}^{\trans}, \bm y_{H}^{\trans} \right)^{\trans}$. 

Kennedy and O'Hagan \cite{kennedy2000predicting} proposed a multifidelity 
formulation based on the auto-regressive model for GP $Y_{H}$
($\sim\mathcal{GP}(\mu_{H}(\cdot), k_{H}(\cdot,\cdot))$):
\begin{equation}
  \label{eq:arm}
  Y_{H}(\bm x) = \rho Y_{L}(\bm x) + Y_{d}(\bm x),
\end{equation}
where $Y_{L}(\cdot)$ ($\sim\mathcal{GP}(\mu_{L}(\cdot), k_{L}(\cdot, \cdot))$) 
regresses the low-fidelity data, $\rho \in \mathbb{R}$ is a
regression parameter and $Y_d(\cdot)$ 
($\sim\mathcal{GP}(\mu_{d}(\cdot), k_{d}(\cdot,\cdot))$ models the difference 
between $Y_{H}$ and $\rho Y_{L}$. This model assumes that
\begin{equation}
  \label{eq:arm-cov}
  \cov \left\{Y_{H}(\bm x), Y_{L}({\bm x}') \mid Y_{L}(\bm x) \right\}=0,
  \quad \text{for all}\quad \bm x' \neq \bm x, \ \bm x, \bm x' \in D.
\end{equation}
The covariance of observations, $\wtC$, is then given by
\begin{equation}
  \label{eq:cokrig_cov}
  \wtC=
  \begin{pmatrix}
    \tensor C_{L}(\bm X_{L}, \bm X_{L}) & \rho \tensor C_{L}(\bm X_{L}, \bm X_{H}) \\
    \rho \tensor C_{L} (\bm X_{H}, \bm X_{L}) & \rho^2 \tensor C_{L}(\bm
    X_{H}, \bm X_{H}) + \tensor C_{d}(\bm X_{H}, \bm X_{H})
  \end{pmatrix},
\end{equation}
where $\tensor C_{L}$ and $\tensor C_{d}$ are the covariance matrices computed
from $k_{L}(\cdot,\cdot)$ and $k_{d}(\cdot,\cdot)$, respectively. One can 
assume parameterized forms for these kernels (e.g., Gaussian kernel) and employ 
the following two-step approach~\cite{forrester2007multi,
forrester2008engineering} to identify hyperparameters:
\begin{enumerate}
\item Use Kriging to construct $Y_{L}$ using  $\{\bm X_{L}, \bm y_{L}\}$. 
\item Denote $\bm y_{d} = \bm y_{H} - \rho \bm y_{L}(\bm X_{H})$, where 
$\bm y_{L}(\bm X_{H})$ are the values of $\bm y_{L}$ at locations common to
those of $\bm X_{H}$, then construct $Y_d$ using $\{\bm X_{H}, \bm y_{d}\}$ 
via Kriging. 
\end{enumerate}
The posterior mean and variance of $Y_H$ at $\bm x^*\in D$ are given by
\begin{align}
  \label{eq:cokrig_mean}
  \hat y(\bm x^*) & =  \mu_{H}(\bm x^*) + 
      \wtc(\bm x^*)^\trans\wtC^{-1}(\wty-\wtm),  \\
  \label{eq:cokrig_var}
      \hat s^2(\bm x^*) & = \rho^2\sigma^2_{L}(\bm x^*) + \sigma^2_{d}(\bm x^*)
      - \wtc(\bm x^*)^\trans \wtC^{-1}\wtc(\bm x^*),
\end{align}
where $\mu_{H}(\bm x^*)=\rho\mu_{L}(\bm x^*)+\mu_d(\bm x^*)$, 
$\sigma^2_{L}(\bm x^*)=k_{_{L}}(\bm x^*, \bm x^*)$,
$\sigma^2_d(\bm x^*)=k_d(\bm x^*, \bm x^*)$, and
\begin{align}
  \label{eq:cokrig_mu}
\wtm & = \begin{pmatrix}\bm\mu_{L}\\ \bm\mu_{H}\end{pmatrix}
  = \begin{pmatrix}\big(\mu_{L}(\bm x_{_{L}}^{(1)})\dotsc, 
    \mu_{L}(\bm x_{_{L}}^{(N_{L})})\big)^\trans \\
     \big(\mu_{H}(\bm x_{_{H}}^{(1)})\dotsc, 
\mu_{H}(\bm x_{_{H}}^{(N_{H})})\big)^\trans\end{pmatrix},  \\
  \label{eq:cokrig_c}
  \wtc(\bm x^*)& = \begin{pmatrix}\rho\bm c_{L}(\bm x^*)\\ \bm c_{H}(\bm
x^*)\end{pmatrix}
= \begin{pmatrix}\big(\rho k_{L}(\bm x^*,\bm x_{L}^{(1)}),\dotsc,
  \rho k_{_{L}}(\bm x^*,\bm x_{L}^{(N_{L})})\big)^\trans \\
  \big(k_{H}(\bm x^*,\bm x_{H}^{(1)}),\dotsc,k_{H}(\bm x^*,\bm x_{H}^{(N_{H})})
\big)^\trans\end{pmatrix}, 
\end{align}
where
$k_{H}(\bm x,\bm x') = \rho^2k_{L}(\bm x,\bm x') + k_{d}(\bm x, \bm x')$.
Here, we have neglected a small contribution to $\hat s^2$
(see \cite{forrester2008engineering}).
Alternatively, one can simultaneously identify hyperparameters in
$k_{L}(\cdot,\cdot)$ and $k_{d}(\cdot,\cdot)$ along with $\rho$ by maximizing 
the following log marginal likelihood: 
\begin{equation}
  \label{eq:lml2}
  \ln \wtl=-\dfrac{1}{2}(\wty-\wtm)^\trans\wtC^{-1} (\wty
  -\wtm)-\dfrac{1}{2}\ln \big\vert\wtC\big\vert-\dfrac{N_{H}+N_{L}}{2}\ln 2\pi.
\end{equation}


\subsection{PhIK and CoPhIK}
\label{subsec:cophik}
The recently proposed PhIK method \cite{YangTT18} takes advantage of the 
existing domain knowledge, e.g., approximate numerical or analytical 
physics-based models, in the form of realizations of a stochastic model of the
system. Consequently, the mean and covariance of the GP model can be 
approximated using these realizations. As such, there is no need to assume a
specific form of the correlation functions and solve an optimization problem for
the hyperparameters. These stochastic models typically include random parameters
or random processes/fields to reflect the lack of understanding (of physical 
laws) or knowledge (of the coefficients, parameters, etc.) of the real system.
Then, MC simulations can be conducted to generate an ensemble of the state of
interest, from which the statistics, e.g., mean and standard deviation, are
estimated. 

Specifically, assume that we have $M$ realizations of a stochastic model 
$u(\bm x;\omega) $ ($\bm x \in D, \omega\in\Omega$) denoted as 
$\{u^m(\bm x)\}_{m=1}^M$, and we can build the following GP model:
\begin{equation} 
  Y(\bm x) \sim \mathcal{GP}(\mu_{\mathrm{MC}}(\bm x), k_{\mathrm{MC}}(\bm x, \bm x')),
\end{equation} 
where
\begin{equation}
\begin{aligned}
  \label{eq:phik}
  \mu(\bm x) & \approx \mu_{\mathrm{MC}}(\bm x)=\dfrac{1}{M}\sum_{m=1}^M u^m(\bm x), \\
  k(\bm x, \bm x') &\approx k_{\mathrm{MC}}(\bm x, \bm x')=\dfrac{1}{M-1} \sum_{m=1}^M
  \left(u^m(\bm x)-\mu_{\mathrm{MC}}(\bm x)\right) \left(u^m(\bm x')-\mu_{\mathrm{MC}}(\bm x')\right).
\end{aligned}
\end{equation}
Thus, the covariance matrix of $\bm Y$ can be estimated as
\begin{equation}
  \label{eq:mc_cov}
  \tensor C\approx \tensor C_{\mathrm{MC}}=\dfrac{1}{M-1} \sum_{m=1}^M
  \left(\bm u^m-\bm\mu_{\mathrm{MC}} \right)
  \left(\bm u^m-\bm\mu_{\mathrm{MC}} \right)^\trans,
\end{equation}
where $\bm u^m=\left(u^m(\bm x^{(1)}), \dotsc, u^m(\bm x^{(N)})\right)^\trans$, 
$\bm\mu_{\mathrm{MC}}=\left(\mu_{\mathrm{MC}}(\bm x^{(1)}), \dotsc, \mu_{\mathrm{MC}}(\bm x^{(N)})\right)^\trans$.
The prediction and MSE at location $\bm x^*\in D$ are
\begin{align}
\label{eq:krig_pred2}
\hat y(\bm x^*) & = \mu_{\mathrm{MC}}(\bm x^*) + 
\bm c_{\mathrm{MC}}(\bm x^*)^\trans\tensor{C}_{\mathrm{MC}}^{-1}(\bm y-\bm\mu_{\mathrm{MC}}),  \\
\hat s^2(\bm x^*) & = \hat\sigma_{\mathrm{MC}}^2(\bm x^*)-\bm c_{\mathrm{MC}}(\bm x^*)^\trans 
\tensor{C}_{\mathrm{MC}}^{-1}\bm c_{\mathrm{MC}}(\bm x^*),
\end{align}
where $\hat\sigma^2_{\mathrm{MC}}(\bm x^*)=k_{\mathrm{MC}}(\bm x^*,\bm x^*)$ is
the variance of data set $\{u^m(\bm x^*)\}_{m=1}^M$, and  \\
$\bm c_{\mathrm{MC}}(\bm x^*)=\left(k_{\mathrm{MC}}(\bm x^{(1)},\bm x^*), 
\dotsc,k_{\mathrm{MC}}(\bm x^{(N)},\bm x^*)\right)^\trans$.

Similarly, the CoPhIK method uses model realizations to construct $Y_{L}$ in
CoKriging. Specifically, we set $\bm X_{L}=\bm X_{H}$ to simplify the formula
and computing, and denote $N=N_{H}=N_{L}$. We set 
$\mu_{L}(\bm x)=\mu_{\mathrm{MC}}(\bm x)$ and 
$k_{L}(\bm x,\bm x')=k_{\mathrm{MC}}(\bm x,\bm x')$ to construct $Y_{L}$,
where $\mu_{\mathrm{MC}}$ and $k_{\mathrm{MC}}$ are given by 
Eq.~\eqref{eq:phik}. The GP model $Y_d$ is constructed using the same approach
as in the second step of the Kennedy and O'Hagan CoKriging framework.
Specifically, we set $\bm y_{d}=\bm y_{H}-\rho\bm\mu_{L}(\bm X_{L})$. The reason
for this choice is that $\mu_{L}(\bm X_{H})$ is the most probable observation of
the GP $Y_{L}$. Next, we need to assume a specific form of the kernel function. 
Without loss of generality, we use the stationary Gaussian kernel model and 
constant $\mu_{d}$. Once $\bm y_{d}$ is computed, and the form of 
$\mu_{d}(\cdot)$ and $k_{d}(\cdot,\cdot)$ are decided, $Y_{d}$ can be 
constructed as in ordinary Kriging. Now that all components of $\ln\wtl$ in
Eq.~\eqref{eq:lml2} are specified except for the $\bm y_{L}$ in $\wty$. We set 
$\bm y_{L}$ as the realization from the ensemble $\{u^m(\bm x)\}_{m=1}^M$ that 
maximizes $\ln\wtl$. The algorithm is summarized in Algorithm~\ref{algo:phicok}.
\begin{algorithm}[!h]
  \caption{CoPhIK using stochastic simulation model $u(\bm x;\omega)$ on
  $D\times\Omega$ ($D\subseteq\mathbb{R}^d$), and
  high-fidelity observation $\bm y_{H}=(y_{H}^{(1)}, \dotsc
  y_{_{H}}^{(N)})^\trans$ at locations 
  $\bm X_{H}=\{\bm x_{H}^{(i)}\}_{i=1}^{N}$.}
  \label{algo:phicok}
  \begin{algorithmic}[1]
    \State Conduct stochastic simulation, e.g., MC simulation, using 
    $u(\bm x;\omega)$ to generate realizations $\{u^m(\bm x)\}_{m=1}^M$ on the
    domain $D$.
    \State Use PhIK  to construct GP $Y_{L}$ on $D\times\Omega$, i.e.,
    $\mu_{L}(\cdot)=\mu_{\mathrm{MC}}(\cdot)$ and
    $k_{L}(\cdot,\cdot)=k_{\mathrm{MC}}(\cdot,\cdot)$ in Eq.~\eqref{eq:phik}.
    Compute $\mu_{L}(\bm X_{L})=\big(\mu_{L}(\bm x_{L}^{(1)}),\dotsc
    \mu_{L}(\bm x_{L}^{(N)})\big)^\trans$, and 
    $\tensor C_{L}(\bm X_{L}, \bm X_{L})$ whose $ij$-th element is 
    $k_{L}(\bm x_{H}^{(i)},\bm x_{H}^{(j)})$.
    Set $\tensor C_{L}(\bm X_{L}, \bm X_{H})=\tensor C_{L}(\bm X_{H},
    \bm X_{L})=\tensor C_{L}(\bm X_{H}, \bm X_{H})=\tensor C_{L}(\bm
    X_{L}, \bm X_{L})$ (because $\bm X_{L}=\bm X_{H}$). 
    \State Denote $\bm y_d=\bm y_{H}-\rho\mu_{L}(\bm X_{L})$, choose a specific
    kernel function $k_d(\cdot,\cdot)$ (Gaussian kernel in this work) for the GP
    $Y_d$, and identify hyperparameters via maximizing the log marginal 
    likelihood Eq.~\eqref{eq:lml}, where $\bm y, \bm\mu, \tensor C$ are
    specified as $\bm y_{d}, \bm\mu_{d}, \tensor C_{d}$, respectively. Then
    construct $\wtm$ in Eq.~\eqref{eq:cokrig_mu}, and $\tensor C_{d}$ whose 
    $ij$-th element is $k_d(\bm x_{H}^{(i)}, \bm x_{H}^{(j)})$.
    \State Iterate over the set $\{u^m(\bm x)\}_{m=1}^M$ to identify $u^m(\bm x)$
    that maximizes $\ln\tilde L$ in Eq.~\eqref{eq:lml2}, where 
    $\bm y_{L}=(u^m(\bm x_{_{H}}^{(1)}), \dotsc, u^m(\bm x_{_{H}}^{(N)}))^{\trans}$
    is used in $\wty$.
    \State Compute the posterior mean using Eq.~\eqref{eq:cokrig_mean}, and
    variance using Eq.~\eqref{eq:cokrig_var} for any $\bm x^*\in D$.
  \end{algorithmic}
\end{algorithm}


It was demonstrated that PhIK prediction on the entire domain $D$ 
preserves the \emph{linear} physical constraints up to an error bound that relies
on the numerical error, discrepancy between the physical model and real system,
and the smallest eigenvalue of matrix $\tensor C$ \cite{YangTT18}. For example, 
the deterministic periodic, Dirichelet or Neumann boundary condition can be preserved.
Another type of example is the linear derivative operator, e.g., 
$\mathcal{L}u=\nabla^2 u$. If $u$ satisfies $\nabla^2 u(\bm x;\omega)=0$ for any
$\omega\in\Omega$, e.g., $u$ is the velocity potential, $\hat y(\bm x)$ from 
PhIK also guarantees a divergence-free flow field. CoPhIK has the potential to
improve the accuracy of the prediction, namely resulting in a smaller discrepancy
between posterior mean and the exact solution because CoPhIK
incorporates observations in constructing the GP model, while PhIK only uses
model simulations. On the other hand, CoPhIK result may violate some physical
constraints because of the choice of $Y_{d}$ kernel \cite{YangBTT18}.


\subsection{Bi-fidelity approximation}

In this section, we present the bi-fidelity \\ method, and related error 
estimates when it is combined with PhIK and CoPhIK.

\subsubsection{Algorithm}

We briefly describe the bi-fidelity 
method~\cite{narayan2014stochastic,zhu2014computational,zhu2017multi,munipalli2018multifidelity}. 
We slightly modify the notation $u(\bm x;\omega)$ as $u(\bm x;z(\omega))$ to 
denote the stochastic model used in PhIK and CoPhIK. Here, $z(\omega)$ is the 
finite-dimensional random variable or field included in the model, and we denote it as 
$z$ for simplicity. Subsequently, $z^m=z(\omega^m)$ is a sample $z$, and we 
assume that $z^m\in I_z$ for any $m$. Let $u_{H}^m(\bm x)$ denote the high-fidelity
simulation result for $u(\bm x;z^m)$, e.g., simulation using fine grids or 
high-order scheme, and $u_{L}^m(\bm x)$ denote the low-fidelity simulation 
result for $u(\bm x;z^m)$, e.g., simulation using coarse grids or low-order 
scheme. In PhIK, the mean and covariance functions of GP $Y(\bm x)$ are 
approximated using $u^m(\bm x)$ (see Eq.~\eqref{eq:phik}), which are
$u^m_H(\bm x)$ in the bi-fidelity framework, and so as GP $Y_{L}(\bm x)$ in the
CoPhIK method. Now, we aim to approximate these mean and covariance functions 
using a few $u_{H}^m(\bm x)$ and a large number of $u_{L}^m(\bm x)$, as such
to reduce the computational cost of model simulations. 
More specifically, the bi-fidelity approach aims to approximate mean and
covariance functions in PhIK and CoPhIK functions using 
$\{u_{H}^m(\bm x)\}_{m=1}^{M_{H}}$ along with 
$\{u_{L}^m(\bm x)\}_{m=1}^{M_{L}}$, where $M_{H}\ll M_{L}$, and $M_L$ is $M$
in the original PhIK and CoPhIK formulation.

For brevity, we denote $u^m_{L}(\bm x)$ and $u^m_{H}(\bm x)$ as $u^m_{L}$ and 
$u^m_{H}$, respectively. We also assume that $u_{H}^m\in \mathbb{V}_{H}$ and 
$u_{L}^m\in \mathbb{V}_{L}$, where $\mathbb{V}_{H}$ and $\mathbb{V}_{L}$ are
Hilbert spaces with inner product $\langle\cdot,\cdot\rangle_{H}$ and 
$\langle\cdot,\cdot\rangle_{L}$, respectively. Given a collection of parametric 
sample $\Gamma=\{z^1,\dotsc, z^{M_L}\}\subset I_z$, we introduce the following 
notations: 
\begin{equation}
  \begin{aligned}
   u_{L}(\Gamma) &= \left\{u_{L}^m\right\}_{m=1}^{M_{L}}, \quad
 U_{L}(\Gamma) =\text{span}\, u_{L}(\Gamma) 
       =\text{span}\, \left\{u_{L}^1, \dotsc, u_{L}^{M_L} \right\}, \\
   u_{H}(\Gamma) &= \left\{u_{H}^m\right\}_{m=1}^{M_{L}}, \quad
 U_{H}(\Gamma) =\text{span}\, u_{H}(\Gamma) 
       =\text{span}\, \left\{u_{H}^1, \dotsc, u_{H}^{M_L} \right\}.
  \end{aligned}
\end{equation}
The procedure of the bi-fidelity algorithm for estimating mean and covariance 
functions is presented in Algorithm~\ref{algo:bifi}.
\begin{algorithm}[!h]
 \caption{Bi-fidelity method of computing mean and covariance functions in PhIK and
  CoPhIK.}
\begin{algorithmic}[1]
\State Conduct the low-fidelity simulations at the sample set $\Gamma$
  to obtain realizations $u_{L}(\Gamma)$.
\State Select a subset of samples $\gamma=\{z^{i_1}, \dotsc, z^{i_{M_H}}\} \subset \Gamma$,
  where $M_H\ll M_L$.
\State Conduct high-fidelity computations at the subset $\gamma$
  and obtain the high-fidelity simulation samples, 
  $u_{H}(\gamma) = \big\{u_{H}^{i_1},\dotsc, u_{H}^{i_{M_H}}\big\}$.
\State Construct $u_{B}(\Gamma)=\{u_{B}^1, \dotsc, u_{B}^{M_L}\}$ 
  based on $u_{H}(\gamma)$.
\State Compute mean and covariance functions using $u_B(\Gamma)$.
\end{algorithmic}
\label{algo:bifi}
\end{algorithm}

We detail the steps 2 and 4\ in Algorithm~\ref{algo:bifi} as follows
\cite{narayan2014stochastic,zhu2014computational}. 
Let $\tensor{W}$ be the Gramian matrix of the low-fidelity simulations
$u_{L}(\Gamma)$, i.e.,
\begin{equation} \label{Gramian}
w_{ij} = \langle u_{L}^i, u_{L}^j\rangle_{L},\qquad 
1\leq i,j\leq M_{L}.
\end{equation}
Applying the pivoted Cholesky decomposition to the matrix $\tensor{W}$ yields
\begin{equation} \label{eq:chol}
\tensor{W} = \tensor{P}^\trans\tensor{LL}^\trans\tensor{P},
\end{equation}
where $\tensor{L}$ is lower-triangular and $\tensor{P}$ is a permutation matrix
due to pivoting. This will produce an ordered permutation vector 
$\bm p=(i_1,\dots,i_{M_L})$, from which we choose the first $M_H$ points to 
define $\gamma=\{z^{i_1},\dots,z^{i_{M_H}}\}$. In practice, we can use a greedy 
algorithm to identify $\gamma$. In each iteration, we find the next sample whose 
corresponding low-fidelity simulation is furthest to the space spanned by the 
existing low-fidelity simulation set. Specifically, starting from a trivial 
initial choice $\gamma_0 = \emptyset$, we let 
$\gamma_k=\{z_{i_1},\dots,z_{i_k}\}\subset \Gamma$ be the $k$-point existing
subset in $\Gamma$. We then find the $(k+1)$-th point by
\begin{equation}
  z^{i_{k+1}} = \argmax_{z \in \Gamma} \textrm{dist}(u_{L}(\bm x;z),U_{L}(\gamma_k)), 
  \qquad \gamma_{k+1} = \gamma_{k} \cup z^{i_{k+1}}, 
 \end{equation} 
where,
$U_{L}(\gamma_k) =\text{span}\, u_{L}(\gamma_k) 
=\text{span}\, \left\{u_{L}^{i_1}, \dotsc, u_{L}^{i_k}\right\}$,
the distance function $\textrm{dist}(g, G)$ between the function
$g\in u_{L}(\Gamma)$ and the space $G\subset U_{L}(\Gamma)$ follows the 
standard definition. This greedy algorithm can be readily implemented via simple 
operations of numerical linear algebra. More details and properties of the 
algorithm can be found in \cite{narayan2014stochastic, zhu2014computational}.

Once the steps 1-3 are accomplished in Algorithm~\ref{algo:bifi}, we have sample
set $\Gamma$, low-fidelity simulations $u_{L}(\Gamma)$, the subset samples 
$\gamma\subset \Gamma$, and high-fidelity simulations $u_{H}(\gamma)$. The next
step is a lifting procedure: we use the best approximation rule of $u_{L}$ we learned on $\mathbb{V}_{L}$ 
to construct an interpolation operator, then apply it to $\mathbb{V}_{H}$. 
Therefore, $u_{L}(\gamma)=\big\{u_{L}^{i_1},\dotsc, u_{L}^{i_{M_H}}\big\}$ forms
a linearly independent set. The convergence of the bifidelity method with 
respect to $M_H$ is investigated in~\cite{narayan2014stochastic}. In this work,
we set the threshold to be $10^{-12}$ in the greedy algorithm presented 
in~\cite{zhu2014computational} such that $\text{span}~u_L(\gamma)$ is almost the
same as $\text{span}~u_L(\Gamma)$.

For any $v\in U_{L}(\Gamma)\subset\mathbb{V}_{L}$, the
projection of $v$ on $\text{span}\, u_{L}(\gamma)$ is
\begin{equation}
  \mathcal{P}_{L} v = \sum_{j=1}^{M_{H}} c_j u_{L}^{i_j}. 
\end{equation}
Then, we define the interpolation operator $\mathcal{I}_L^H(\gamma,
v):\mathbb{V}_{L}\rightarrow U_{H}(\gamma)$ as
\begin{equation}
 \mathcal{I}_L^H(\gamma, v) = \sum_{j=1}^{M_H} c_j u_{H}^{i_j}, 
 \quad v\in \mathbb{V}_L,
\end{equation}
where $U_{H}(\gamma)=\text{span}\,u_{H}(\gamma)$.
Subsequently, we construct $u_B({\Gamma})$ as 
\begin{equation}
u_B^m=\mathcal{I}_L^H(\gamma, u_L^m), \quad m = 1, \dotsc, M_L.
\end{equation}
The mean and covariance in PhIK are approximated by replacing $u^m$, where 
$u^m=u_H^m$, in Eq.~\eqref{eq:phik} with $u^m_B$, and set $M=M_L$. The GP $Y_L$
in CoPhIK is constructed in the same manner. We name the bifidelity-based PhIK and
CoPhIK as \emph{BiPhIK} and \emph{CoBiPhIK}, respectively.

Here, we roughly compare the computational cost of constructing $u_H(\Gamma)$ 
and $u_B(\Gamma)$ for the sample set $\Gamma$ of size $M$. We denote the cost of obtaining one realization of $u_H$ and
$u_L$ with $\mathrm{C}_H$ and $\mathrm{C}_L$, respectively. Therefore, the total
cost of obtaining $u_H(\Gamma)$ is $M_L\mathrm{C}_H$. The computational costs of
the pivot Cholesky decomposition and the lifting procedure are negligible when 
the simulation model is complicated. Thus, the total cost of obtaining 
$u_B(\Gamma)$ is approximated $M_L\mathrm{C}_L+M_H\mathrm{C}_H$. Therefore, the
ratio of computational cost for obtaining $u_B(\Gamma)$ and $u_H(\Gamma)$ is
$\dfrac{\mathrm{C}_L}{\mathrm{C}_H}+\dfrac{M_H}{M_L}$. The speedup  of bifidelity approximation over high-fidelity simulation on the data set $\Gamma$ can be significant 
when $\mathrm{C}_L\ll\mathrm{C}_H$ and $M_H\ll M_L$.


\subsubsection{Error estimate}

Because $u^m_{H}$ in PhIK and CoPhIK are replaced by $u^m_B$ in BiPhIK and
CoBiPhIK, we present error estimate results to compare the new method with 
the original approaches.
Recalling that $u^m(\bm x)$ in Eq.~\eqref{eq:phik} are $u_H^m(\bm x)$ used in
the bi-fidelity framework and $M=M_L$, we denote $\mu_{\mathrm{MC}}(\bm x)$ and 
$k_{\mathrm{MC}}(\bm x,\bm x')$ in this equation as $\mu_H(\bm x)$ and 
$k_H(\bm x, \bm x')$, respectively. Subsequently, we denote the resulting 
posterior mean and variance as $\hat y_H(\bm x)$ and $\hat s^2_H(\bm x)$. 
Similarly, when using bi-fidelity ensemble $u_B(\Gamma)$ to approximate mean
and covariance functions in Eq.~\eqref{eq:phik}, i.e. replacing $u^m(\bm x)$ 
with $u^m_B(\bm x)$, we denote these two functions as $\mu_B(\bm x)$ and 
$k_B(\bm x, \bm x')$, and the corresponding results as $\hat y_B(\bm x)$ and 
$\hat s^2_B(\bm x)$. We use $\Vert\cdot\Vert$ to denote the norm induced from
the inner product $\langle\cdot,\cdot\rangle_H$ in the Hilbert space $\mathbb{V}_H$,
and introduce the following functions:
\begin{equation}
  \begin{aligned}
  & \sigma_H(\bm x) = \left(\dfrac{1}{M-1}\sum_{m=1}^M
    \left\vert u_H^m(\bm x)-\mu_H(\bm x)\right\vert^2\right)^{\half}, \\
  & \sigma_B(\bm x) = \left(\dfrac{1}{M-1}\sum_{m=1}^M
    \left\vert u_B^m(\bm x)-\mu_B(\bm x)\right\vert^2\right)^{\half},
  \end{aligned}
\end{equation}
and the following constants:
\begin{equation}
  \begin{aligned}
  & \delta_1=\sup_{z\in I_z}\Vert u_H(\bm x;z)-u_B(\bm x;z)\Vert, \\
  & \delta_2=\sup_{z\in I_z}\Vert u_H(\bm x;z)-u_B(\bm x;z)\Vert_{\infty}, \\
  & \sigma_H(\Gamma) = \left(\dfrac{1}{M-1}\sum_{m=1}^M
    \left\Vert u_H^m(\bm x)-\mu_H(\bm x)\right\Vert^2\right)^{\half}, \\
  & \sigma_B(\Gamma) = \left(\dfrac{1}{M-1}\sum_{m=1}^M\left 
    \Vert u_B^m(\bm x)-\mu_B(\bm x)\right\Vert^2\right)^{\half}, \\
    & S_H = \left(\sum_{n=1}^N \sigma^2_H(\bm x^{(n)})\right)^{\half}, \quad 
    S_B = \left(\sum_{n=1}^N \sigma^2_B(\bm x^{(n)})\right)^{\half},   \\
    & \Delta_H=\sup_{\bm x\in D}\sigma_H(\bm x),\quad
     \Delta_B=\sup_{\bm x\in D}\sigma_B(\bm x). 
  \end{aligned}
\end{equation}
The following two theorems describe the difference between the results by PhIK 
and BiPhIK.
\begin{thm}
  \label{thm:diff_mean}
\begin{equation}
  \label{eq:post_mean}
\Vert \hat y_H(\bm x)-\hat y_B(\bm x) \Vert \leq C_1\delta_1+C_2\delta_2,\\
\end{equation}
where
\begin{equation}
  \begin{aligned}
    C_1 = &  1+2S_B\sqrt{\dfrac{MN}{M-1}}\left\Vert\tensor C_B^{-1}\right\Vert_2
  \left\Vert \bm y-\bm\mu_B \right\Vert_2,\\
C_2 = & \sqrt{N}S_H\sigma_H(\Gamma)\left\Vert\tensor C_H^{-1}\right\Vert_2
    \left\{ 2\sqrt{\dfrac{2 M}{M-1}}\left(S^2_H+S^2_B \right)^{\half} 
    \left\Vert\tensor C_H^{-1}\right\Vert_2\Vert\bm y-\bm\mu_B\Vert_2 +1\right\} \\
    & +  2\sqrt{\dfrac{MN}{M-1}}\sigma_H(\Gamma)\left\Vert\tensor C_B^{-1}\right\Vert_2
  \left\Vert \bm y-\bm\mu_B \right\Vert_2.
\end{aligned}
\end{equation}
\end{thm}

\begin{thm}
  \label{thm:diff_var}
\begin{equation}
  \label{eq:post_var}
  \Vert \hat s^2_H(\bm x)-\hat s^2_B(\bm x) \Vert_{\infty}\leq C_3\delta_2,
\end{equation}
where
\begin{equation}
    C_3 =  2\sqrt{\dfrac{2M}{M-1}}(\Delta_H^2+\Delta_B^2)^{\half}\cdot
           \left\{ 1 + N
\left(\Delta_H^2\Vert\tensor C_H^{-1}\Vert_2 
  +\sqrt{N}\Delta_H^2\Delta_B^2\Vert\tensor C_H^{-1}\Vert_2^2
+\Delta_B^2\Vert\tensor C_B^{-1}\Vert_2\right)\right\}.
\end{equation}
\end{thm}
We present the proof of these two theorems in Appendix~\ref{sec:append_proof}. 
We note that Theorem~\ref{thm:diff_var} uses the $L_{\infty}$ norm because the
greedy algorithm we use to add new observations is based on the maximum of
$\hat s$. The error estimate in $L_2$ norm can also be derived using the
similar procedure in the proofs of Theorems~\ref{thm:diff_var}.
Moreover, although we present upper bounds in terms of $\delta_1$ and $\delta_2$,
the error estimate is dependent on the mean and covariance functions constructed 
by different methods. It is possible that in some cases, the pathwise difference
is large in different methods (i.e., $\delta_1$ and $\delta_2$ are large), but
the difference between the mean and covariance functions are small. Moreover,
the quantitative error estimate for CoBiPhIK is also dependent on the kernel
function property for $Y_d$ and the convexity of the optimization, and it is not 
available at this time. Empirically, CoPhIK is more sensitive to the difference
between $u_H(\Gamma)$ and $u_B(\Gamma)$. 

The next two theorems describe how well a linear physical constraint is preserved
in BiPhIK and CoBiPhIK posterior means. 
\begin{thm}
 \label{thm:biphik_err}
Assume that $\Vert\mathcal{L}u_H(\bm x;z(\omega^m))-g(\bm x)\Vert\leq\epsilon$ 
for any $\omega^m\in\Omega$, where $\mathcal{L}$ is a deterministic bounded linear
operator, $g(\bm x)$ is a well-defined deterministic function on $\mathbb{R}^d$,
and $\Vert\cdot\Vert$ is the norm in $\mathbb{V}_H$.
Then, the posterior mean $\hat y_B(\bm x)$ from BiPhIK satisfies
  \begin{equation}
  \label{eq:err_bound}
    \Vert\mathcal{L}\hat y_B(\bm x)-g(\bm x)\Vert 
    \leq \epsilon\left\{1+2 S_H \sqrt{\dfrac{M}{M-1}}
     \left\Vert\tensor C^{-1}_H\right\Vert_2
   \left\Vert\bm y-\bm\mu_H\right\Vert_2\right\}  
   + M_{\mathcal{L}}(C_1\delta_1 + C_2\delta_2),
\end{equation}
where $M_{\mathcal{L}}$ is the bound of $\mathcal{L}$, $C_1$ and $C_2$ are 
defined in Theorem~\ref{thm:diff_mean}.
\end{thm}
\begin{proof}
  \[\begin{aligned}
    \Vert\mathcal{L}\hat y_B(\bm x) - g(\bm x) \Vert
= & \Vert\mathcal{L}[\hat y_B(\bm x) -\hat y_H(\bm x) + \hat y_H(\bm x)]-g(\bm x)\Vert \\
 \leq &
  \Vert\mathcal{L}\hat y_H(\bm x) -g(\bm x) \Vert + 
  \Vert\mathcal{L}(\hat y_B(\bm x)-\hat y_H(\bm x))\Vert \\
  = & \Vert\mathcal{L}\hat y_H(\bm x) -g(\bm x) \Vert + 
  M_{\mathcal{L}}\Vert\hat y_B(\bm x)-\hat y_H(\bm x)\Vert. 
\end{aligned} \]
Theorem~\ref{thm:diff_mean} presents the upper bound of 
$\Vert\hat y_B(\bm x)-\hat y_H(\bm x)\Vert$. The upper bound for
$\Vert\mathcal{L}\hat y_H(\bm x)-g(\bm x)\Vert$ needs slightly modifying the 
proof of Theorem 2.1\ in \cite{YangTT18}. Specifically, by setting 
$g(\bm x;\omega)=g(\bm x)$ the last line in that proof will be
\begin{equation}
  \begin{aligned}
  \Vert\mathcal{L}\hat y_H(\bm x)-g(\bm x)\Vert 
  \leq & \epsilon + 2\epsilon\sqrt{\dfrac{M}{M-1}}\sum_{n=1}^N|a_n|\sigma_H(\bm x^{(n)}),
  \end{aligned}
\end{equation}
where 
\[
  \sum_{n=1}^N|a_n|\sigma_H(\bm x^{(n)})
  \leq \left(\sum_{n=1}^N a_n^2\right)^{\half}
         \left(\sum_{n=1}^N\sigma_H^2(\bm x^{(n)})\right)^{\half} 
         = \Vert\tensor C_H^{-1}(\bm y-\bm\mu_H)\Vert_2S_H 
         \leq \Vert\tensor C_H^{-1}\Vert_2\Vert\bm y-\bm\mu_H\Vert_2 S_H.
\]
\end{proof}

\begin{thm}
  \label{thm:cobiphik_err}
Assume that $\Vert\mathcal{L}u_H(\bm x;z(\omega^m))-g(\bm x)\Vert\leq\epsilon$ 
for any $\omega^m\in\Omega$, where $\mathcal{L}$ is a deterministic bounded linear
operator, $g(\bm x)$ is a well-defined deterministic function on $\mathbb{R}^d$,
and $\Vert\cdot\Vert$ is the norm in $\mathbb{V}_H$. Then, the posterior mean 
$\hat y_B(\bm x)$ from CoBiPhIK, in which $Y_L$ is constructed by $u_B(\Gamma)$,
satisfies
\begin{equation}
\label{eq:err_bound_phicok}
\begin{aligned}
\left\Vert\mathcal{L}\hat y_{B}(\bm x)-g(\bm x)\right\Vert
\leq & \rho\epsilon  + (1-\rho) \Vert g(\bm x)\Vert 
+ 2\epsilon\rho S_H\sqrt{\dfrac{M}{M-1}}
\Vert \tensor C_B^{-1}\Vert_2\Vert\bm y_L-\bm \mu_B \Vert_2 \\
&  + \Vert\mathcal{L}\mu_d\Vert + 
\Vert \tensor C_2^{-1}\Vert_2\Vert\bm y_H-\rho\bm y_B-\bm 1 \mu_d \Vert_2
  \sum_{n=1}^N \left\Vert\mathcal{L} k_{_d}(\bm x, \bm x^{(n)})\right\Vert
   + M_{\mathcal{L}}(C_1\delta_1+C_2\delta_2),
\end{aligned}
\end{equation}
where $\tensor C_1$ and $\tensor C_2$ are $\tensor C_L(\bm X_L, \bm X_L)$ and
$\tensor C_d(\bm X_H, \bm X_H)$ in Algorithm~\ref{algo:phicok}, respectively,
and $M_{\mathcal{L}}$ is the bound of $\mathcal{L}$.
\end{thm}
\begin{proof}
  \[\begin{aligned}
    \Vert\mathcal{L}\hat y_B(\bm x) - g(\bm x) \Vert
  \leq \Vert\mathcal{L}\hat y_H(\bm x) -g(\bm x) \Vert + 
  M_{\mathcal{L}}\Vert\hat y_B(\bm x)-\hat y_H(\bm x)\Vert,
\end{aligned} \]
where $\hat y_H(\bm x)$ is the posterior mean by the original CoPhIK method,
i.e., $\{u^m(\bm x)\}_{m=1}^M$ is taken as $u_H(\Gamma)$ in 
Algorithm~\ref{algo:phicok}. Then similar to the proof in 
Theorem~\ref{thm:cobiphik_err}, we set $g(\bm x;\omega)=g(\bm x)$ and use 
Cauchy-Schwartz inequality to slightly modify the upper bound estiamte of 
$\sum_{n=1}^N|a_n|\sigma_H(\bm x^{(n)})$ in Theorem 2.2\ in 
\cite{YangBTT18} to finish the proof.
\end{proof}

\section{Numerical Examples}
\label{sec:numeric}

We present three numerical examples to demonstrate the performance of the
proposed methods. The first two examples are drawn from the previous examples 
in~\cite{YangBTT18} to compare different methods, and they are two-dimensional
problems in physical space. We denote a reference solution, a discretized 
two-dimensional field, as matrix $\tensor F$, the reconstructed field (posterior
mean) as $\tensor F_r$. We present the RMSE $\hat s$, difference 
$\tensor F_r-\tensor F$ and relative error 
$\Vert\tensor F_r-\tensor F\Vert_F/\Vert \tensor F\Vert_F$
($\Vert\cdot\Vert_F$ is the Frobenius norm) to compare different methods. 
Moreover, we adaptively add new observations at the maxima of
$\hat s$ (see Appendix~\ref{sec:append_act}) to numerically study the 
convergence with respect to number of observations. In all three examples, we
use Gaussian kernel in Kriging, CoPhIK, and CoBiPhIK because the fields in the
examples are relatively smooth. 

\subsection{Branin function}
We consider the following modified Branin function \cite{forrester2008engineering}:
\begin{equation}
\label{eq:branin_fun}
f(\bm x) = a(\bar y-b\bar x^2+c\bar x-r)^2 + g(1-p)\cos(\bar x)+g + qx,
\end{equation}
where $\bm x=(x, y)$,
\[\bar x=15x-5,~\bar y=15y,~(x,y)\in D=[0,1]\times [0,1],\]
and
\[a=1,~b=\frac{5.1}{4\pi^2},~c=\frac{5}{\pi},~r=6,~g=10,~p=\frac{1}{8\pi},~q=5.\]
The contour of $f$ and eight randomly chosen observation locations
$\{(0.1,0.225),\\ (0.475,0.2), (0.625,0.5), (0.675,0.55), (0.7,0), (0.775,0.1), (0.8,0.9), (0.925,0.9)\}$
are \\ presented in Fig.~\ref{fig:branin_truth}. The function $f$ is evaluated on 
a $41\times 41$ uniform grid. 
\begin{figure}[!h]
\centering
\includegraphics[width=0.40\textwidth]{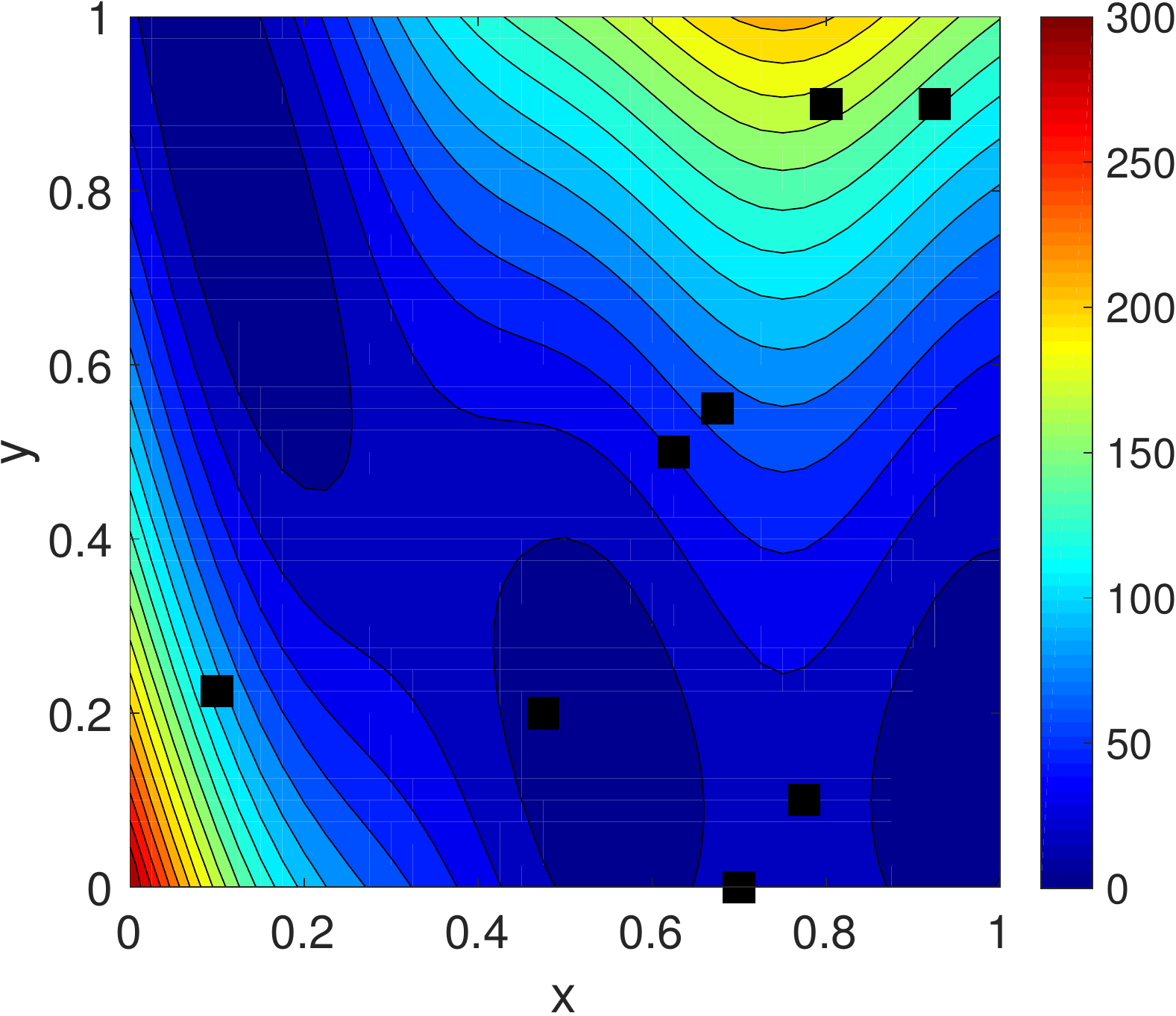}
\caption{Contours of modified Branin function (on $41\times 41$ uniform grids)
         and locations of eight observations (black squares).}
\label{fig:branin_truth}
\end{figure}

We assume that based on ``domain knowledge",  $f(\bm x)$ is partially known.
Specifically, its form is known but the coefficients $b$ and $q$ are unknown. 
Then, we treat these coefficients as random fields $\hat b$ and $\hat q$, and we
also modify the second $g$ as $\hat g$, which indicates that the field $f$ is 
described by a random function $\hat f: D\times\Omega\rightarrow\mathbb{R}$:
\begin{equation}
\label{eq:branin_rf}
  \hat f(\bm x;\omega) = a(\bar y-\hat b(\bm x;\omega)\bar x^2+c\bar x-r)^2 
      + g(1-p)\cos(\bar x)+ \hat g + \hat q(\bm x;\omega)x,
\end{equation}
where $\hat g=20$,
\begin{align}
  \hat b(\bm x;\omega) 
  & = b\bigg\{0.9+\dfrac{0.2}{\pi}\sum_{i=1}^3\bigg[
  \dfrac{1}{4i-1}\sin((2i-0.5)\pi x)\xi_{2i-1}(\omega)\notag  
    + \dfrac{1}{4i+1}\sin((2i+0.5)\pi y)\xi_{2i}(\omega)\bigg]\bigg\}, \\
  \hat q(\bm x;\omega) 
  & = q\bigg\{1.0+\dfrac{0.6}{\pi}\sum_{i=1}^3\bigg[
  \dfrac{1}{4i-3}\cos((2i-1.5)\pi x)\xi_{2i+5}(\omega)\notag 
   + \dfrac{1}{4i-1}\cos((2i-0.5)\pi y)\xi_{2i+6}(\omega)\bigg]\bigg\},
\end{align}
and $\{\xi_i(\omega)\}_{i=1}^{12}$ are i.i.d. Gaussian random variables with 
zero mean and unit variance. We use this ``physical knowledge" to compute the 
mean and covariance function of $\hat{f}$ by generating $M=300$ samples of 
$\xi_i(\omega)$ and evaluating $\hat f$ on the $21\times 21$ uniform grid for 
each sample of $\xi_i(\omega)$ to obtain realization ensemble $u_L(\Gamma)$.
We set $M_H=21$ to construct $\gamma$ and subsequently evaluate $\hat f$ on a
$41\times 41$ uniform grid to obtain $u_H(\gamma)$. Finally, we construct 
$u_B(\Gamma)$ based on $u_H(\gamma)$ and $u_L(\Gamma)$ 
(Algorithm~\ref{algo:bifi}), and use it in BiPhIK and CoBiPhIK.

It is shown in~\cite{YangTT18, YangBTT18} that Kriging results in inaccurate
reconstruction of $\tensor F$ by using the eight observation data. The results 
of PhIK and BiPhIK are very similar in this case, and we present the latter in
Fig.~\ref{fig:branin_bphik}. These results are much better than the Kriging 
(see also the quantitative comparison in Fig.~\ref{fig:branin_rel_err}).
\begin{figure}[!h]
\centering
\subfigure[BiPhIK $\tensor F_r$]{
\includegraphics[height=0.20\textwidth]{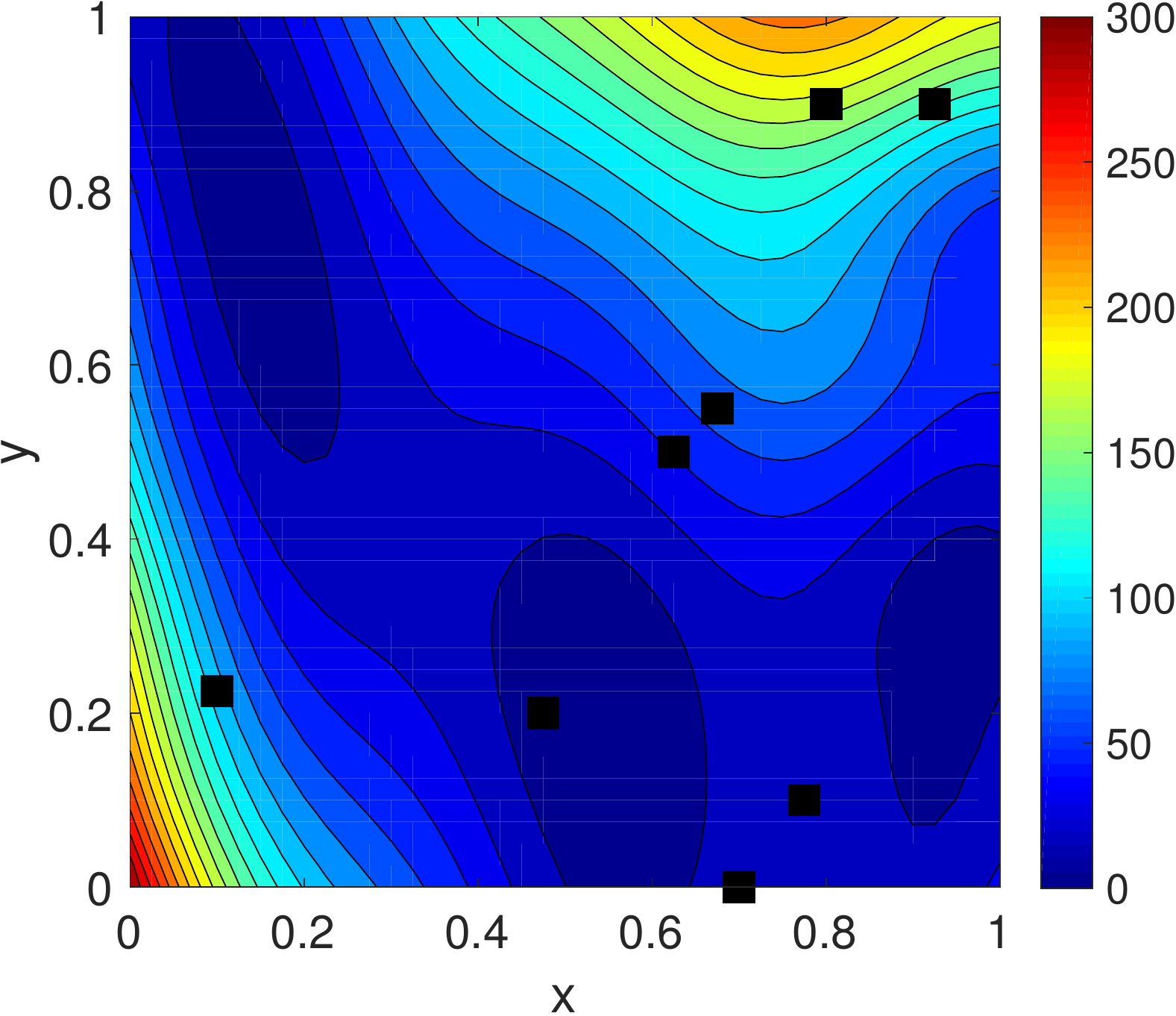}}\qquad
\subfigure[BiPhIK $\hat s$]{
\includegraphics[height=0.20\textwidth]{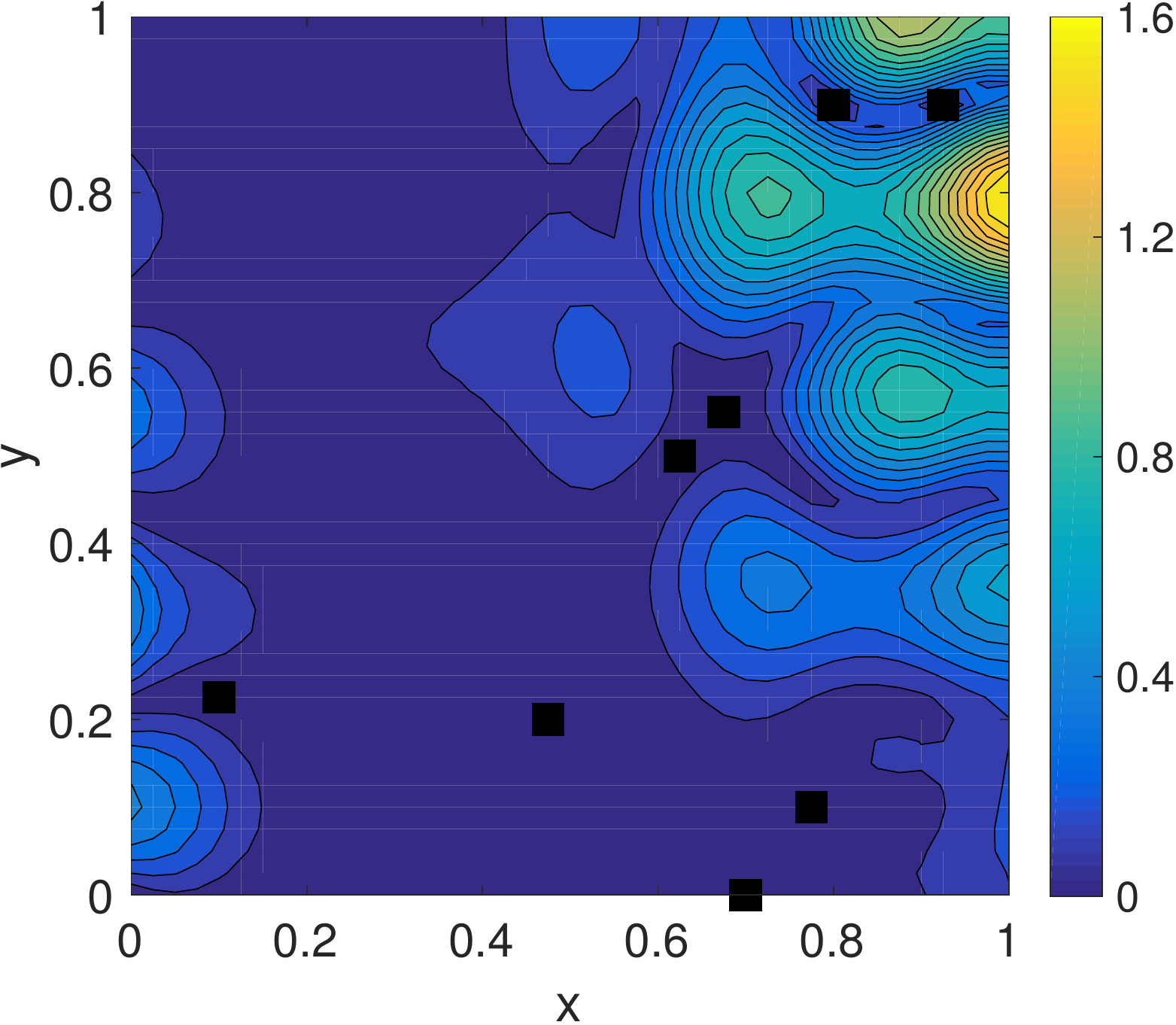}}\qquad
\subfigure[BiPhIK $\tensor F_r-\tensor F$]{
\includegraphics[height=0.20\textwidth]{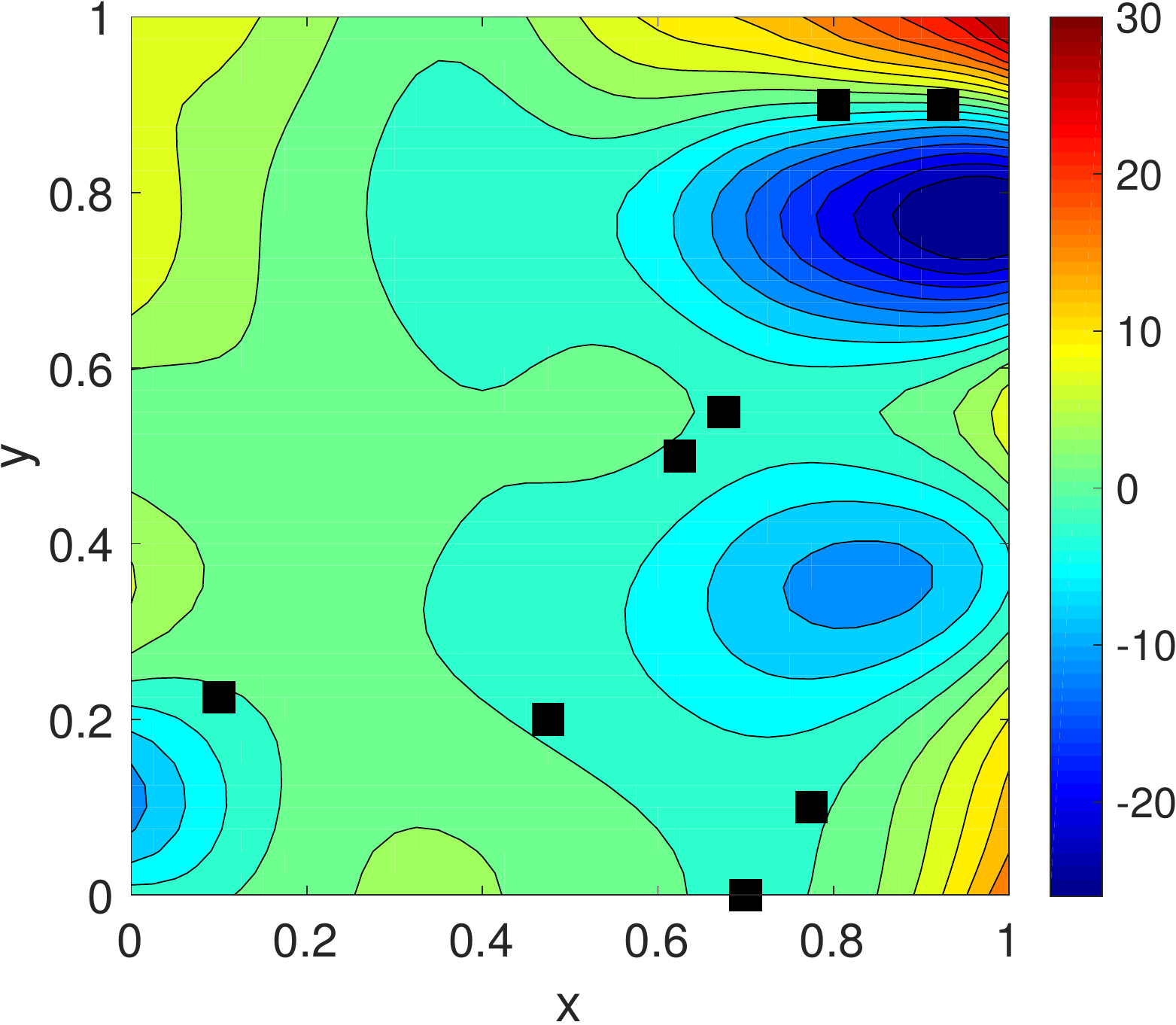}} 
\caption{Reconstruction of the modified Branin function by BiPhIK with eight
original observations (squares).}
\label{fig:branin_bphik}
\end{figure}
There are slight difference betwen the results from CoPhIK and CoBiPhIK as shown in
Fig.~\ref{fig:branin_bphicok}.
\begin{figure}[!h]
\centering
\subfigure[CoPhIK $\tensor F_r$]{
\includegraphics[height=0.20\textwidth]{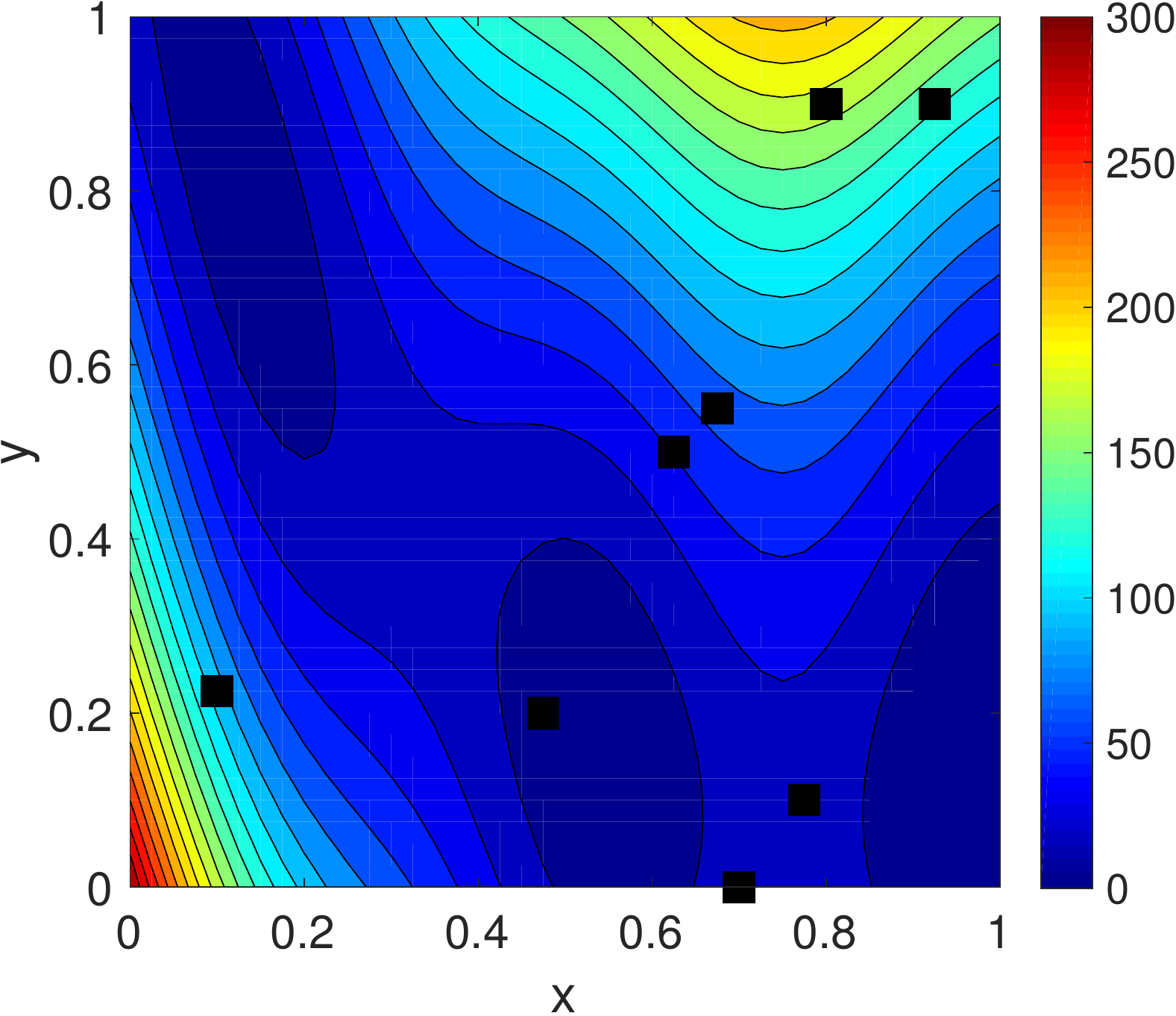}}\qquad
\subfigure[CoPhIK $\hat s$]{
\includegraphics[height=0.20\textwidth]{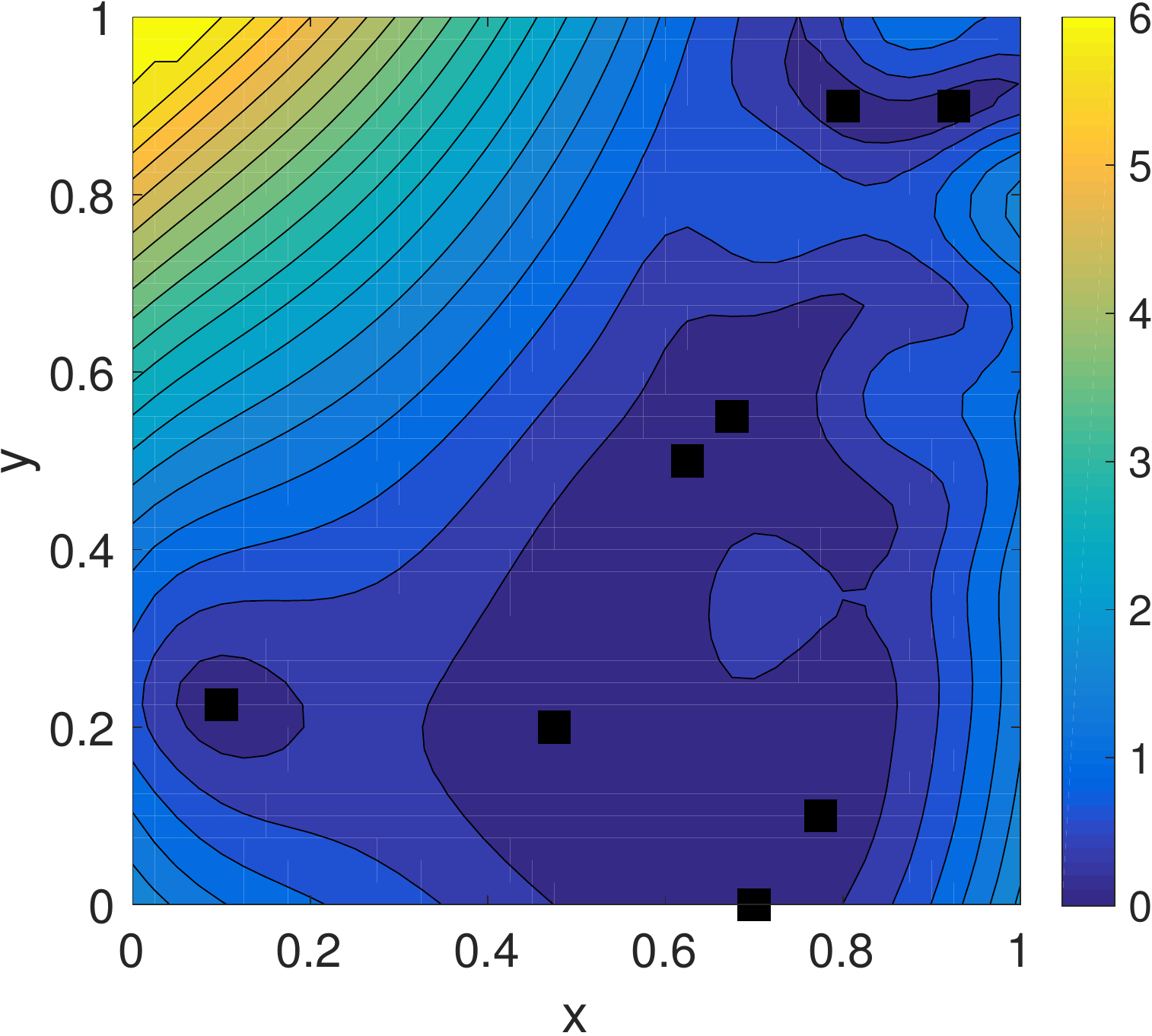}}\qquad
\subfigure[CoPhIK $\tensor F_r-\tensor F$]{
\includegraphics[height=0.20\textwidth]{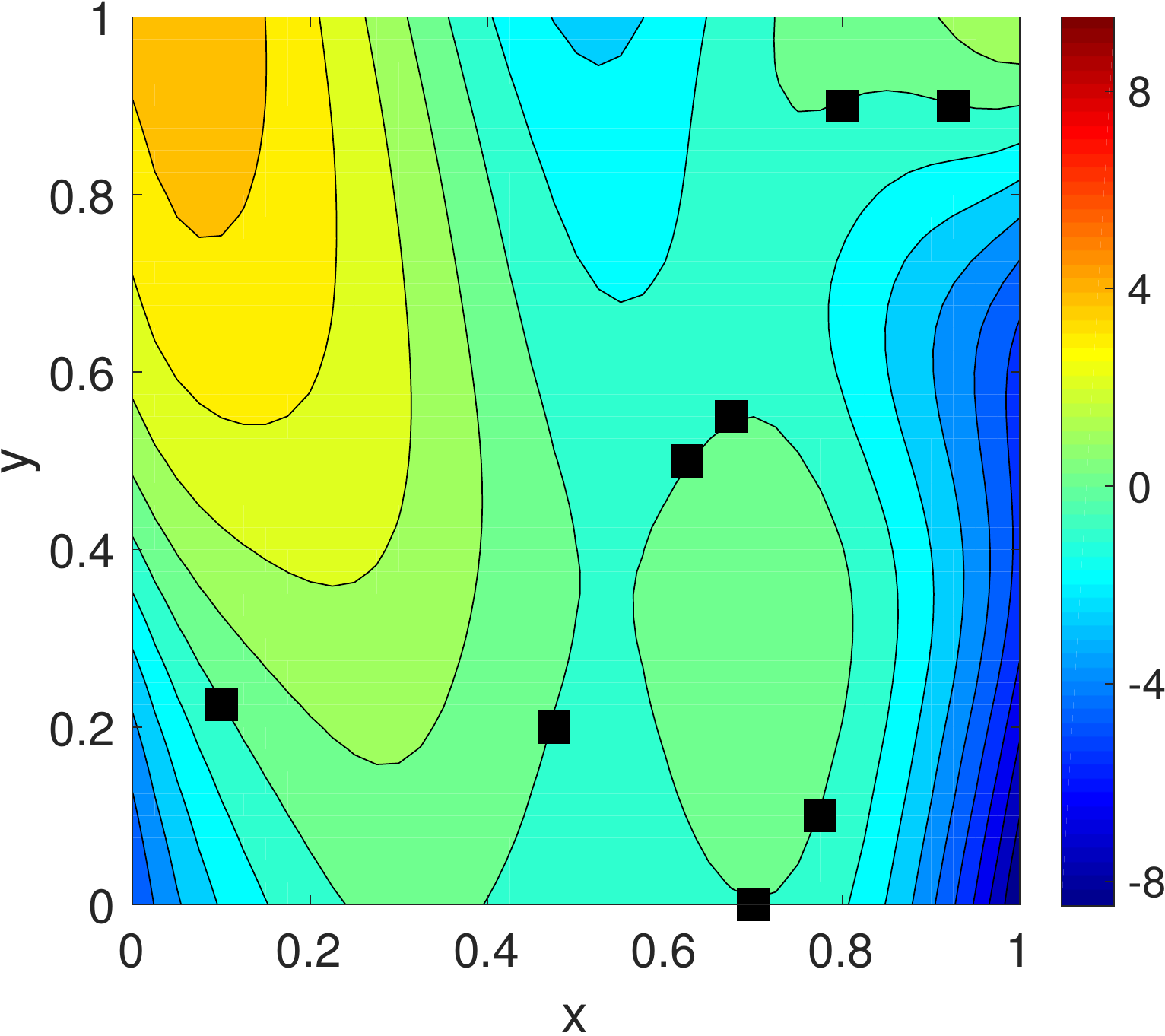}} \\
\subfigure[CoBiPhIK $\tensor F_r$]{
\includegraphics[height=0.20\textwidth]{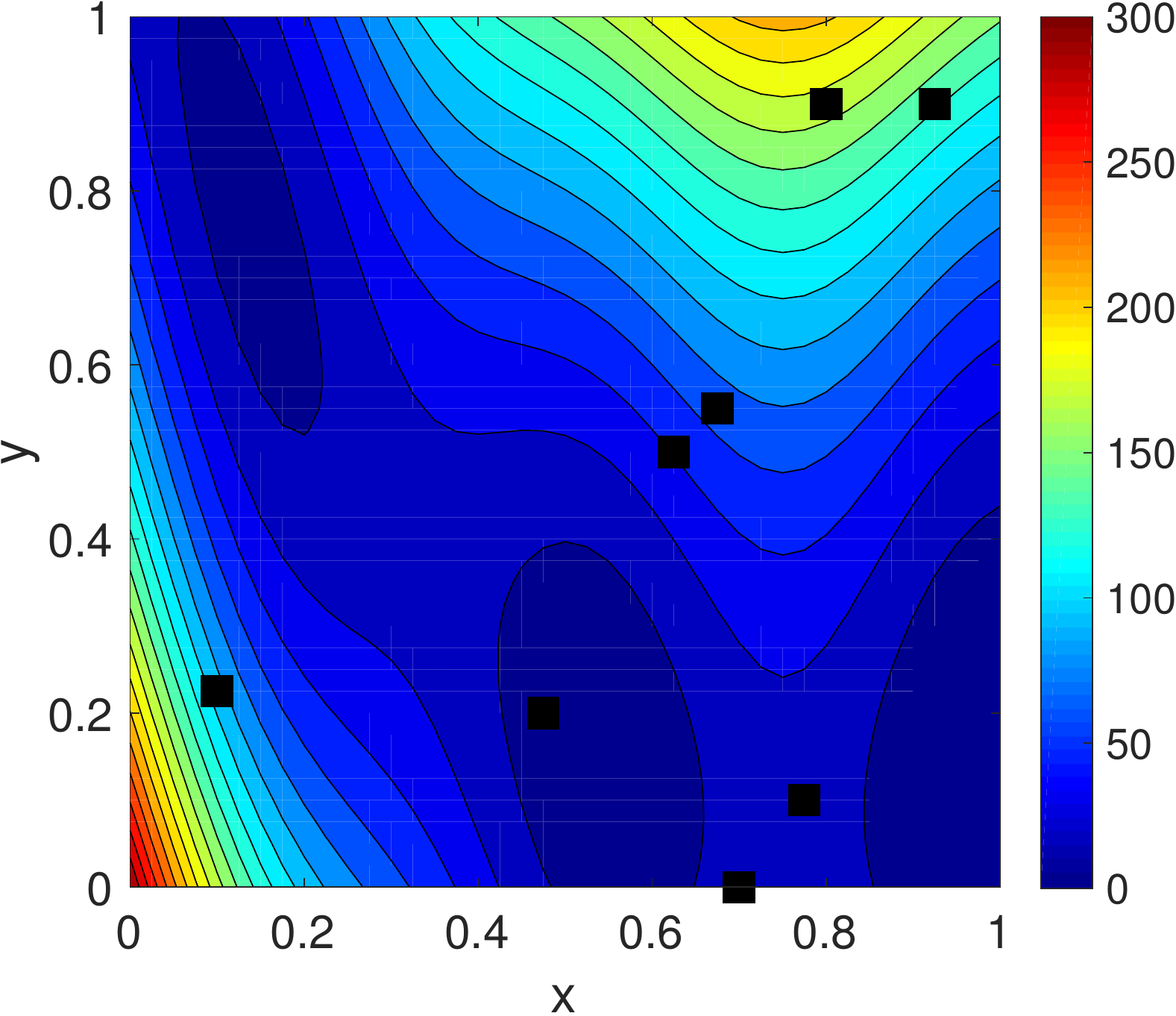}}\qquad
\subfigure[CoBiPhIK $\hat s$]{
\includegraphics[height=0.20\textwidth]{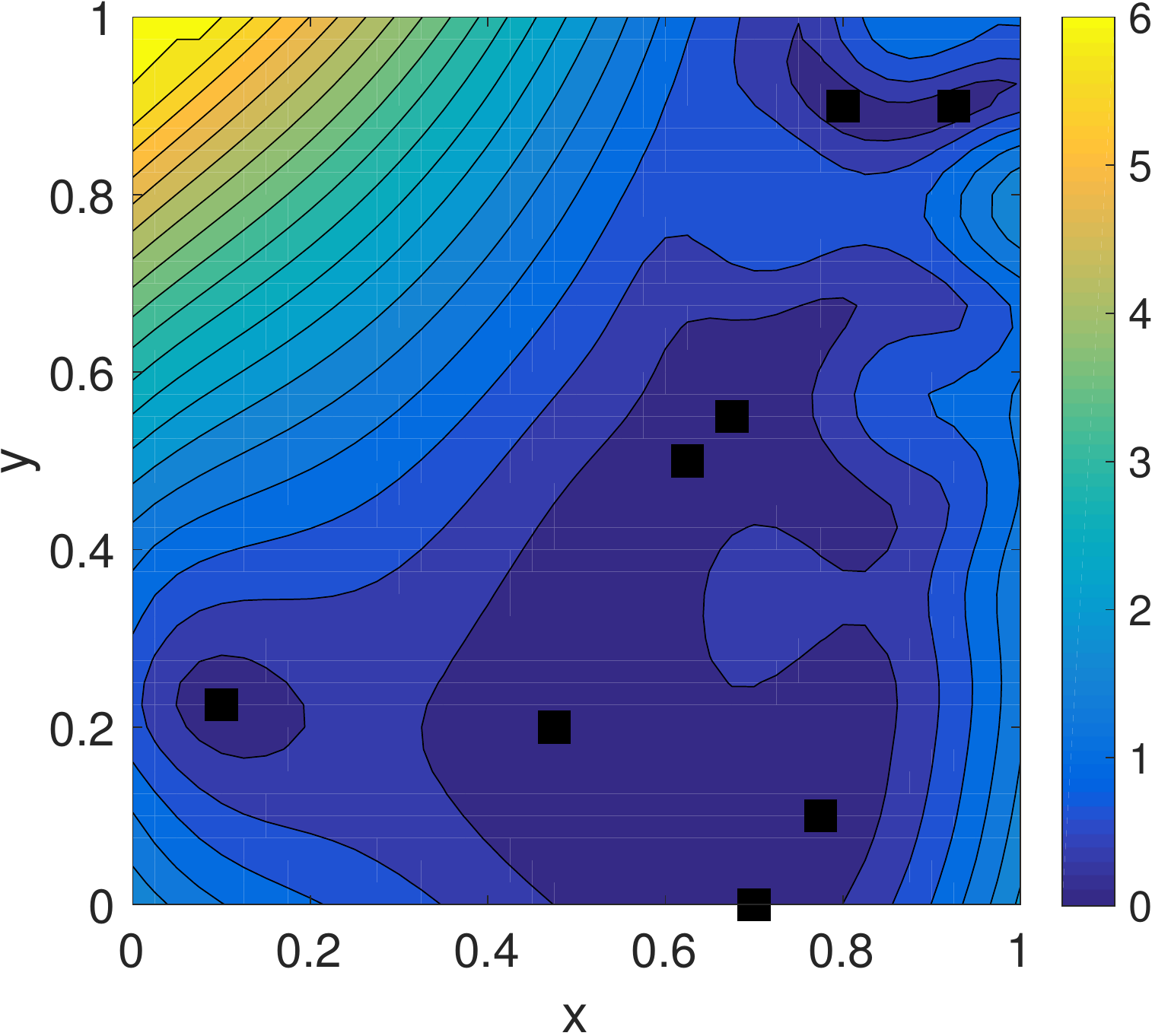}}\qquad
\subfigure[CoBiPhIK $\tensor F_r-\tensor F$]{
\includegraphics[height=0.20\textwidth]{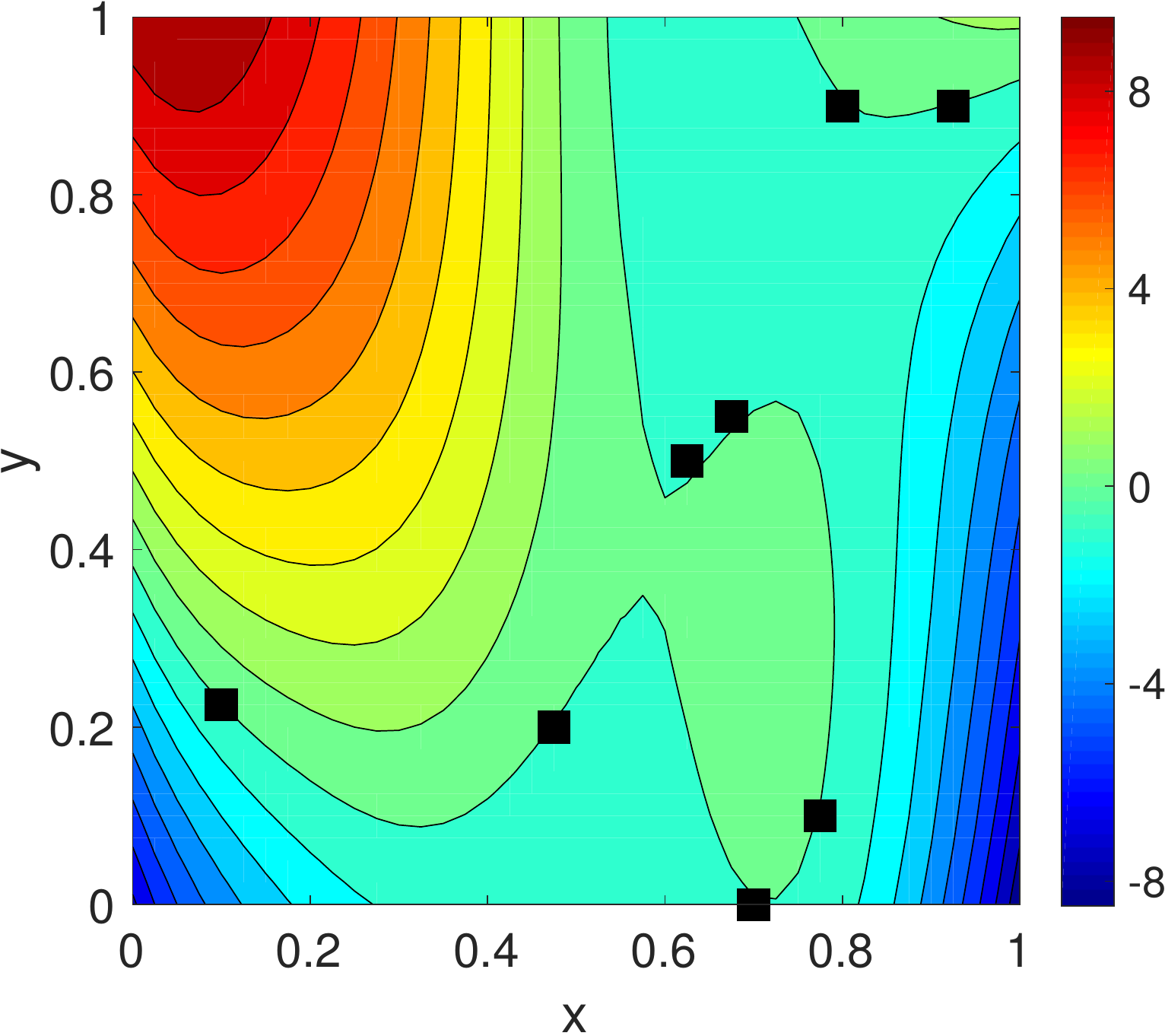}}
\caption{Reconstruction of the modified Branin function by CoPhIK and Co-BiPhIK with eight
original observations (squares).}
\label{fig:branin_bphicok}
\end{figure}

We then use a greedy algorithm (in the appendix) that acquires additional observations of the 
exact field one by one.
Fig.~\ref{fig:branin_bphik_act} presents the comparison of PhIK and BiPhIK
when eight additional observations (marked as black stars) are added, which 
shows a slight difference in the location of additional observations and small 
discrepancies between the results. Also, Fig.~\ref{fig:branin_bphicok_act}
illustrates the difference between CoPhIK and CoBiPhIK, and we see more
significant differences in the pattern of $\hat s$ and $\bm F_r-\bm F$ than the
comparison between PhIK and BiPhIK.
\begin{figure}[!h]
\centering
\subfigure[PhIK $\tensor F_r$]{
\includegraphics[height=0.20\textwidth]{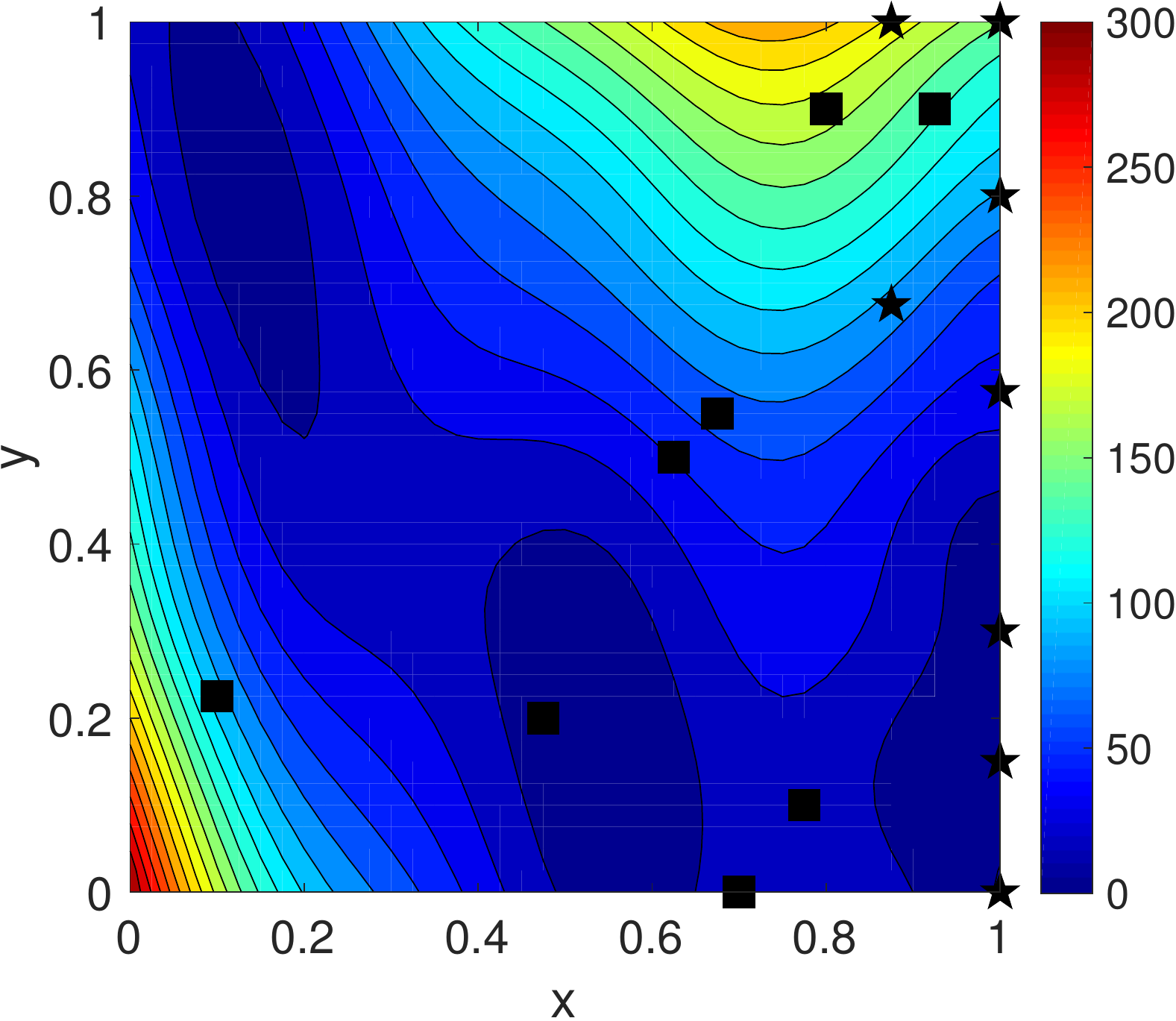}}\qquad
\subfigure[PhIK $\hat s$]{
\includegraphics[height=0.20\textwidth]{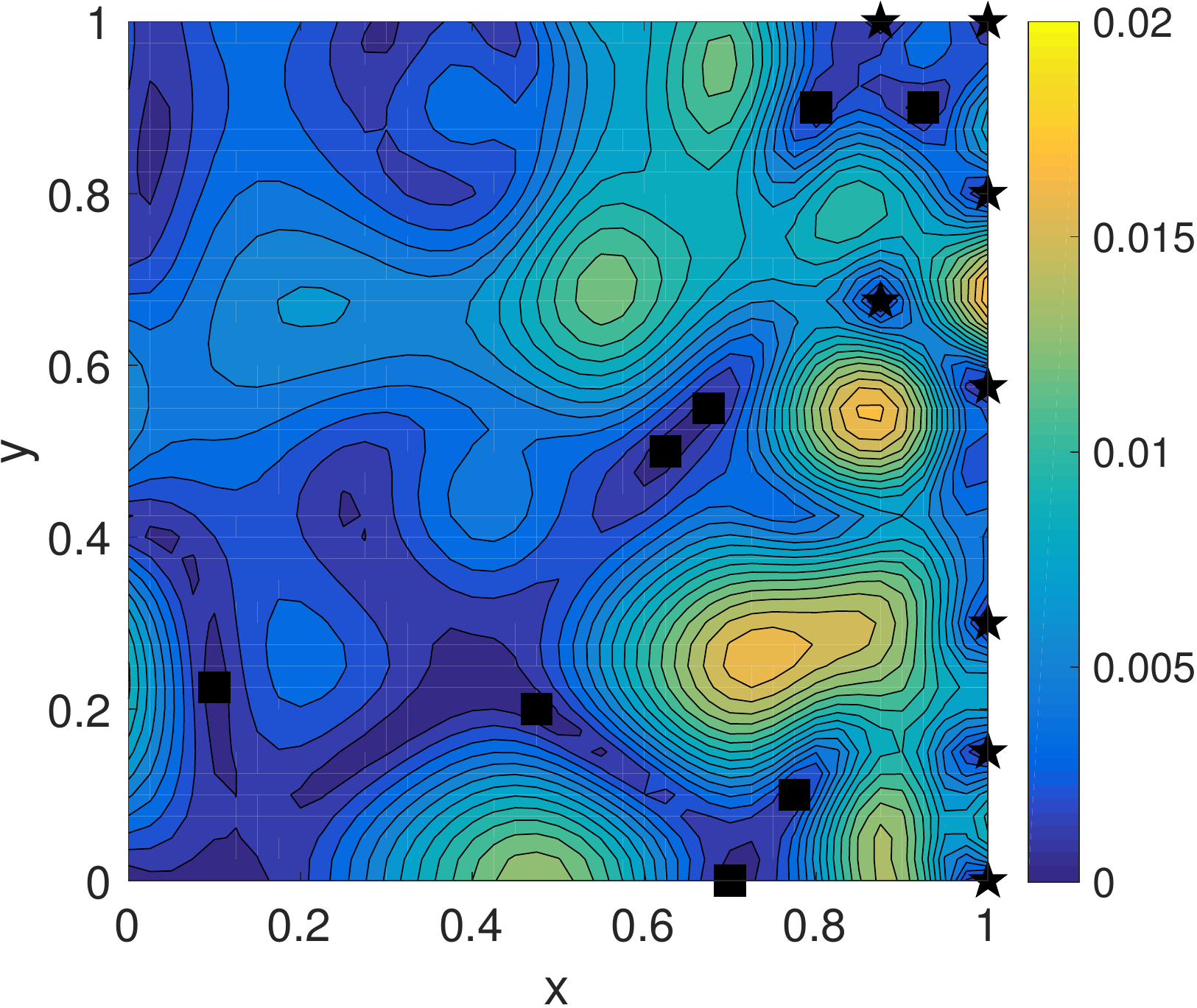}}\qquad
\subfigure[PhIK $\tensor F_r-\tensor F$]{
\includegraphics[height=0.20\textwidth]{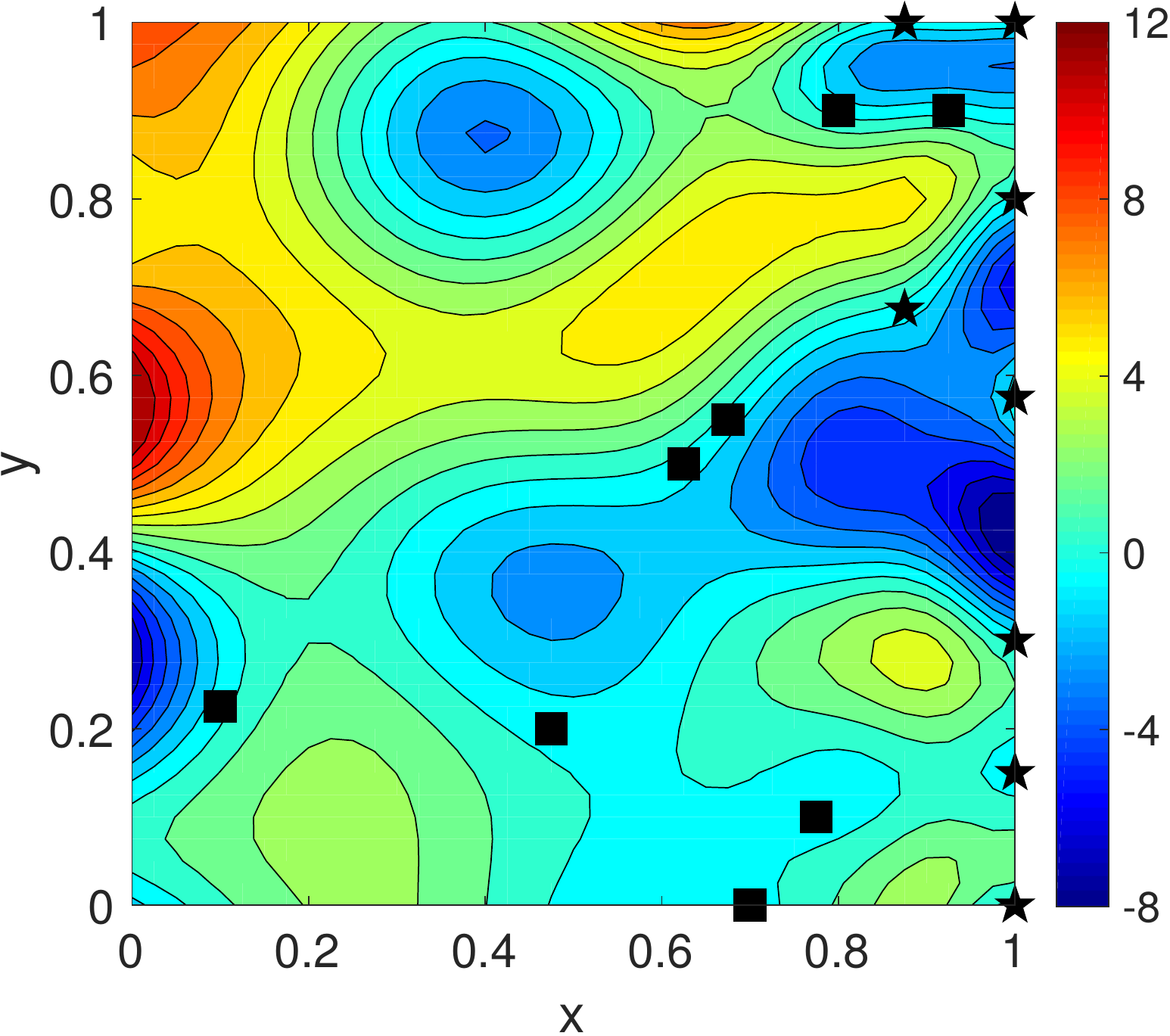}} \\
\subfigure[BiPhIK $\tensor F_r$]{
\includegraphics[height=0.20\textwidth]{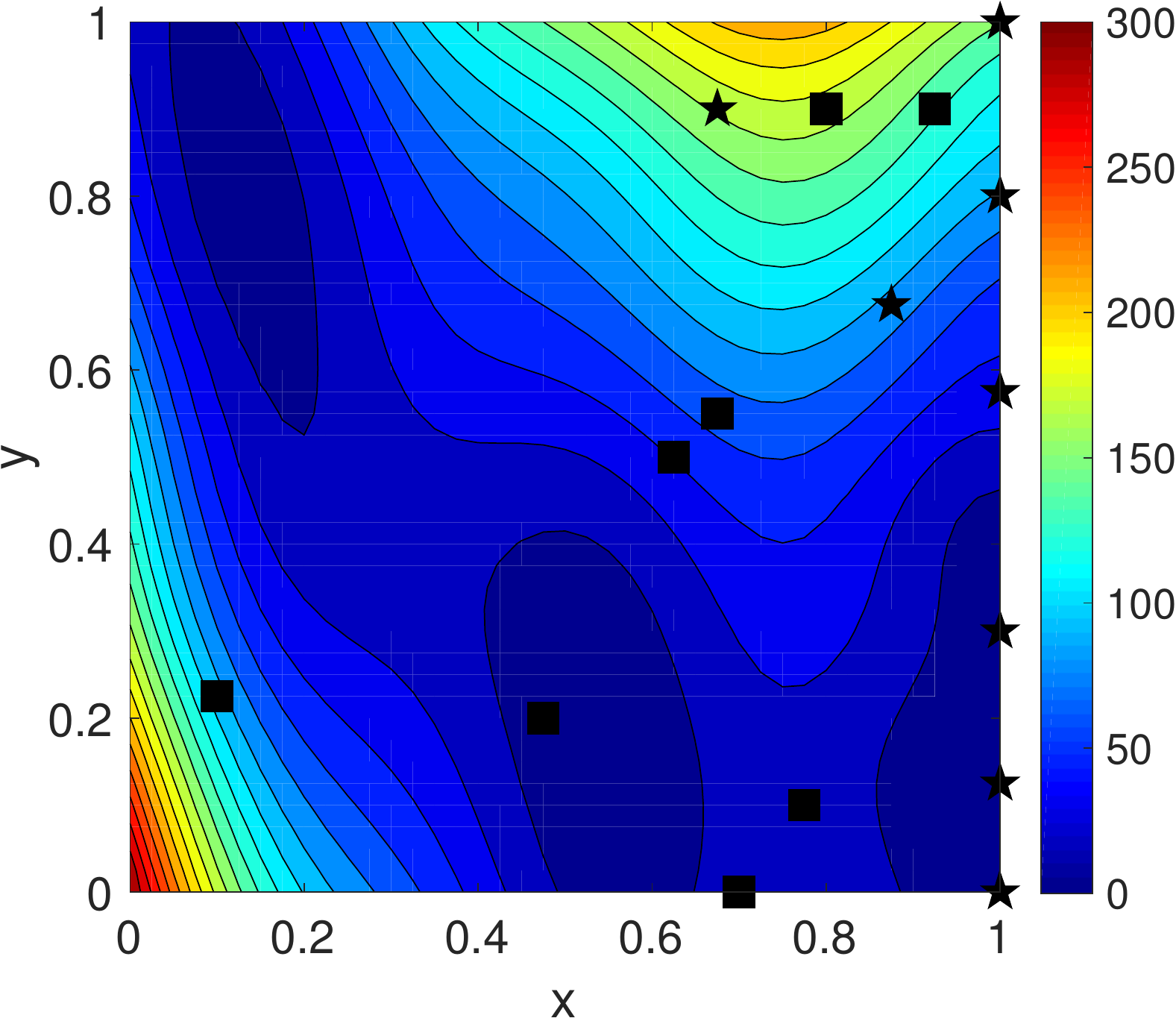}}\qquad
\subfigure[BiPhIK $\hat s$]{
\includegraphics[height=0.20\textwidth]{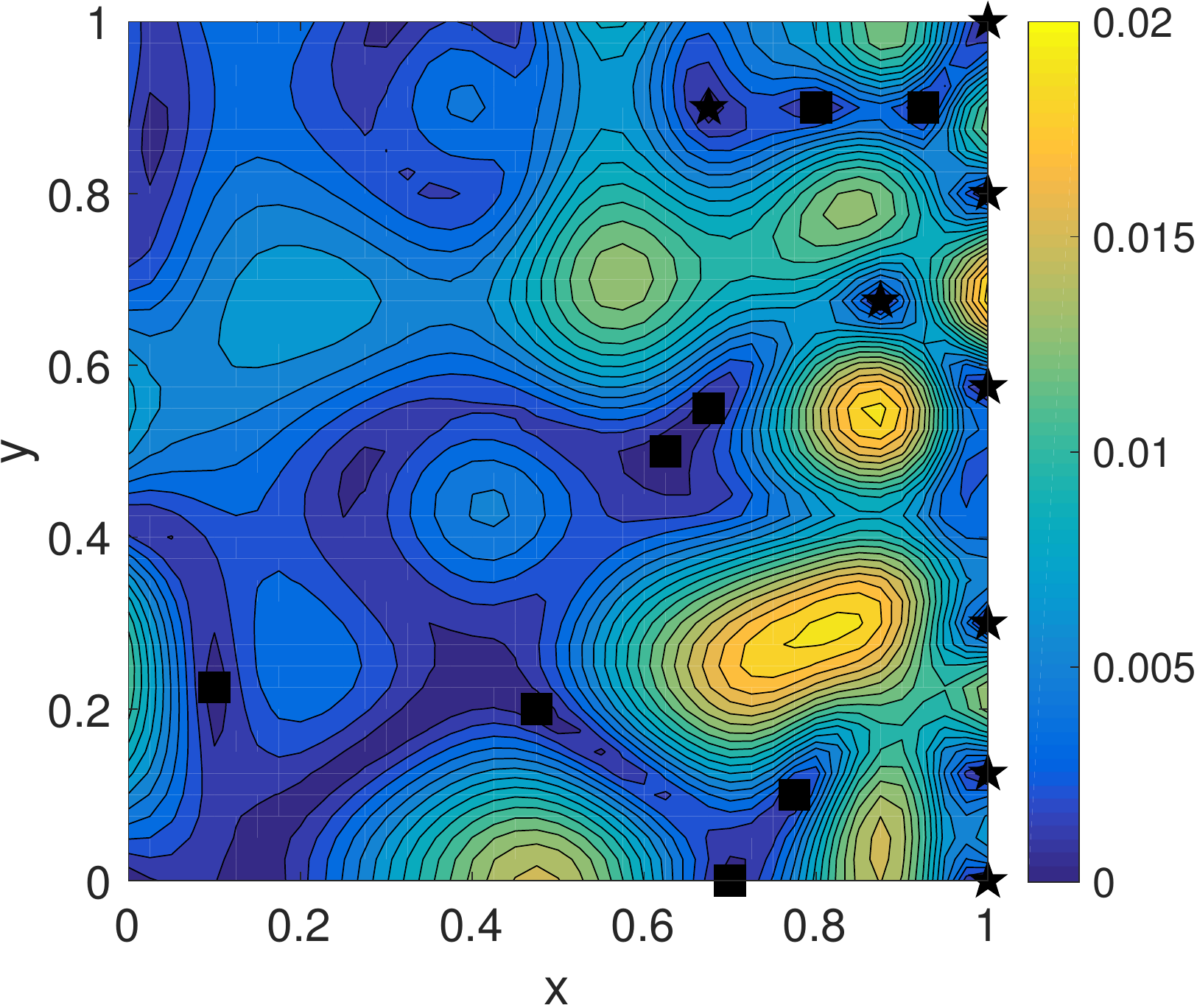}}\qquad
\subfigure[BiPhIK $\tensor F_r-\tensor F$]{
\includegraphics[height=0.20\textwidth]{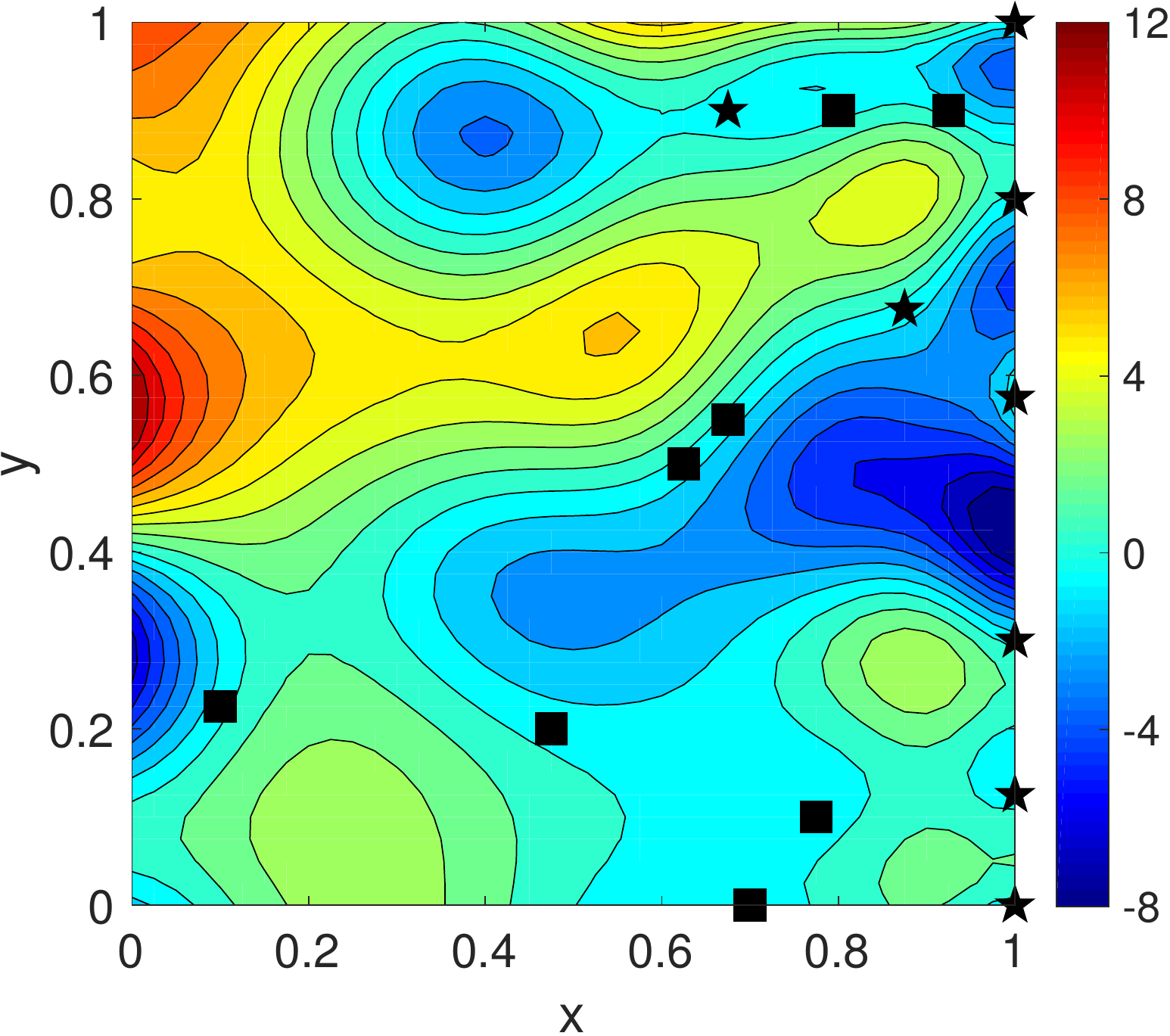}}
\caption{Reconstruction of the modified Branin function by PhIK and BiPhIK with eight
original observations (squares) and eight additional observations (stars).}
\label{fig:branin_bphik_act}
\end{figure}
\begin{figure}[!h]
\centering
\subfigure[CoPhIK $\tensor F_r$]{
\includegraphics[height=0.20\textwidth]{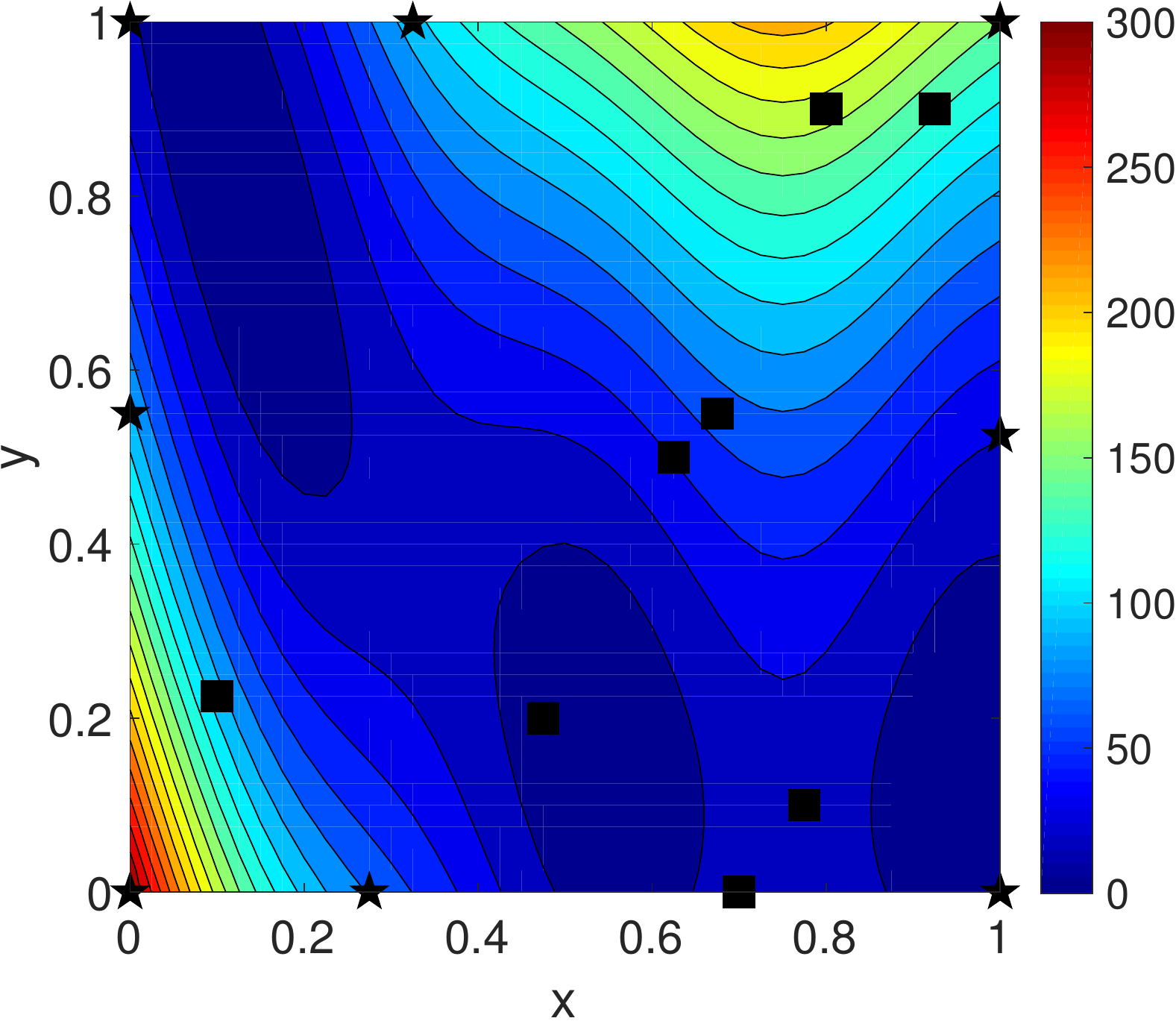}}\qquad
\subfigure[CoPhIK $\hat s$]{
\includegraphics[height=0.20\textwidth]{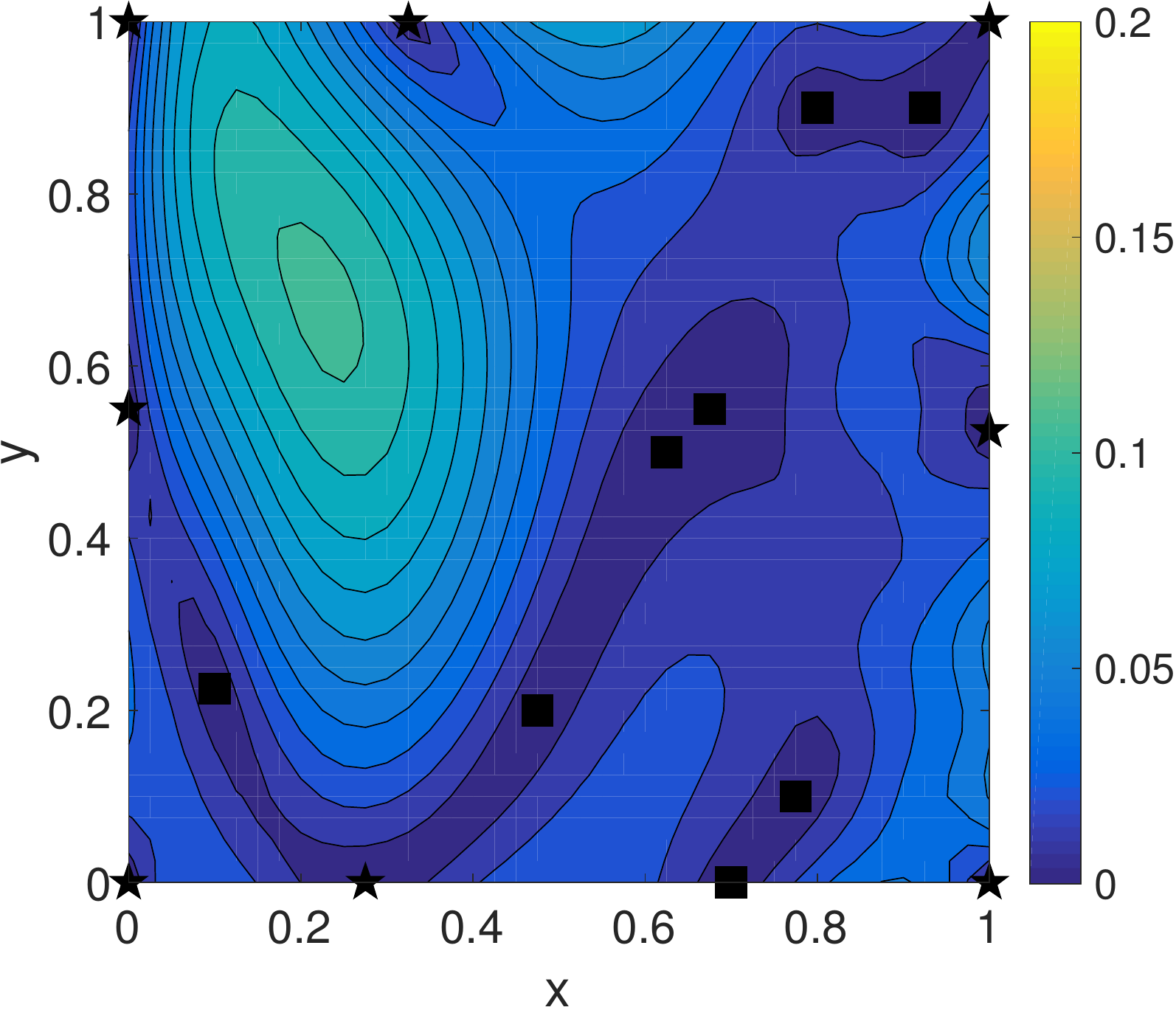}}\qquad
\subfigure[CoPhIK $\tensor F_r-\tensor F$]{
\includegraphics[height=0.20\textwidth]{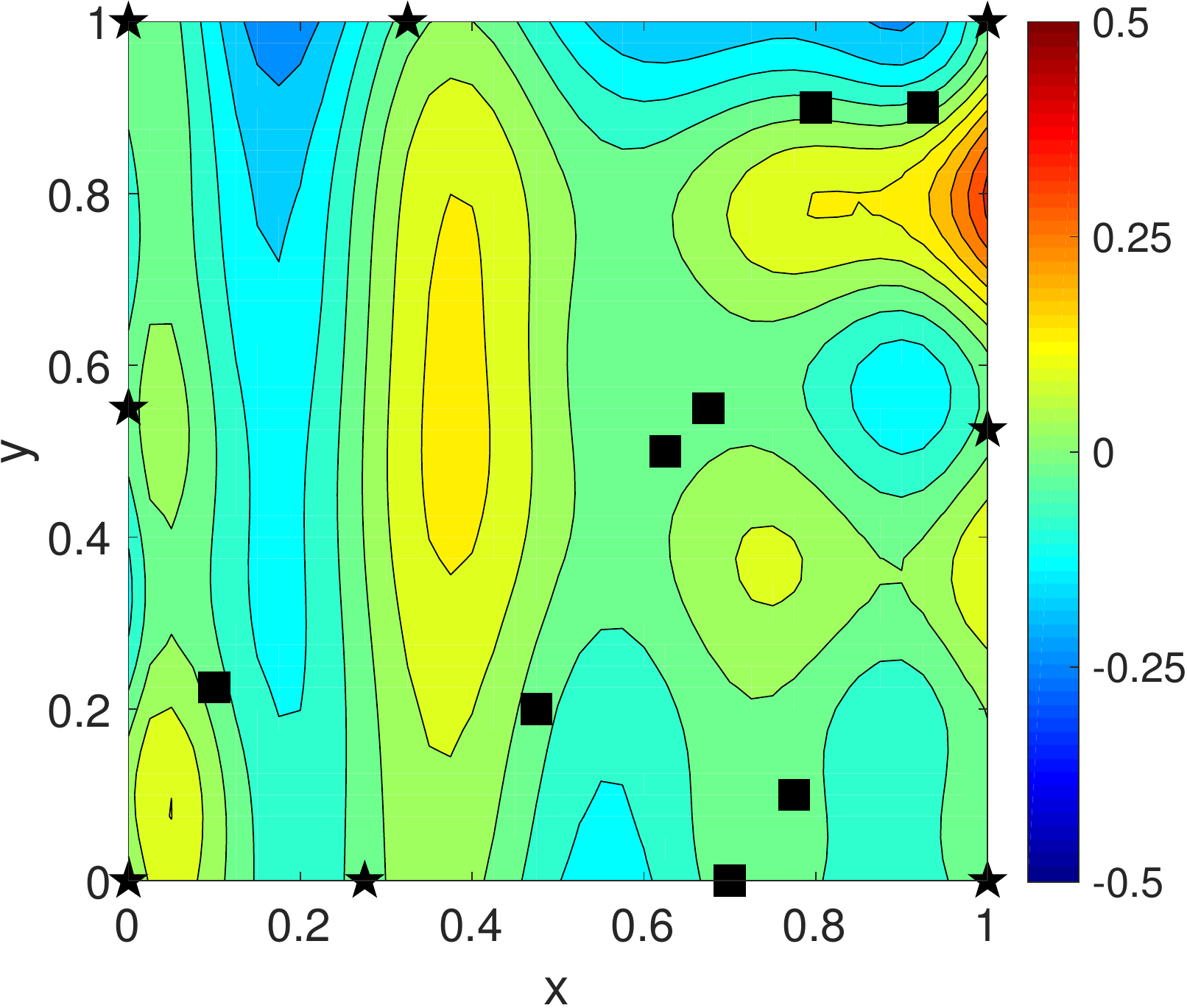}} \\
\subfigure[CoBiPhIK $\tensor F_r$]{
\includegraphics[height=0.20\textwidth]{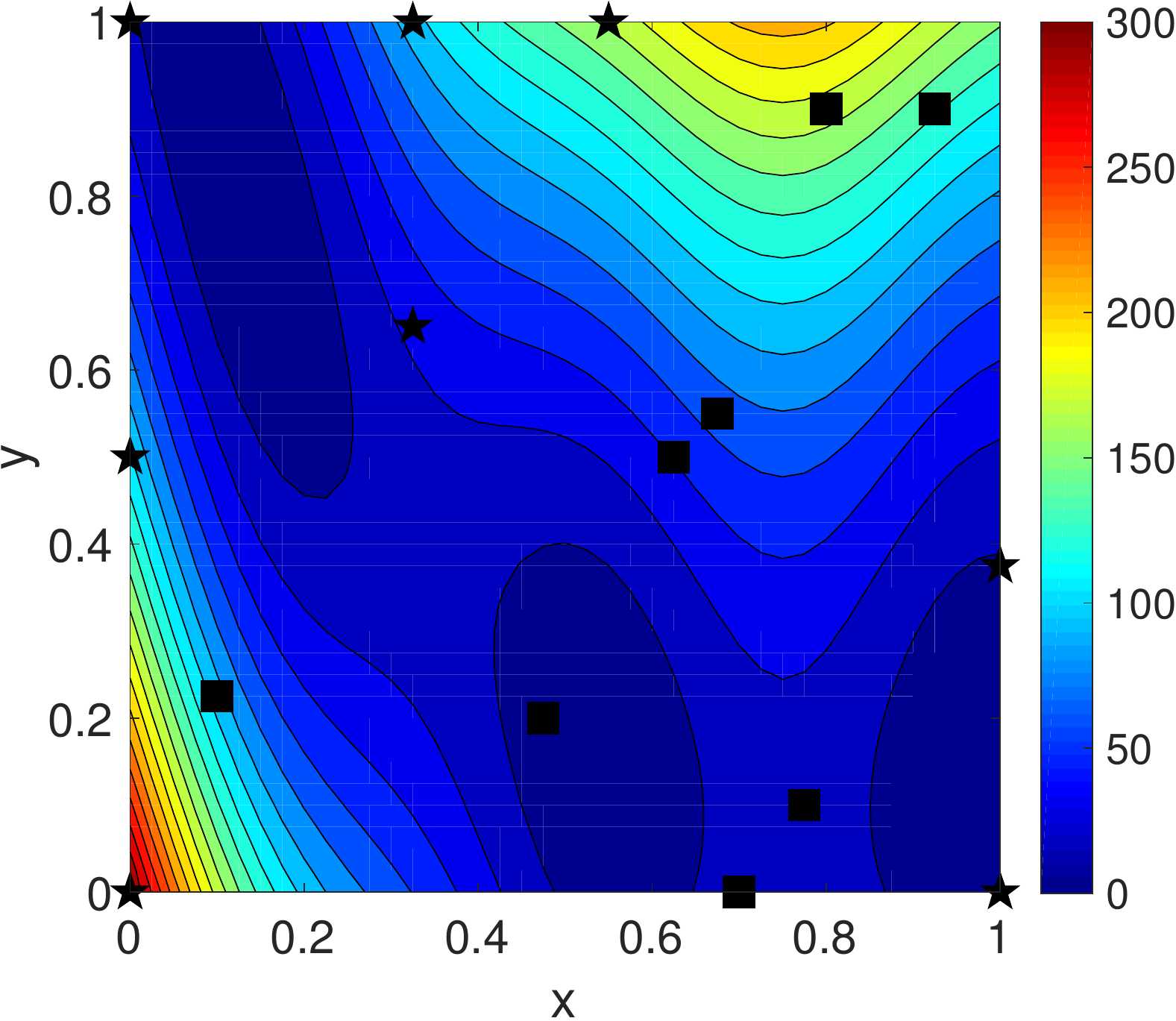}}\qquad
\subfigure[CoBiPhIK $\hat s$]{
\includegraphics[height=0.20\textwidth]{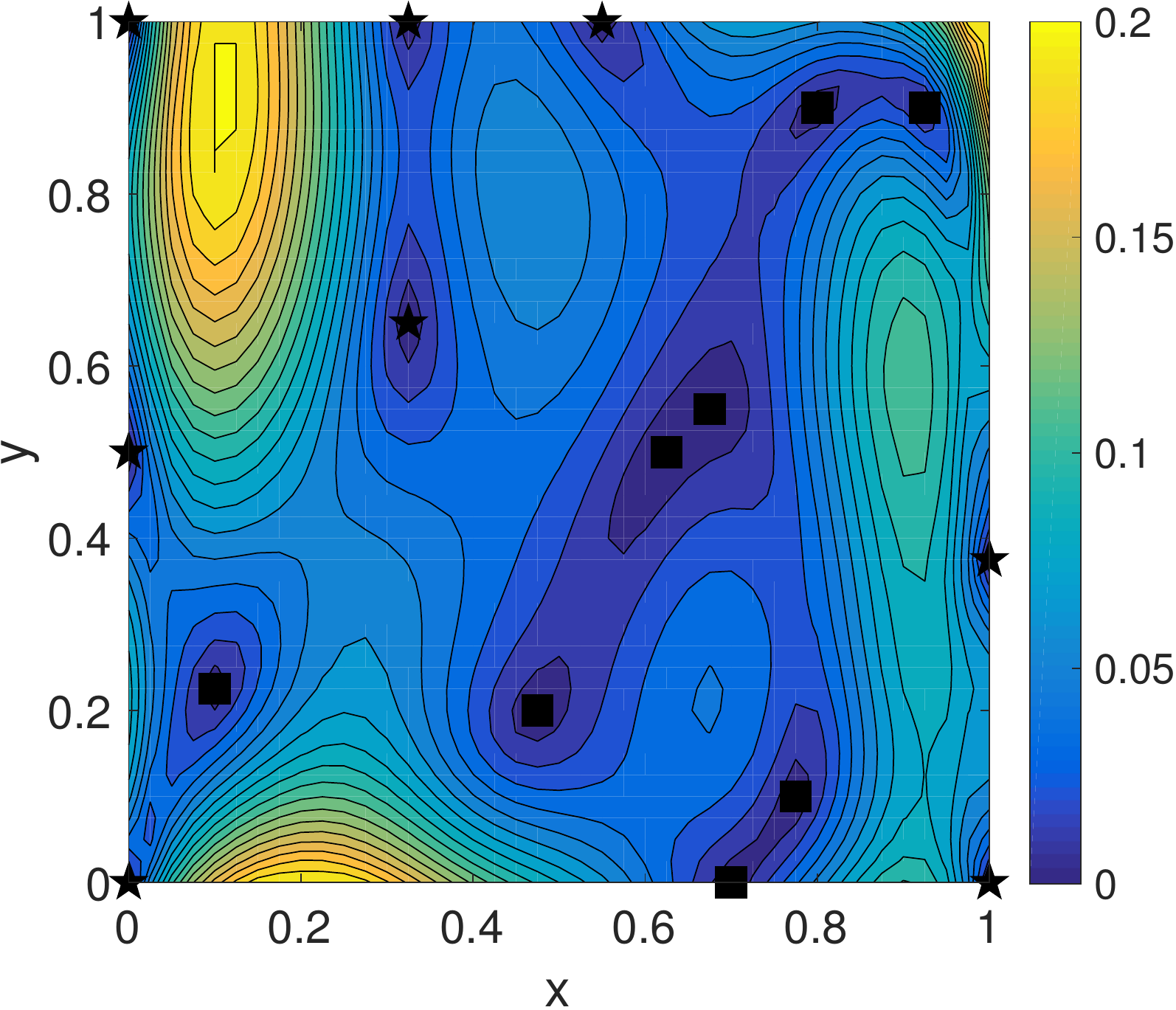}}\qquad
\subfigure[CoBiPhIK $\tensor F_r-\tensor F$]{
\includegraphics[height=0.20\textwidth]{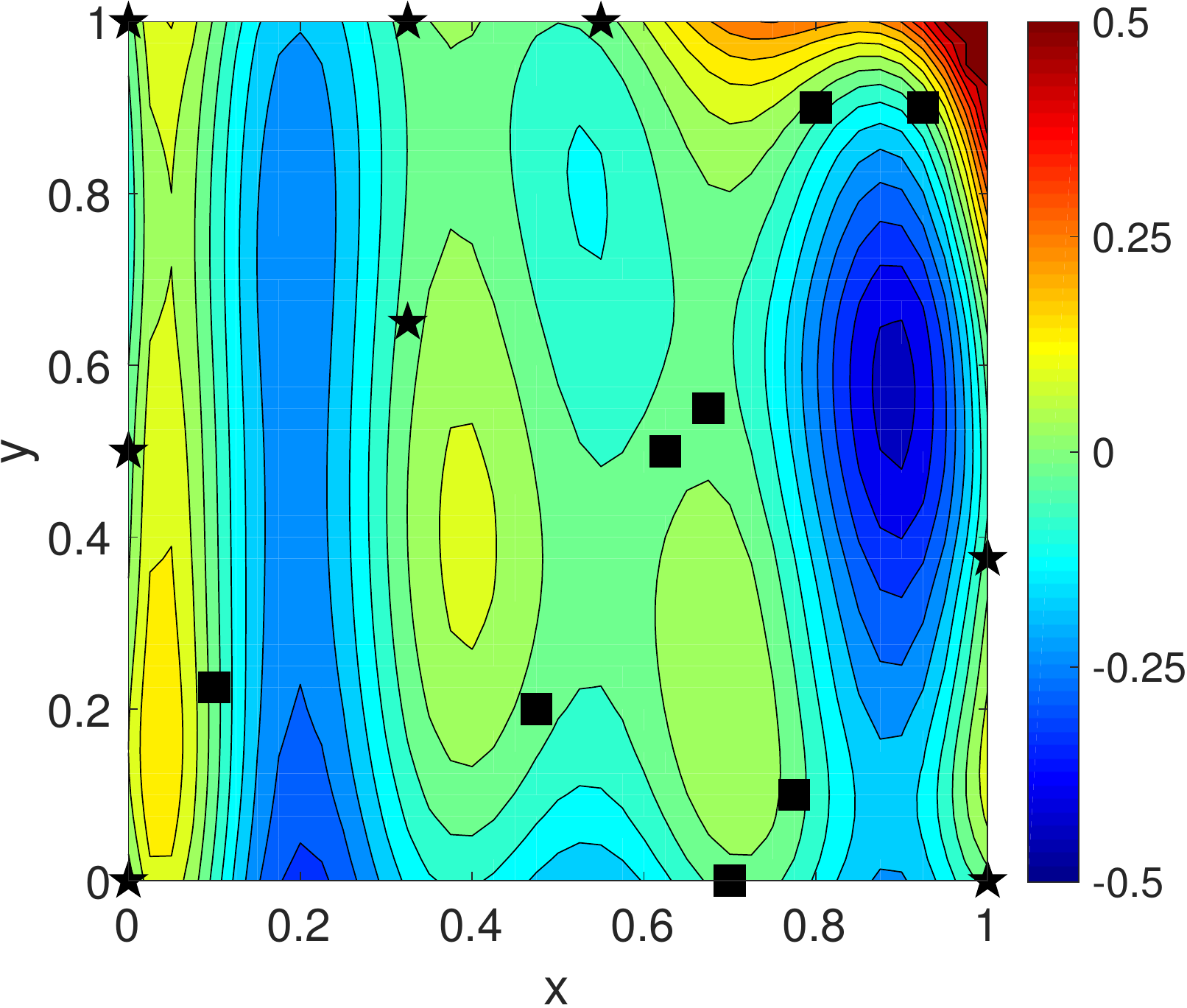}}
\caption{Reconstruction of the modified Branin function by CoPhIK and CoBiPhIK with eight
original observations (squares) and eight additional observations (stars).}
\label{fig:branin_bphicok_act}
\end{figure}

Fig.~\ref{fig:branin_rel_err} presents a quantitative study of the difference 
between the posterior mean and the reference solution with respect to the total
number of observation data. The results are consistent with 
Fig.s~\ref{fig:branin_bphik}-\ref{fig:branin_bphicok_act} in that the 
difference between PhIK and BiPhIK is very small while the difference between 
CoPhIK and CoBiPhIK is larger. We note that the latter is still very small 
ranging from $\mathcal{O}(10^{-3})$ to $\mathcal{O}(10^{-2})$ depending on the
number of observations. This is because $u_B(\Gamma)$ approximates $u_H(\Gamma)$
very well in this case. Specifically, $\delta_1=0.0279$ and $\delta_2=0.0012$.
\begin{figure}[!h]
\centering
\includegraphics[width=0.40\textwidth]{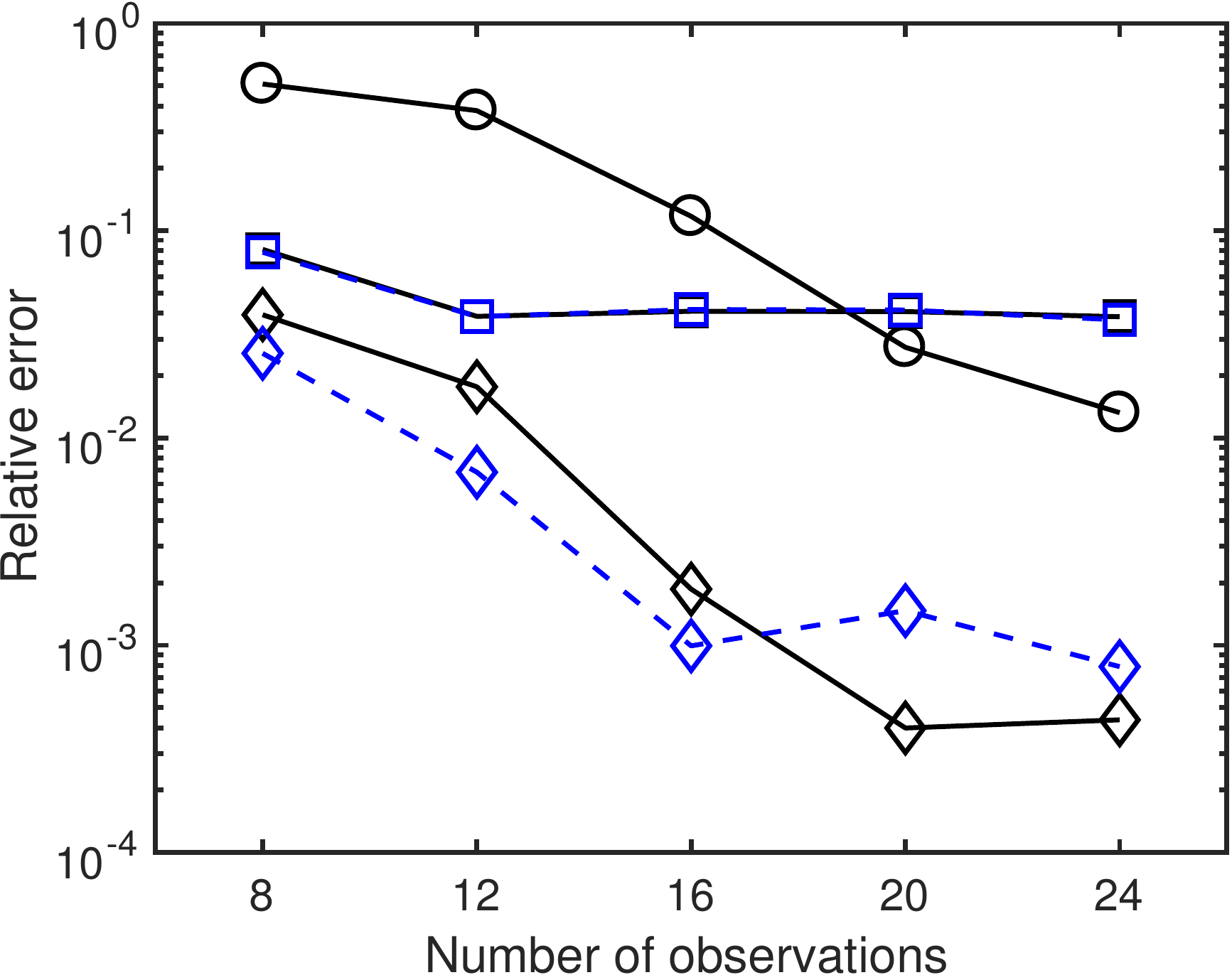}
\caption{Relative error of reconstructed modified Branin function 
  $\Vert\bm F_r-\bm F\Vert_F/\Vert\bm F\Vert_F$ using Kriging (``$\,\circ$"), 
  PhIK (blue ``$\,\square$"), BiPhIK (black ``$\,\square$"), CoPhIK 
  (blue ``$\,\diamond$") and CoBiPhIK (black ``$\,\diamond$") with different numbers
  of total observations via active learning.}
\label{fig:branin_rel_err}
\end{figure}


\subsection{Heat transfer}

In the second example, we consider the steady state of a heat transfer problem.
The nondimesionalized heat equation is given as
\begin{equation}
  \label{eq:heat}
  \dfrac{\partial T}{\partial t} - \nabla\cdot(\kappa(T) \nabla T) = 0, \quad
  \bm x\in D,
\end{equation}
where $T(\bm x, t)$ is the temperature, and the heat conductivity $\kappa$ is 
set as a function of $T$. The computational domain $D$ is a rectangule
$[-0.5, 0.5]\times [-0.2, 0.2]$ with two circular cavities $R_1(O_1, r_1)$ and 
$R_2(O_2, r_2)$, where $O_1=(-0.3,0), O_2=(0.2,0), r_1=0.1, r_2=0.15$ (see 
Fig.~\ref{fig:heat_truth}). The boundary conditions are given as follows:
\begin{equation}
\begin{dcases}
    T = -30\cos(2\pi x)+40, & x\in\Gamma_1; \\
    \dfrac{\partial T}{\partial\bm n} = -20, & x\in \Gamma_2;  \\
    T = 30\cos(2\pi(x+0.1))+40, & x\in\Gamma_3; \\
    \dfrac{\partial T}{\partial\bm n} = 20, & x\in \Gamma_4; \\
    \dfrac{\partial T}{\partial\bm n} = 0, & x\in \Gamma_5. 
\end{dcases}
\end{equation}
The ``real" conductivity is set as 
\begin{equation}
  \label{eq:heat_cond_true}
 \kappa(T)=1.0+\exp(0.02T),
\end{equation}
and the profile of the steady state temperature is presented in 
Fig.~\ref{fig:heat_truth}. This solution is obtained by the finite element 
method with unstructured triangular mesh using MATLAB PDE toolbox, and the 
degree of freedom (DOF) is $1319$ (maximum grid size is $0.02$). The observations of
this exact profile (denoted as $\tensor F$) are collected at six locations 
$\{(-0.4,\pm 0.1), (-0.05,\pm 0.1), (0.4,\pm 0.1)\}$ (black squares in 
Fig.~\ref{fig:heat_truth}).
\begin{figure}[!h]
\centering
\includegraphics[width=0.45\textwidth]{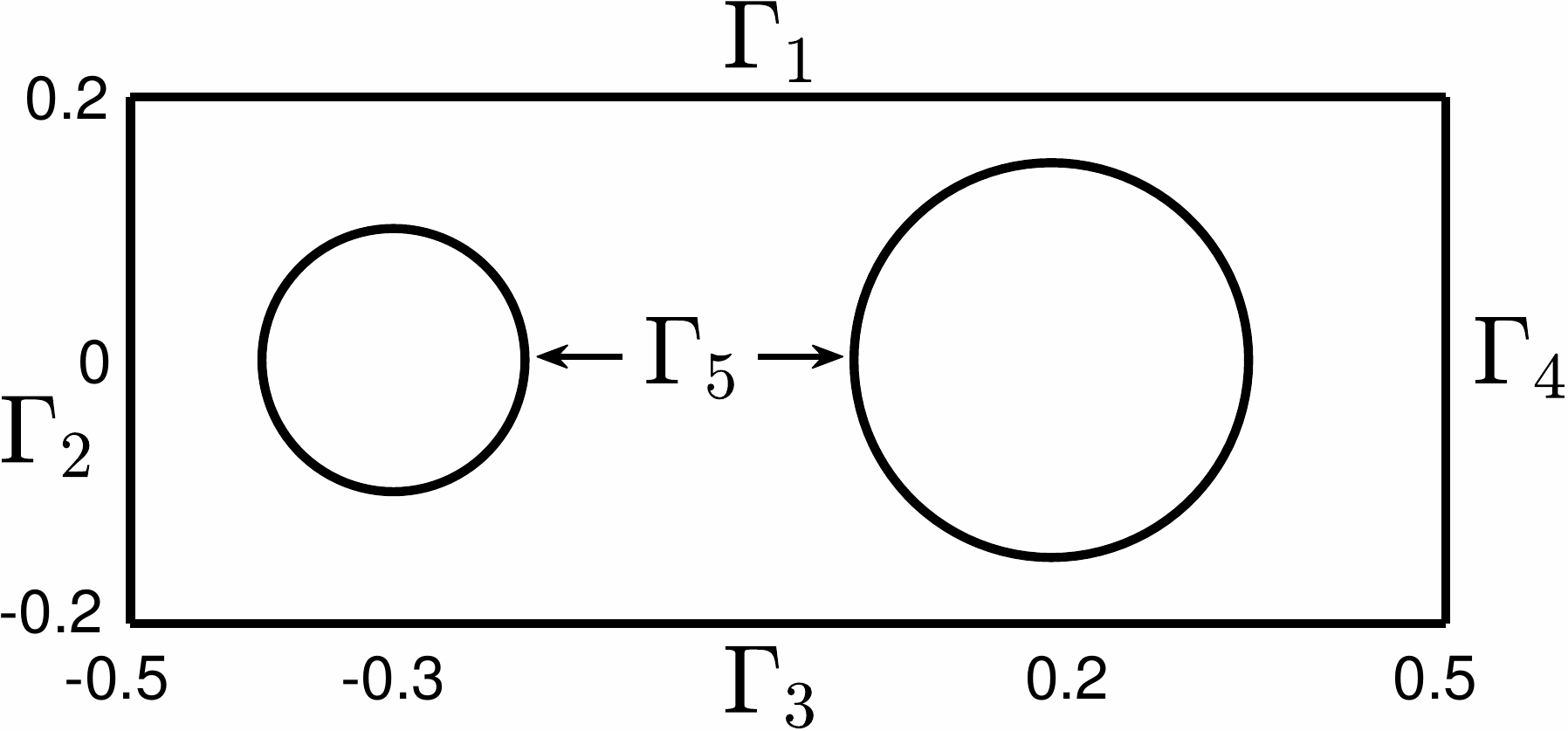}\quad
\includegraphics[width=0.45\textwidth]{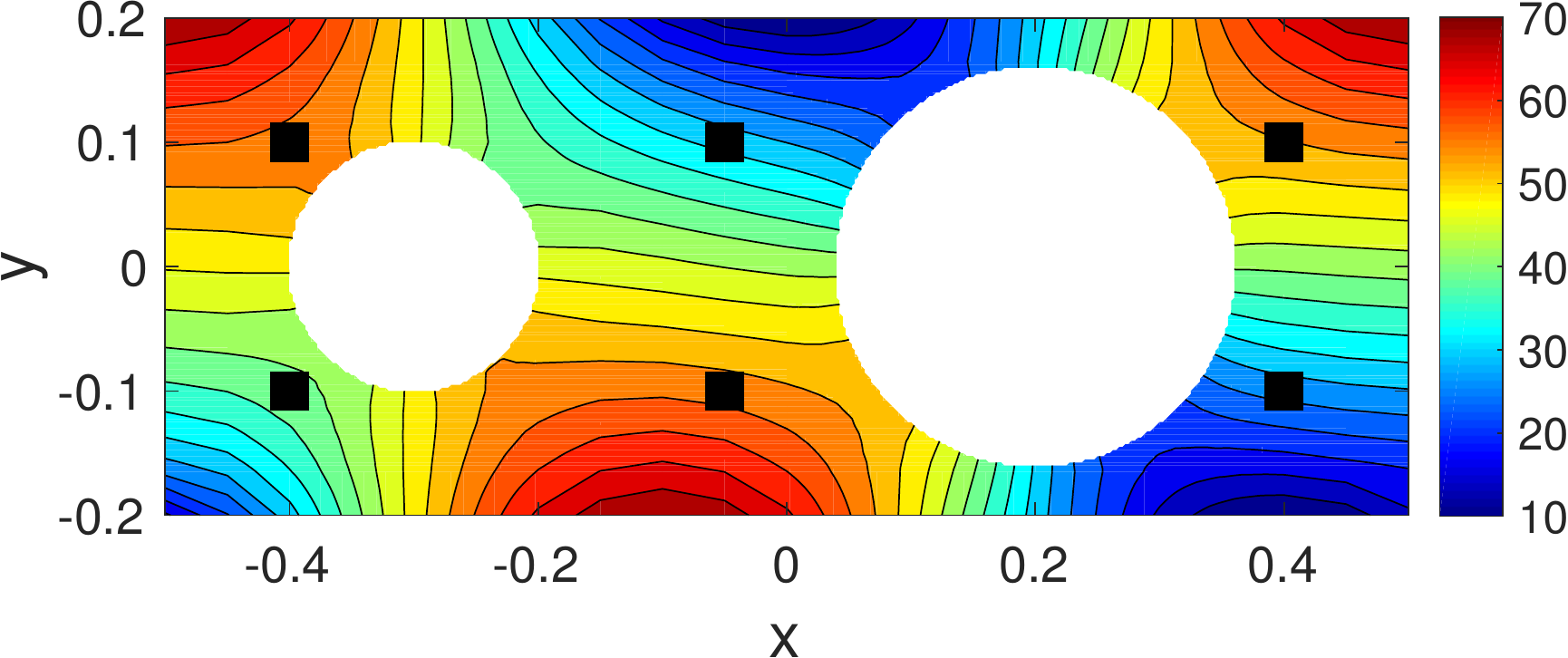}
\caption{Heat transfer problem. Left: computational domain; right: contours of
  steady state solution and locations of six observations (black squares).}
\label{fig:heat_truth}
\end{figure}

Now we assume that due to the lack of knowledge, the conductivity is modeled as 
\begin{equation}
  \label{eq:heat_cond_model}
  \kappa(T;\omega) = 0.1+\xi T, 
\end{equation}
where $\xi(\omega)$ is a uniform random variable $\mathcal{U}[0.0012,0.0108]$.
Apparently, this physical model significantly \emph{underestimates} the heat 
conductivity, and its form is incorrect. We generate $M=400$ samples of
$\xi(\omega)$ and solve Eq.~\eqref{eq:heat} on a coarser grid (maximum grid size
is $0.1$) with $\text{DOF}=96$ to obtain corresponding temperature solutions 
which forms $u_L(\Gamma)$. We set $M_H=19$ in this example.

It is shown in~\cite{YangBTT18}, the Kriging reconstruction is not accurate
because of the selection of observations and the property of the exact solution.
Fig.~\ref{fig:heat_bphik} presents the results by BiPhIK and CoBiPhIK, and
they are very similar to the results by PhIK and CoPhIK in~\cite{YangBTT18} (not
shown here), respectively. These results are better than Kriging (see
Fig.~\ref{fig:heat_rel_err} for quantitative comparison).
\begin{figure}[!h]
\centering
\subfigure[BiPhIK $\tensor F_r$]{
\includegraphics[height=0.13\textwidth]{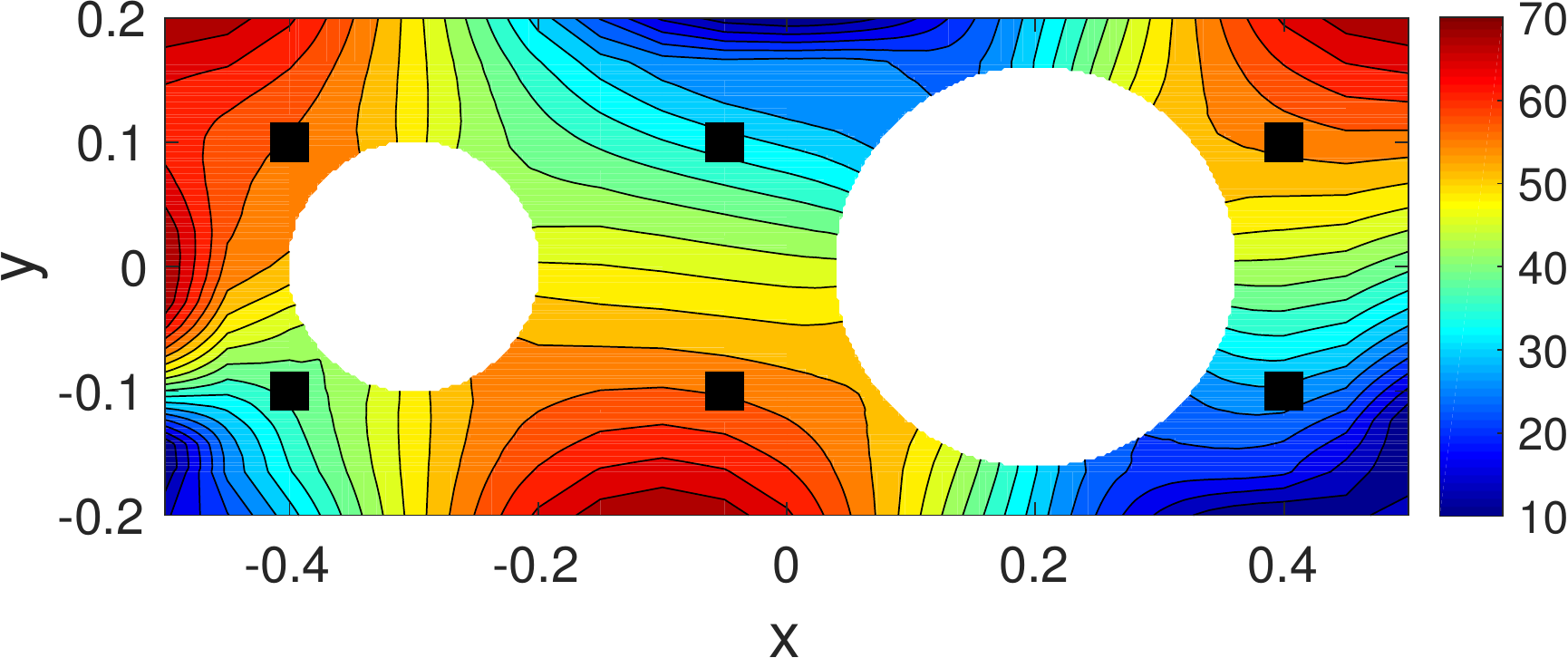}}
\subfigure[BiPhIK $\hat s$]{
\includegraphics[height=0.14\textwidth]{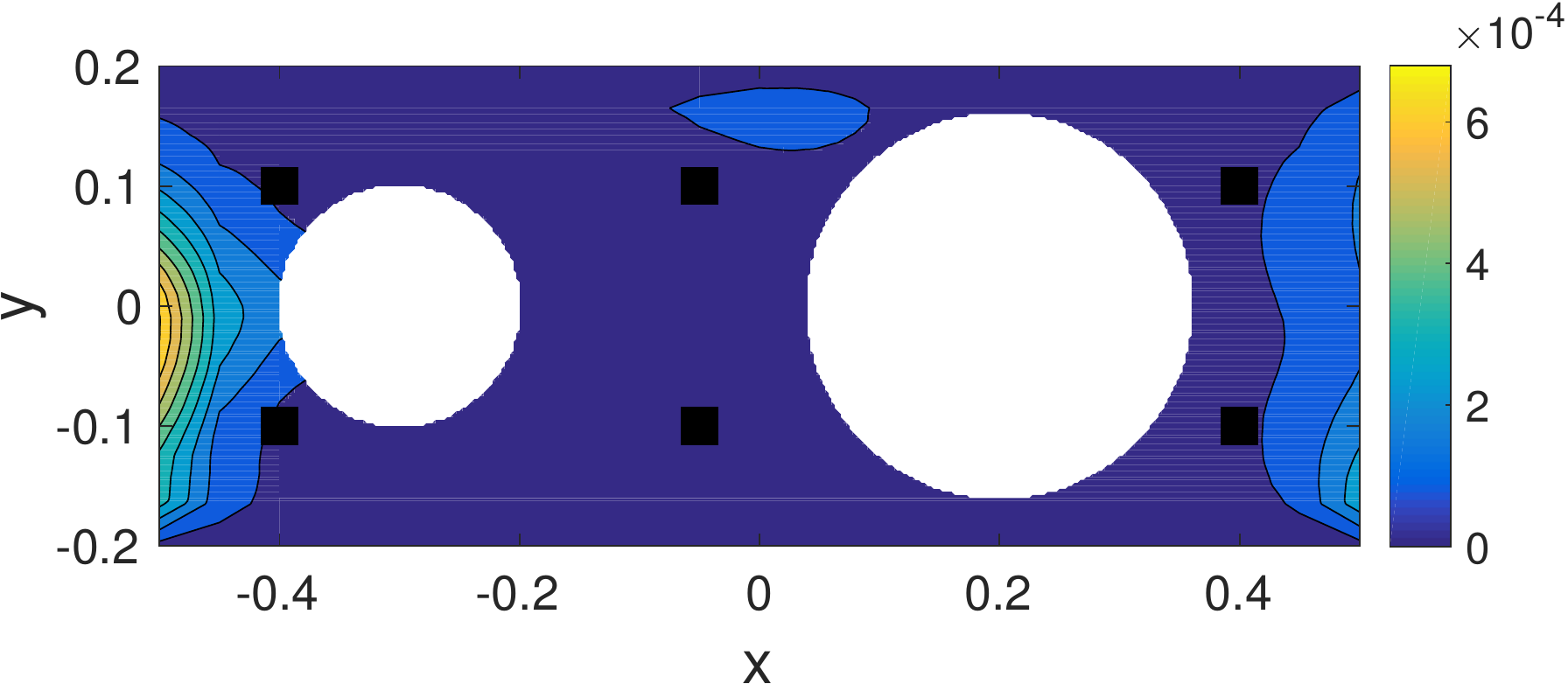}}
\subfigure[BiPhIK $\tensor F_r-\tensor F$]{
\includegraphics[height=0.13\textwidth]{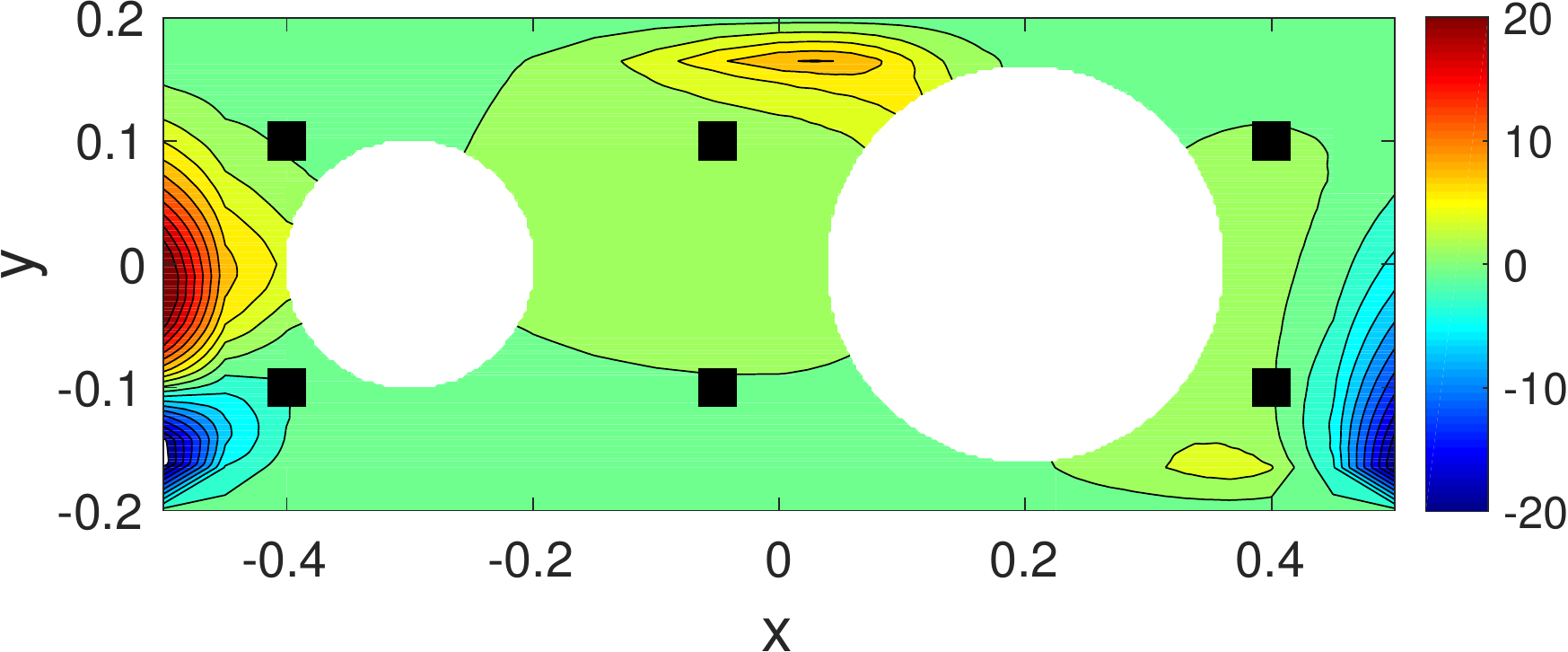}} \\
\subfigure[CoBiPhIK $\tensor F_r$]{
\includegraphics[height=0.13\textwidth]{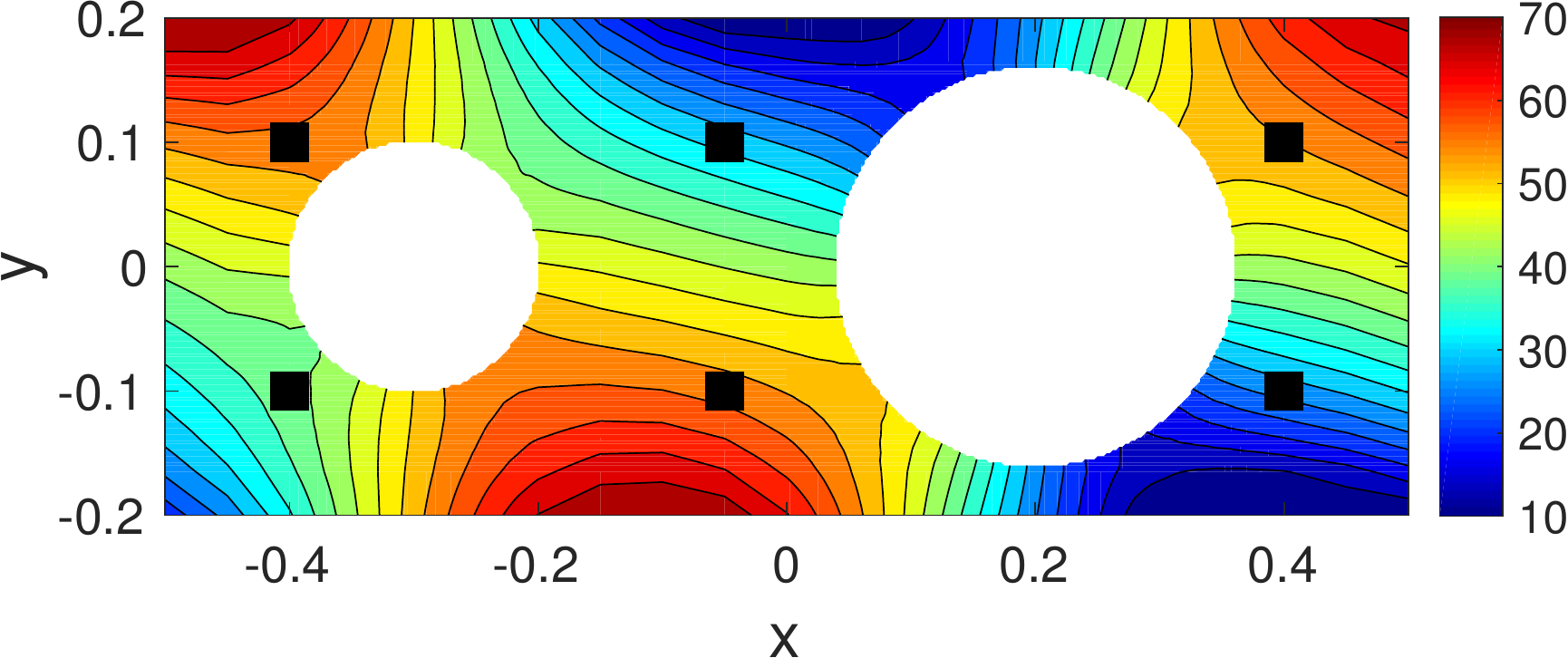}} 
\subfigure[CoBiPhIK $\hat s$]{
\includegraphics[height=0.13\textwidth]{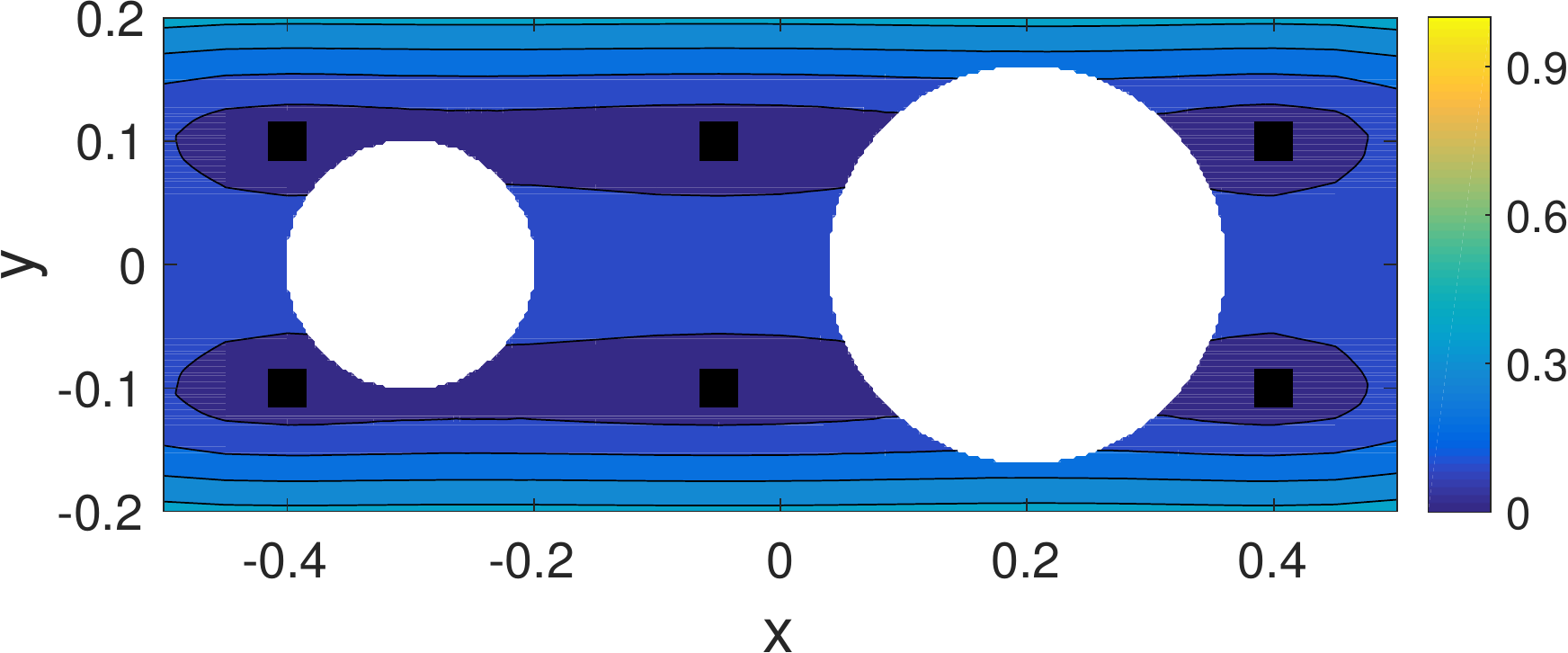}}~
\subfigure[CoBiPhIK $\tensor F_r-\tensor F$]{
\includegraphics[height=0.13\textwidth]{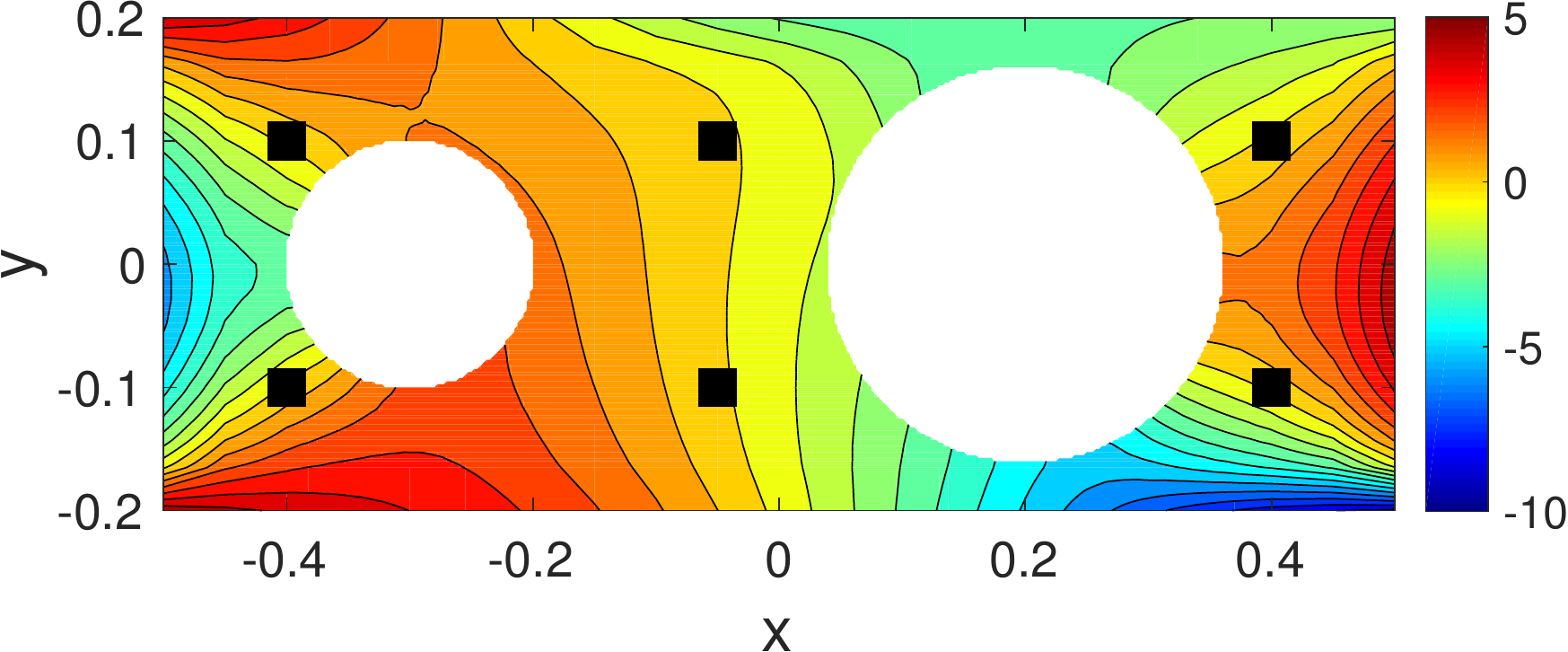}}
\caption{Reconstruction of the steady state solution of heat transfer problem
  by BiPhIK (first row) and CoBiPhIK (second row) with eight original 
  observations (squares).}
\label{fig:heat_bphik}
\end{figure}

Next, we adaptively add more observation data one by one. The results by PhIK 
and BiPhIK are very similar. We only present $\tensor F_r$ by these two methods
in Fig.~\ref{fig:heat_bphik_act} for comparison. We can see that the discrepancy
in $\tensor F_r$ is very insignificant, and there is only a slight difference in 
the locations of new observations (marked as starts) on the boundary $\Gamma_2$.
Fig.~\ref{fig:heat_bphicok_act} compares the results by CoPhIK and CoBiPhIK.
Although $\tensor F_r$ and $\tensor F_r-\tensor F$ are similar, there is
significant discrepancy in $\hat s$, mainly because the correlation length $l_i$
in the $Y_{_d}$'s kernel function are different.
\begin{figure}[!h]
\centering
\subfigure[PhIK $\tensor F_r$]{
\includegraphics[height=0.13\textwidth]{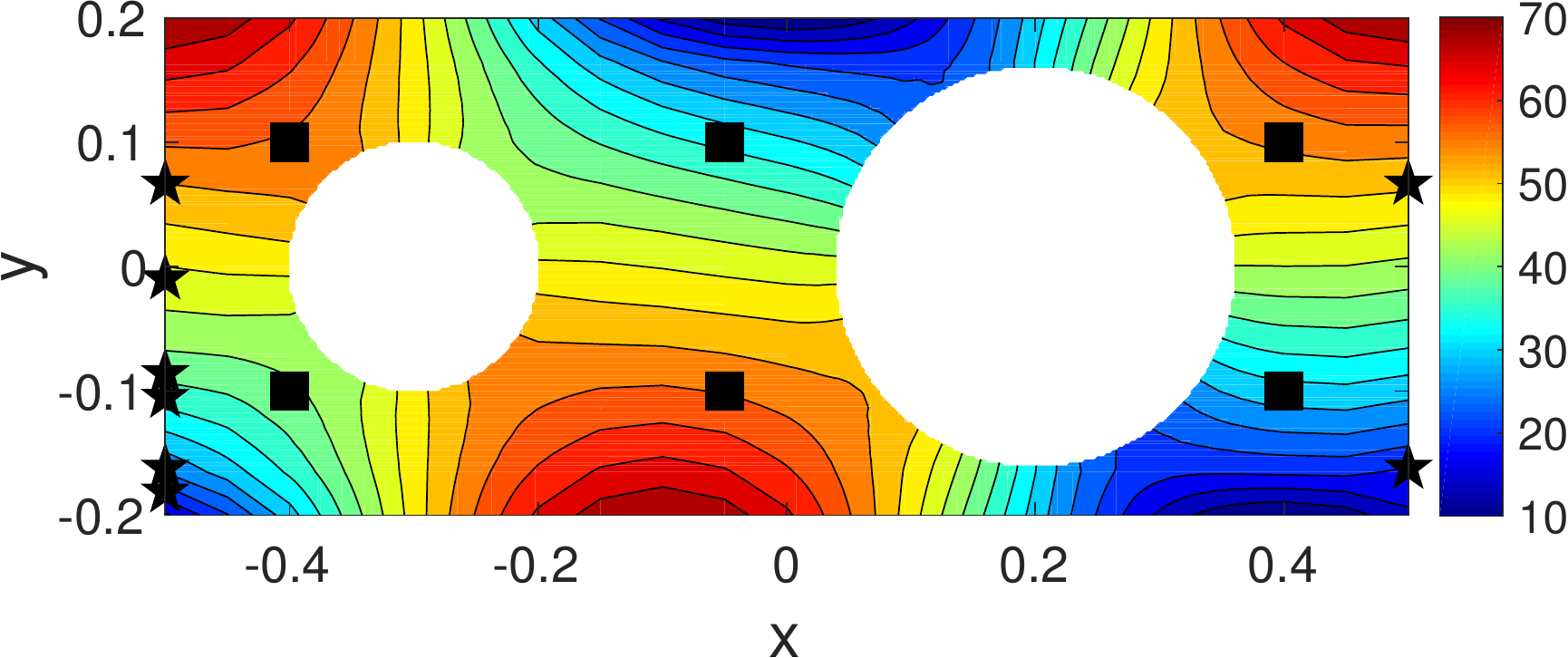}}\qquad
\subfigure[BiPhIK $\tensor F_r$]{
\includegraphics[height=0.13\textwidth]{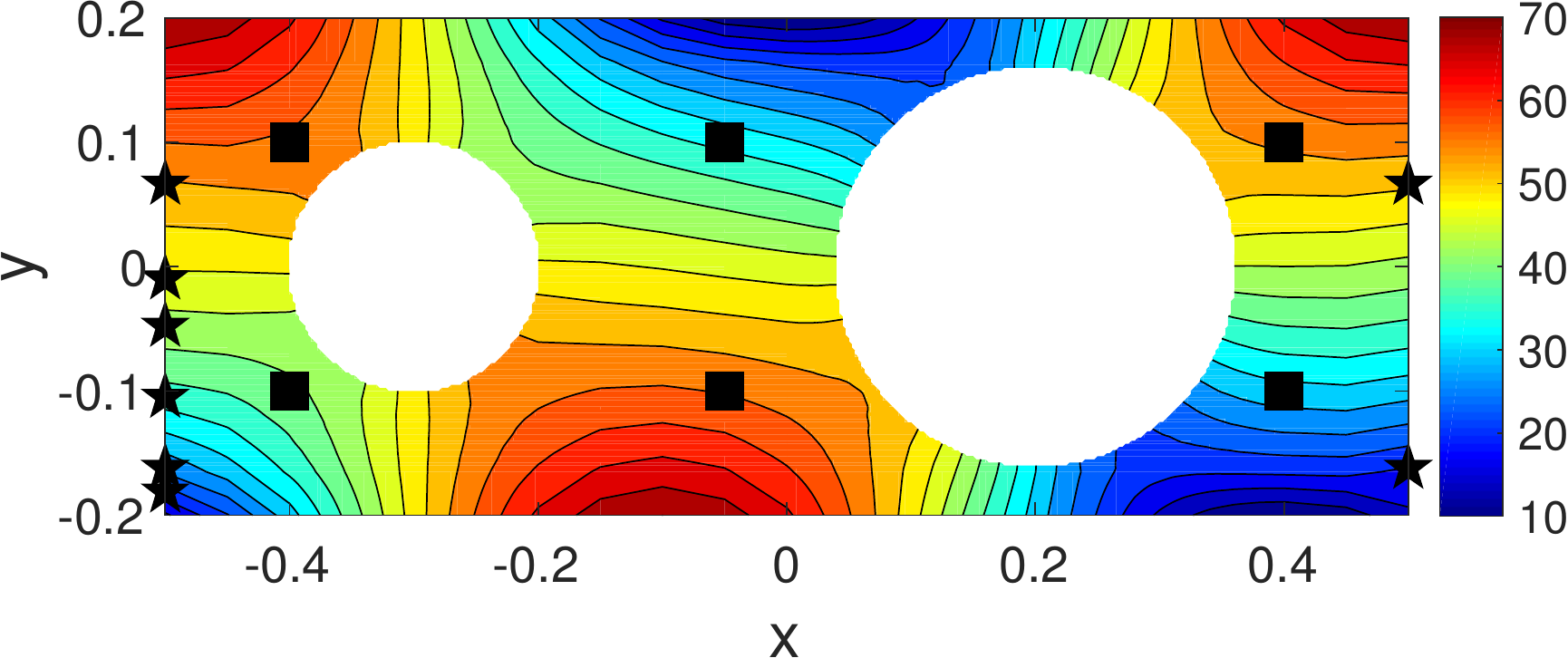}}
\caption{Reconstruction of the steady state solution of heat transfer problem by
  PhIK and BiPhIK with eight original observations (squares) and eight additional 
  observations (stars).} 
  \label{fig:heat_bphik_act}
\end{figure}
\begin{figure}[!h]
\centering
\subfigure[CoPhIK $\tensor F_r$]{
\includegraphics[height=0.13\textwidth]{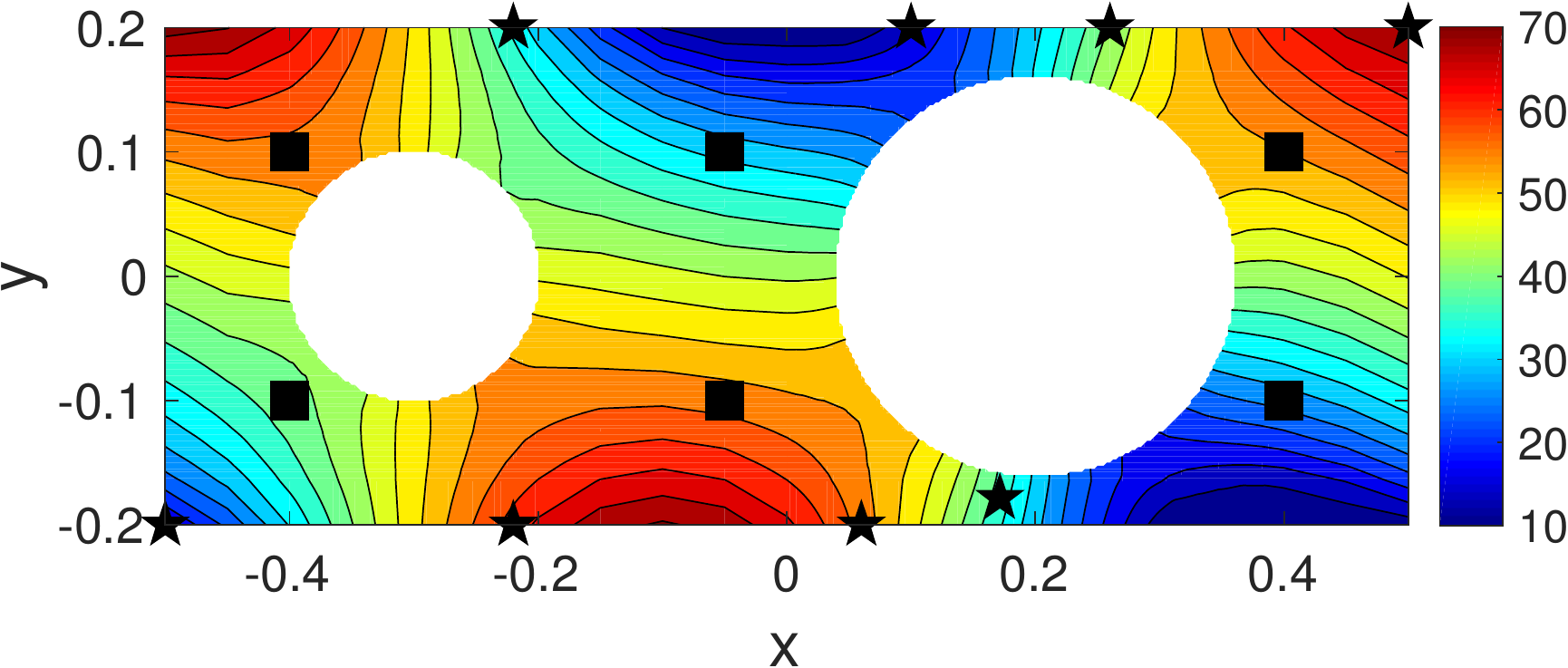}}
\subfigure[CoPhIK $\hat s$]{
\includegraphics[height=0.13\textwidth]{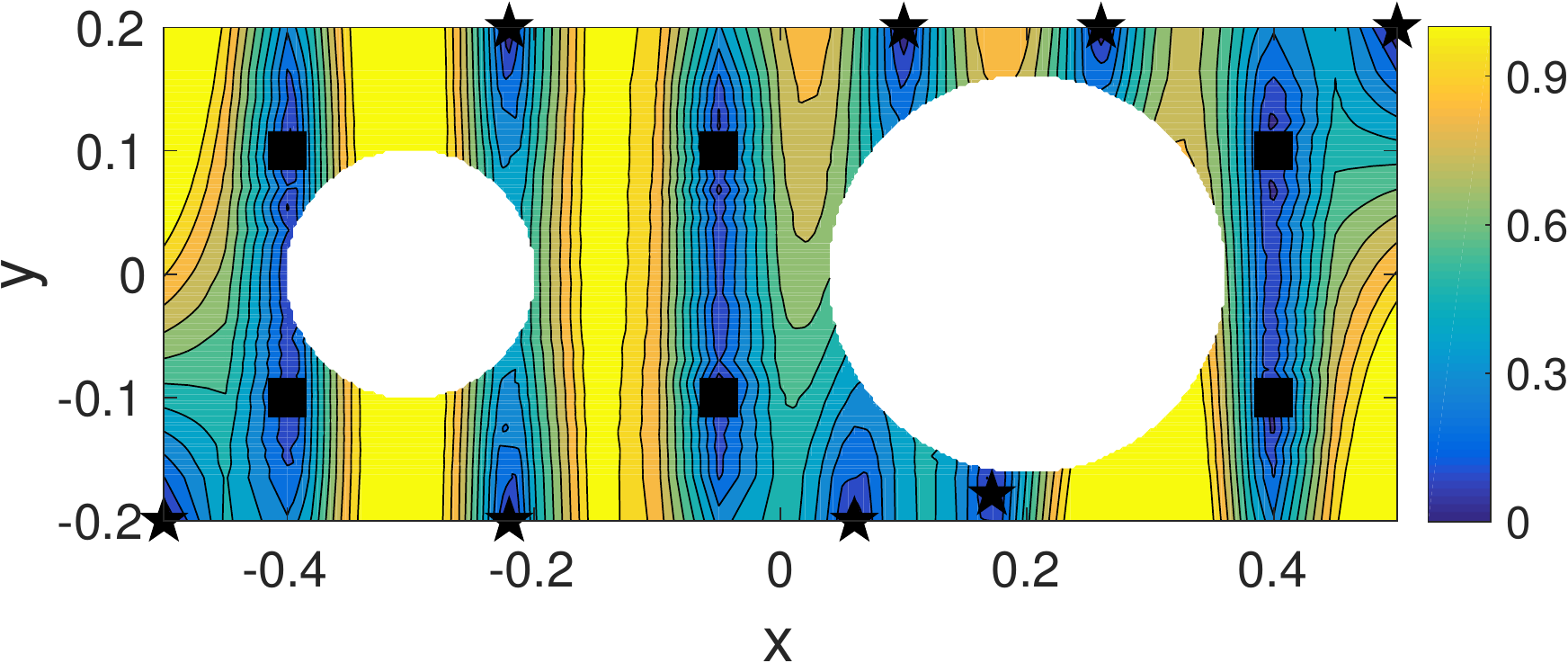}}
\subfigure[CoPhIK $\tensor F_r-\tensor F$]{
\includegraphics[height=0.13\textwidth]{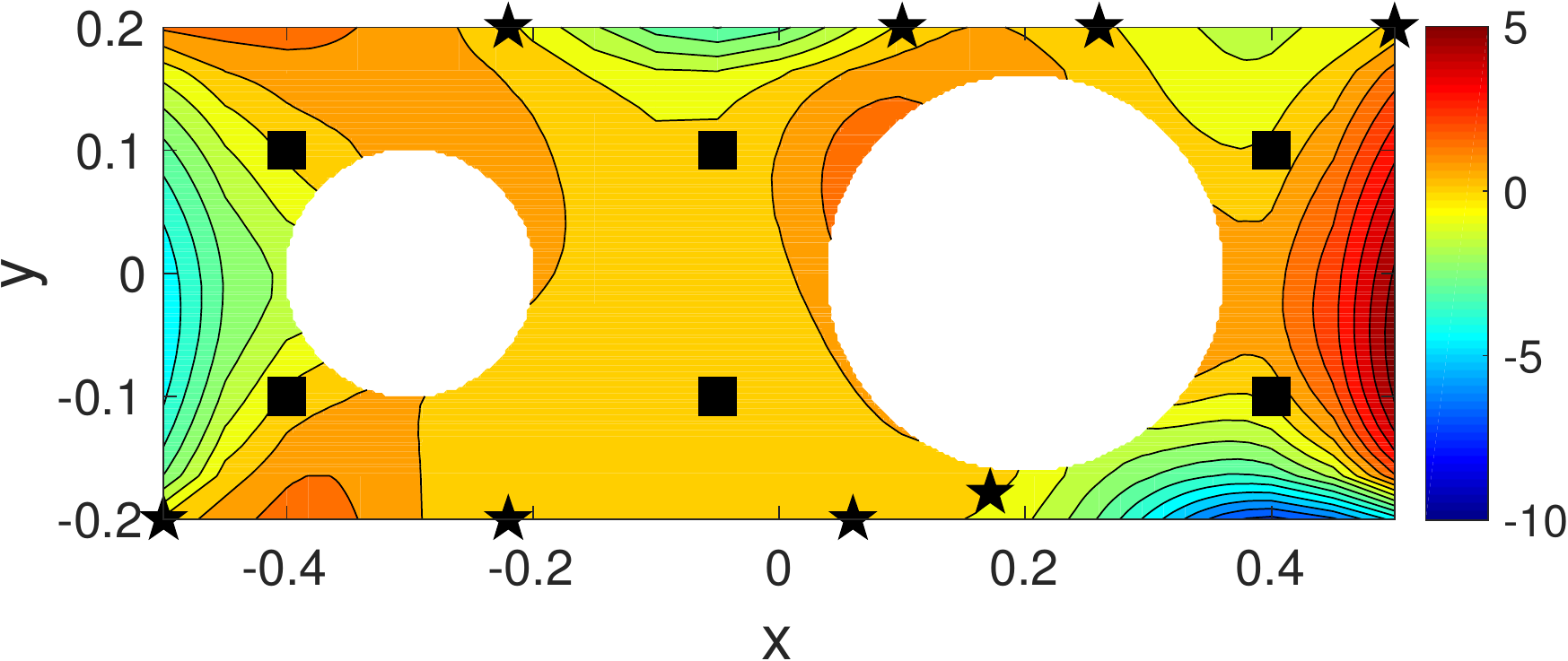}} \\
\subfigure[CoBiPhIK $\tensor F_r$]{
\includegraphics[height=0.13\textwidth]{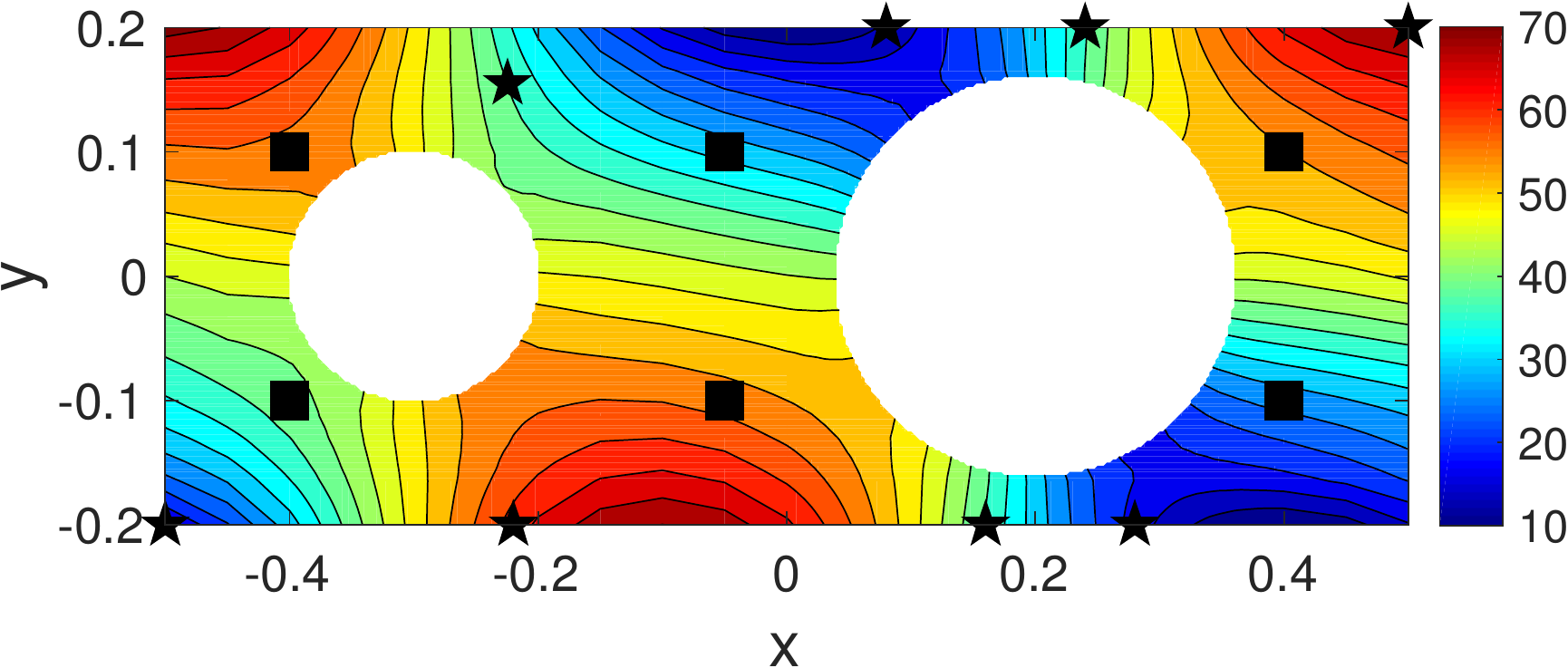}}
\subfigure[CoBiPhIK $\hat s$]{
\includegraphics[height=0.13\textwidth]{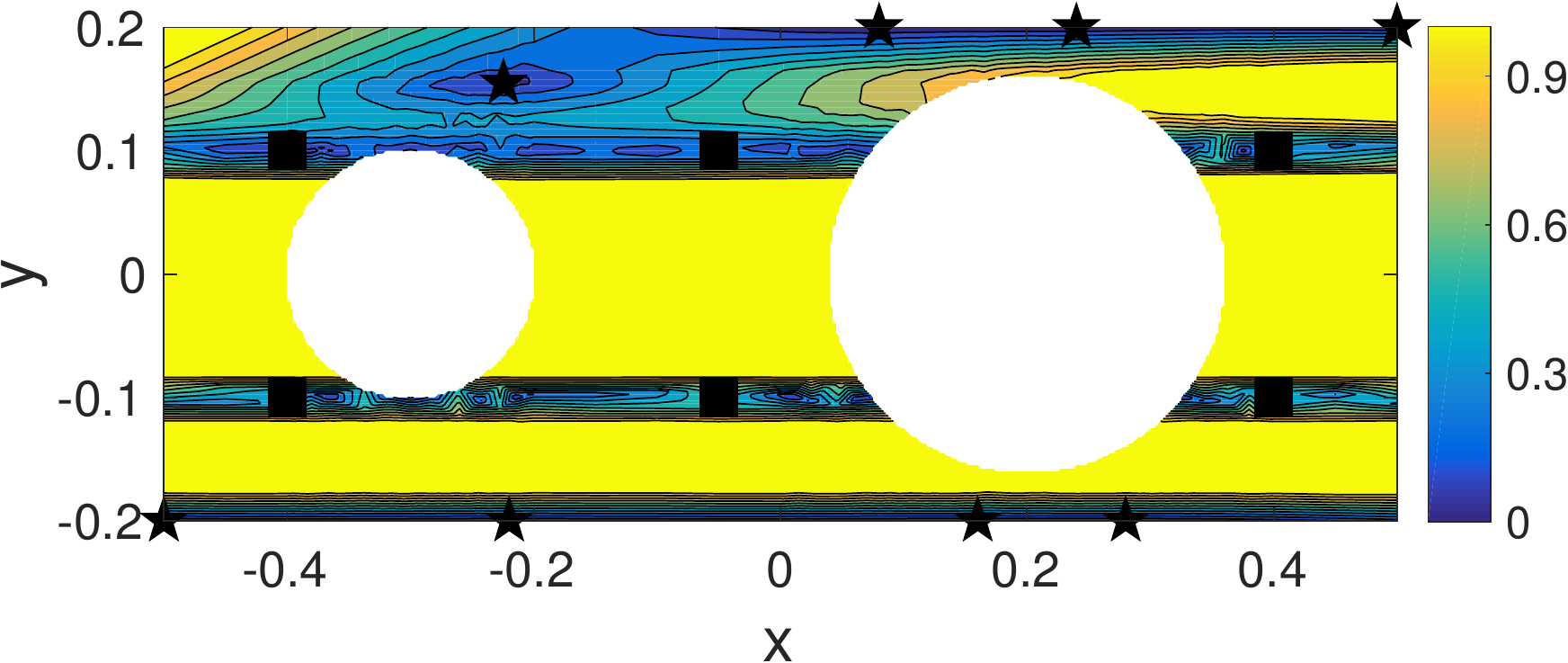}}
\subfigure[CoBiPhIK $\tensor F_r-\tensor F$]{
\includegraphics[height=0.13\textwidth]{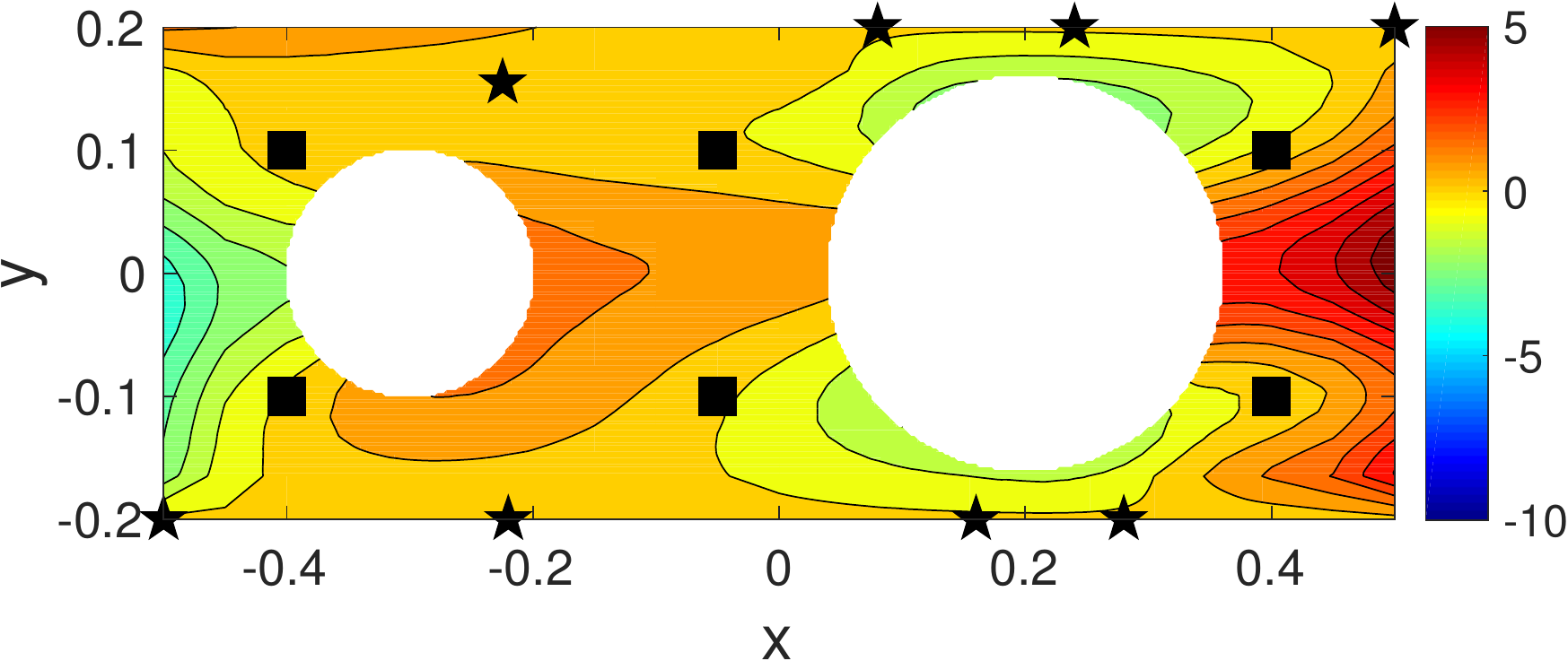}}
\caption{Reconstruction of the steady state solution of heat transfer problem 
  by CoPhIK and CoBiPhIK with eight original observations (squares) and eight additional
  observations (stars).}
\label{fig:heat_bphicok_act}
\end{figure}

Fig.~\ref{fig:heat_rel_err} presents the quantitative comparison of the
relative error by different methods. In this case, the PhIK and BiPhIK results 
are almost the same, and the CoPhIK and CoBiPhIK results are also very similar.
Kriging performs poorly when the number of observations is smaller than $22$,
however it outperforms PhIK (and BiPhIK) when $22$ observations are available.
CoPhIK (and CoBiPhIK) is always better than Kriging and PhIK (and BiPhIK) as
shown in~\cite{YangBTT18}. Again, in this case, $u_B(\Gamma)$ approximates
$u_H(\Gamma)$ very well ($\delta_1=0.0039$ and $\delta_2=0.0015$), which yields
a much smaller difference between PhIK and BiPhIK, and between CoPhIK and CoBiPhIK.
\begin{figure}[!h]
\centering
\includegraphics[width=0.40\textwidth]{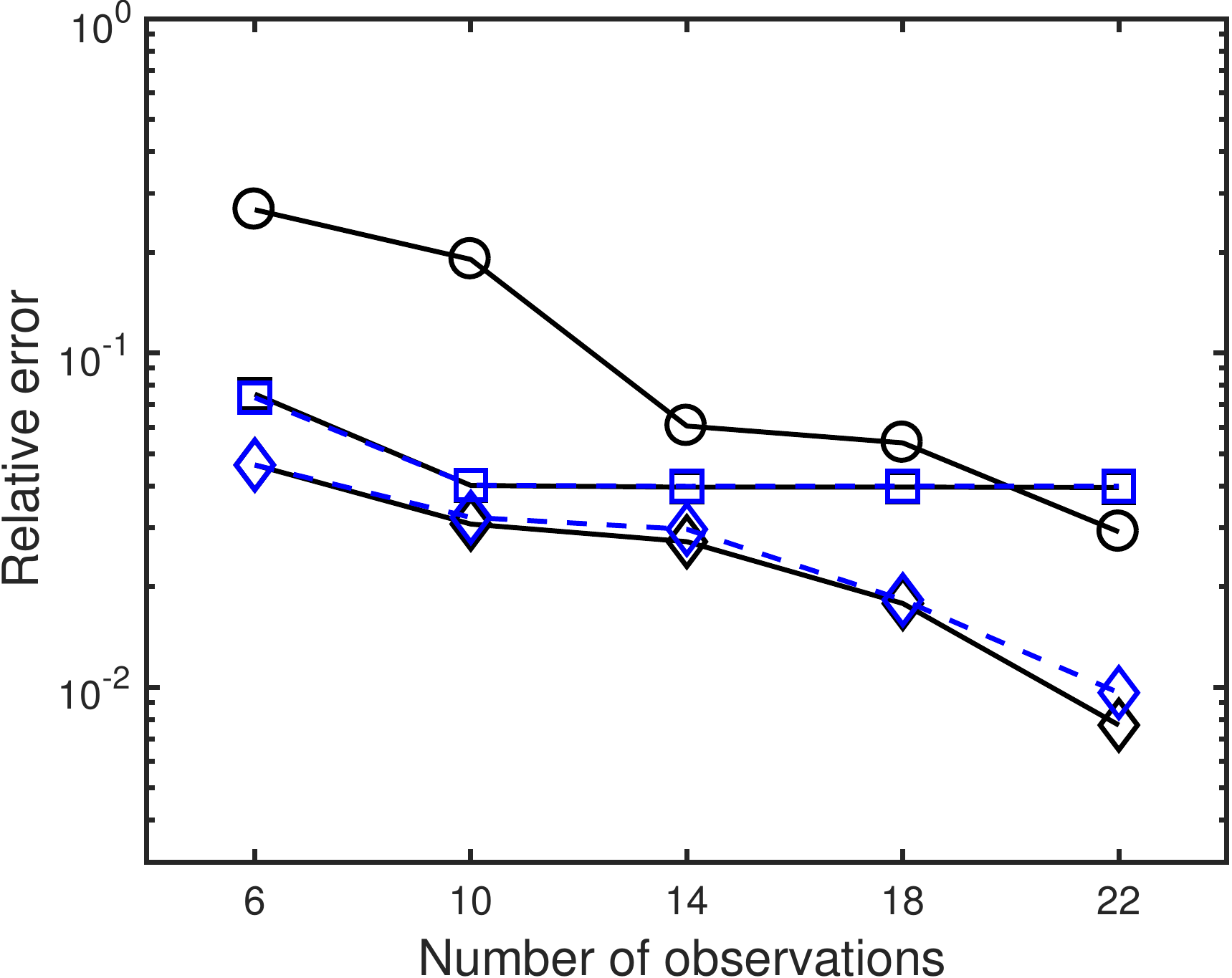}
\caption{Relative error of reconstructed steady state solution of heat transfer
  problem $\Vert\bm F_r-\bm F\Vert_F/\Vert\bm F\Vert_F$ using Kriging (``$\,\circ$"), 
  PhIK (blue ``$\,\square$"), BiPhIK (black ``$\,\square$"), CoPhIK 
  (blue ``$\,\diamond$") and CoBiPhIK (black ``$\,\diamond$") with different numbers
  of total observations via active learning.}
\label{fig:heat_rel_err}
\end{figure}


\subsection{Kuramoto-Sivashinsky equation}
\label{subsec:KS}

As a final example, we consider the one-dimensional Kuramoto-Sivashinsy (KS) 
equation~\cite{kuramoto1978diffusion,sivashinsky1980flame}:
\begin{equation}
  \begin{split}
 & u_t+4u_{xxxx}+\alpha\left[u_{xx}+\dfrac{1}{2}(u_x)^2\right]=0, \quad 0\leq x\leq 2\pi, \\
 &   u(x+2\pi,t) = u(x, t),  \\
 & u(x, 0) = u_0(x),
  \end{split}
\end{equation}
where
\begin{equation}
  u_0(x) = 2.9420\cos(2x) + 0.4642\cos(4x) + 0.0410\cos(6x) + 0.0034\cos(8x).
\end{equation}
It is well known that  this equation can be used to depict a chaotic system,
and it is very sensitive to the parameter $\alpha$ when it is large. More
importantly, in numerical simulation, high precision is necessary because of the
extreme sensitivity of the simulations with respect to numerical
accuracy~\cite{hyman1986kuramoto, hyman1986order}. We use the spectral method 
for spacial derivatives as in~\cite{li2016unified}. Specifically, we use Fourier
expansion with $256$ terms to obtain the reference solution and $u_H$, and use 
expansion with $128$ terms to compute $u_L$. For the time integration, we use a
fourth-order Runge-Kutta method~\cite{gottlieb2001strong} with time step 
$10^{-3}$. We investigate the 
solution of a KS equation at $T=5$, and the ``exact" $\alpha=37.545$. Accurate 
observations are available at 
\[x=\dfrac{12\pi}{256}+\dfrac{56\pi}{256}\cdot j, \quad j=0,1,\dotsc,8.\]
We assume that we do not know the exact $\alpha$, and use the
biased ``domain knowledge" to set $\alpha$ as a uniform random variable
$\mathcal{U}[30,36]$. Apparently, this range is below the exact $\alpha$. We
generate $400$ samples of $\alpha$, and compute corresponding $u_H^m$ and 
$u_L^m$ to compare the performance of different methods. Specifically, the
we set $M_H=17$ to construct $u_B(\Gamma)$.

Fig.~\ref{fig:ks_truth} illustrates the reference solution and the locations of
accurate observations. The mean of the $\{u_H^m\}_{m=1}^{400}$ is illustrated as 
the dashed line, which deviates from the exact solution significantly with the
relative $L_2$ error more than $140\%$. We also compute the standard deviation
of $u_H$ at each $\bm x$. Since we will use statistics of $u_H$ to
construct a GP in PhIK, we present the ``$95\%$ confidence interval", i.e.,
mean plus minus two standard deviations in Fig.~\ref{fig:ks_truth}. We note
that this is not the exact confidence interval of the ensemble 
$\{u_H^m\}_{m=1}^{400}$ itself. Fig.~\ref{fig:ks_truth} shows clearly that the
exact solution is not bounded by the confidence interval. This is because in the
stochastic model, the $\alpha$ is below its exact value, and the KS equation is
very sensitive to $\alpha$.
\begin{figure}[!h]
  \centering
  \includegraphics[width=0.40\textwidth]{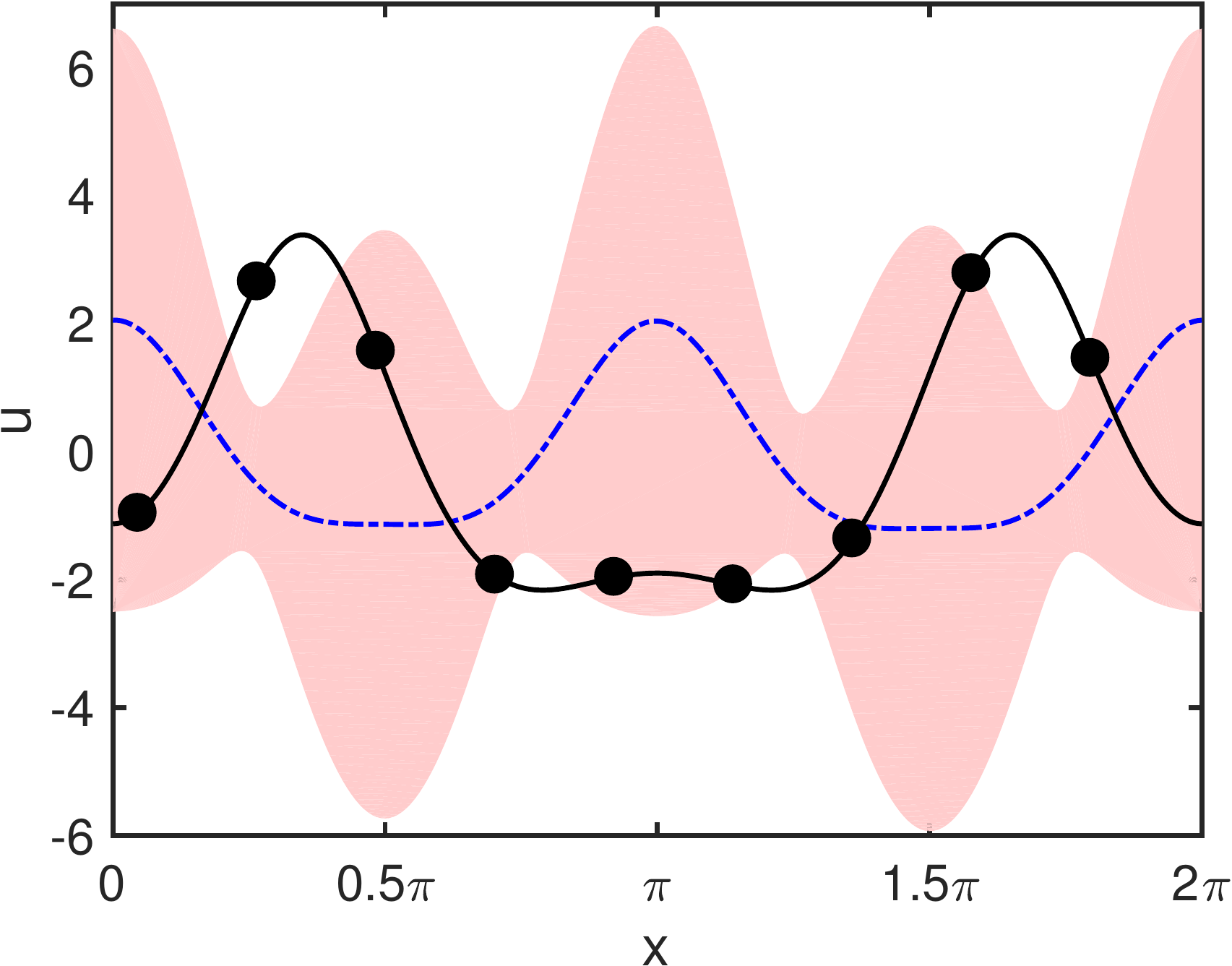}
  \caption{Exact solution of the KS equation (solid line),
  observations (circles), mean of the realizations (dashed line) and the
  ``$\,95\%$ confidence interval" (shaded area).}
  \label{fig:ks_truth}
\end{figure}

Fig.~\ref{fig:ks_kriging} illustrates the Kriging results by using the nine
accurate observations. The reconstruction is less accurate near the right
end than the other regions, and the uncertainty is larger (i.e., the confidence 
interval is wider). This is because there is no observation near the right end
compared with other regions. Also, the periodic boundary condition is not
preserved. Fig.~\ref{fig:ks_phik} shows the results by PhIK (using 
$\{u_H^m\}_{m=1}^{400}$) and BiPhIK (using $\{u_B^m\}_{m=1}^{400}$). The PhIK
performs better than BiPhIK, and, more importantly, the periodic boundary 
condition is preserved by PhIK, and slightly violated by BiPhIK. We note that
in this figure, the confidence intervals in both methods are very narrow
($\sim\mathcal{O}(10^{-2})$). Similar to the examples in~\cite{YangTT18,
YangBTT18} and the other two examples in this session, PhIK usually yields a
less uncertain result, but this result may not be very accurate because the
uncertainty estimate of PhIK relies on its prior covariance, which is totally
dependent on the stochasticity of the physical model. Fig.~\ref{fig:ks_phicok}
demonstrates the results by CoPhIK and CoBiPhIK. The CoPhIK is the most accurate
of all the methods (smallest discrepancy between posterior mean and the exact
accuracy) with confidence intervals that cover the reference solution, and the
periodic boundary condition is slightly violated. The CoBiPhIK is the least 
accurate among all the methods. 
\begin{figure}[!h]
  \centering
  \includegraphics[width=0.40\textwidth]{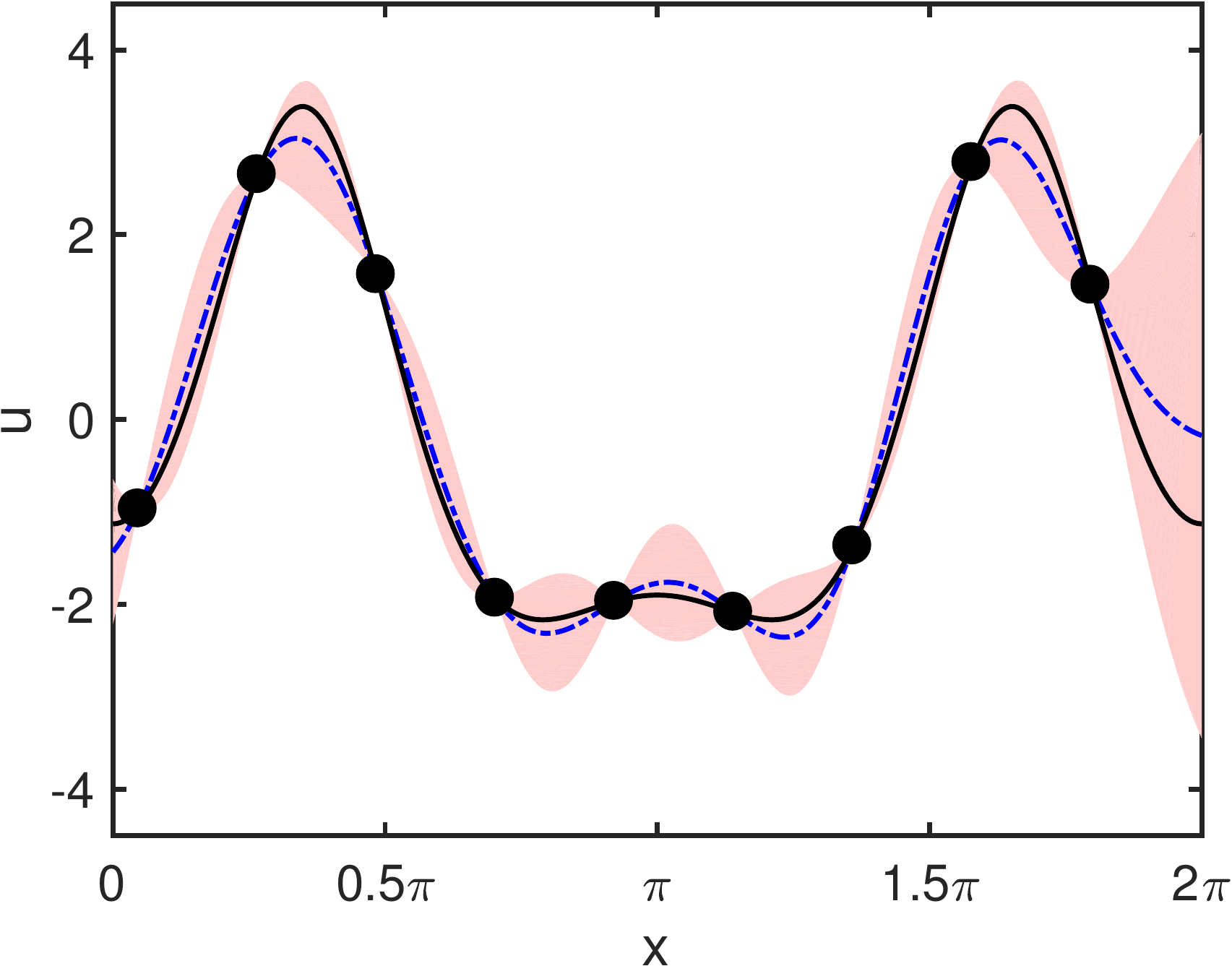}
  \caption{Kriging reconstruction of the KS equation: posterior mean (dashed
    line), exact solution (solid line), observations (circles),  and
    $95\%$ confidence interval (shaded area).}
  \label{fig:ks_kriging}
\end{figure}

\begin{figure}[!h]
  \centering
  \subfigure[PhIK]{
  \includegraphics[width=0.40\textwidth]{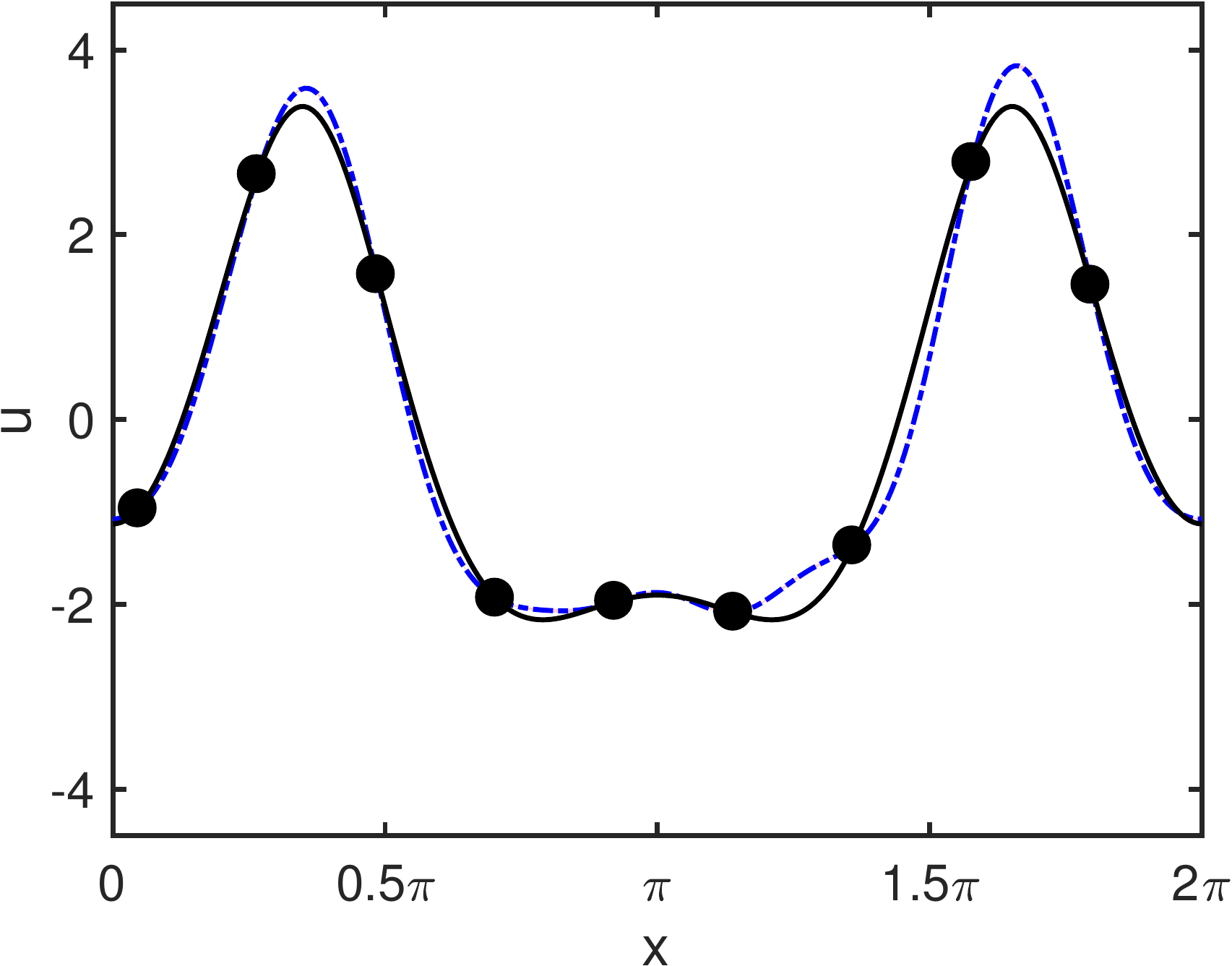}}\qquad
  \subfigure[BiPhIK]{
  \includegraphics[width=0.40\textwidth]{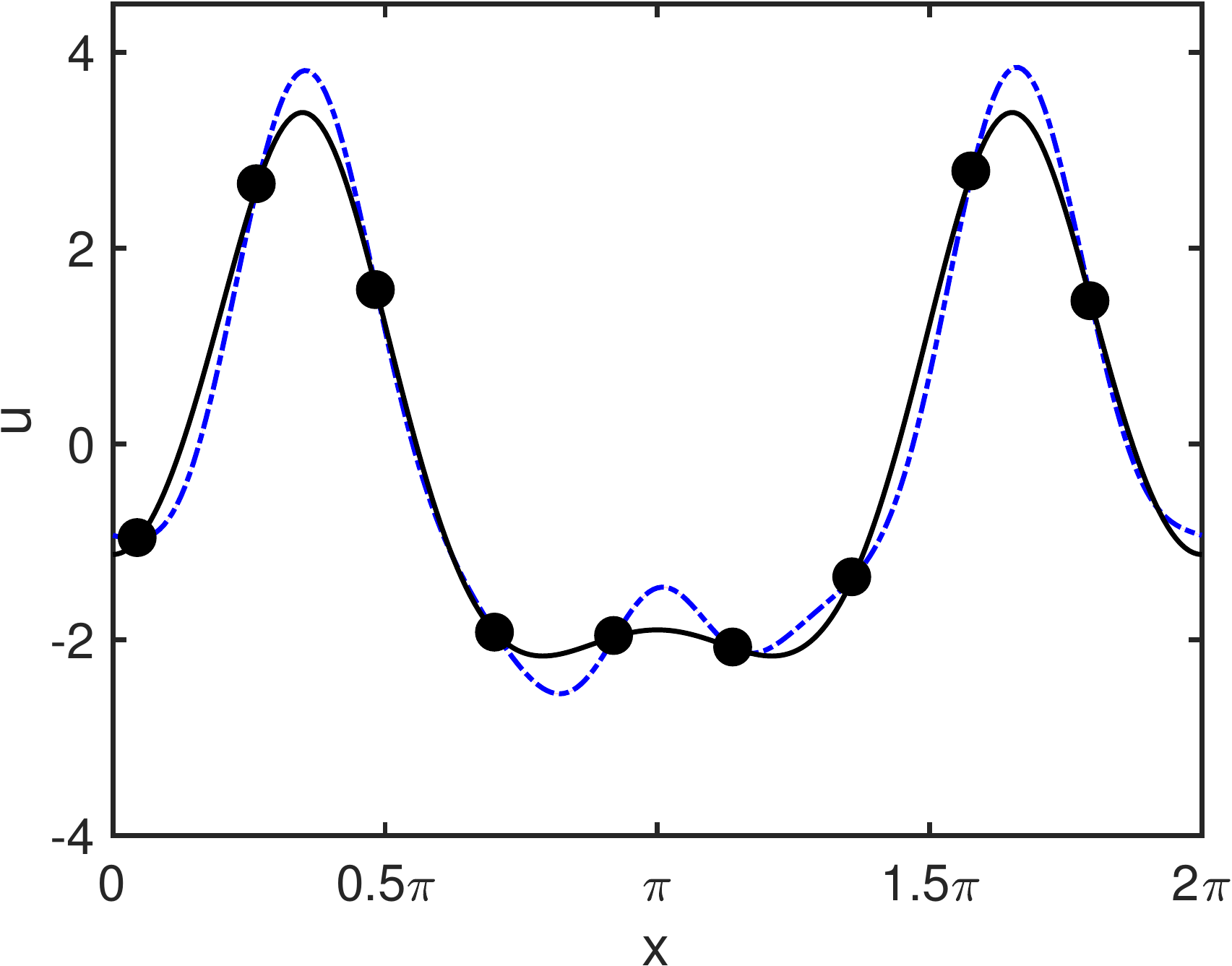}}
  \caption{PhIK (left) and BiPhIK (right) reconstruction of the KS equation:
    posterior mean (dashed line), exact solution (solid line), observations 
    (circles), and $95\%$ confidence interval (very narrow in this case).}
  \label{fig:ks_phik}
\end{figure}

\begin{figure}[!h]
  \centering
  \subfigure[CoPhIK]{
  \includegraphics[width=0.40\textwidth]{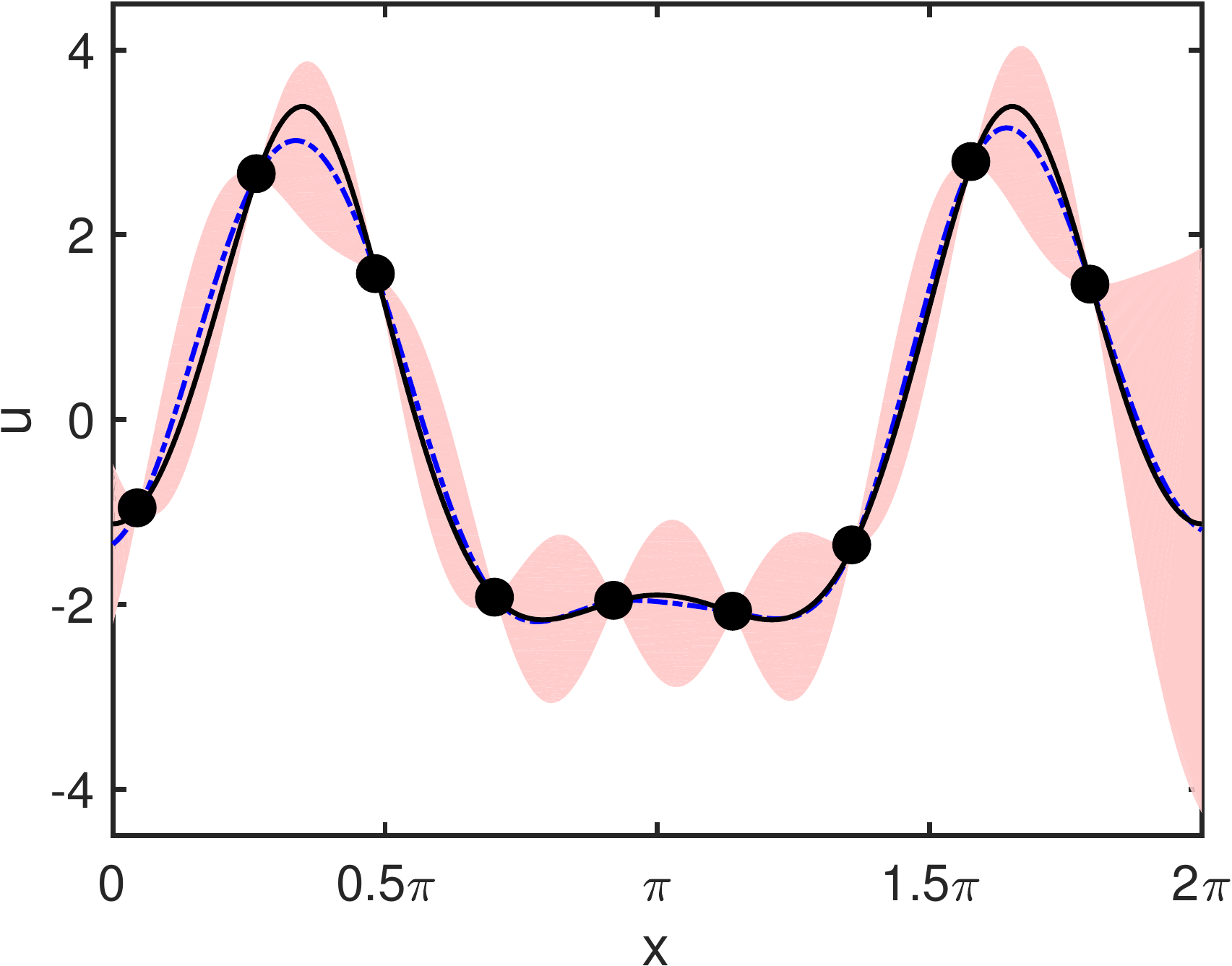}}\qquad
  \subfigure[CoBiPhIK]{
  \includegraphics[width=0.40\textwidth]{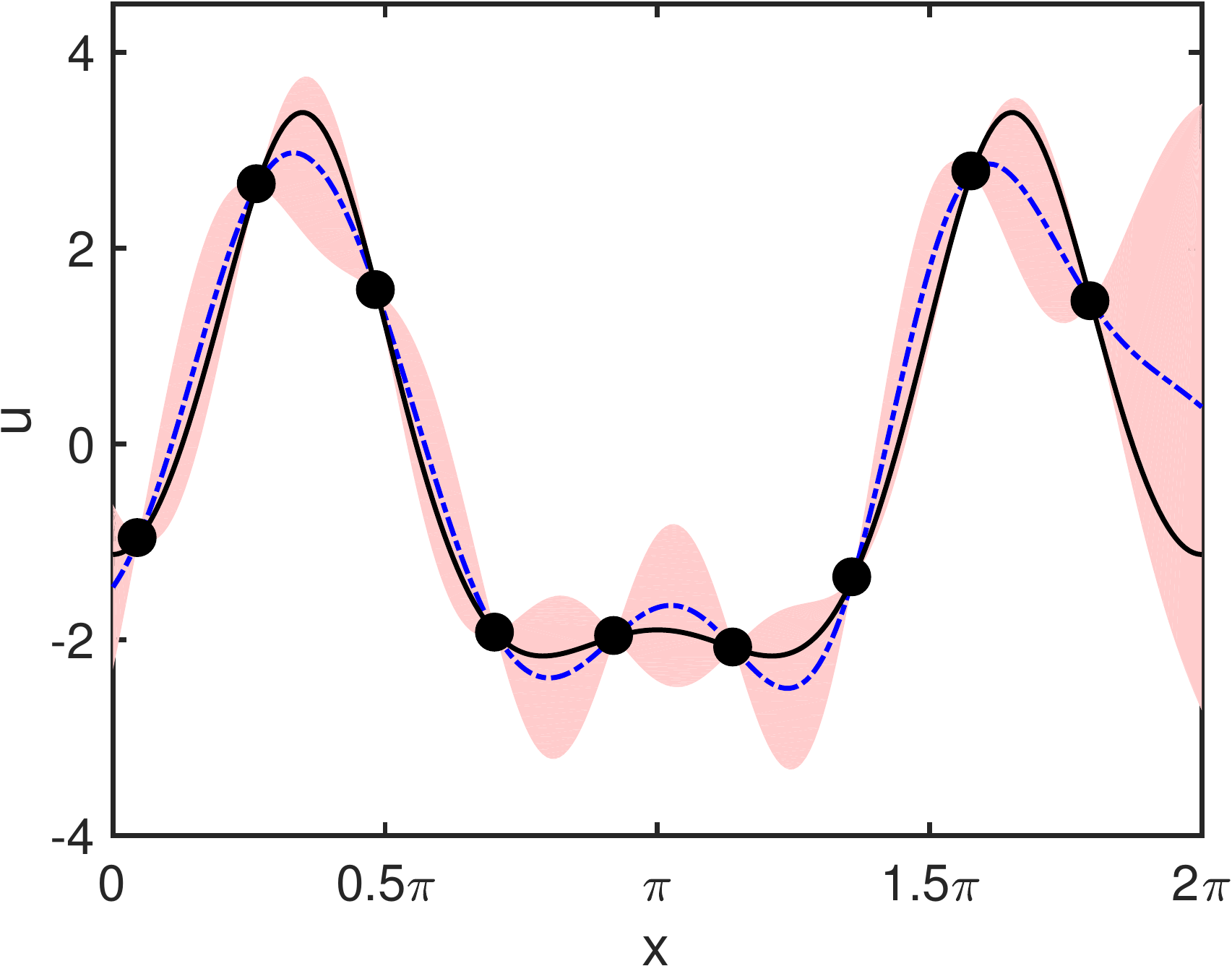}}
  \caption{CoPhIK (left) and BiCoPhIK (right) reconstruction of the KS equation:
    posterior mean (dashed line), exact solution (solid line), observations 
    (circles), and $95\%$ confidence interval (shaded area).}
  \label{fig:ks_phicok}
\end{figure}

Table~\ref{tab:ks} presents the relative $L_2$
errors of all methods. Different from the previous two examples, in this 
example, the structures of space $U_L(\Gamma)$ and $U_H(\Gamma)$ are different
because of the sensitivity of the system to the numerical solution, so
$u_B(\Gamma)$ can not approximate $u_H(\Gamma)$ well. Specifically, in this case
$\delta_1=89.8$ and $\delta_2=7.7$. Consequently, the error of BiPhIK is larger 
than PhIK compared with the other two examples, and the error of CoBiPhIK is 
much larger than CoPhIK.
\begin{table}[!h]
  \centering
  \caption{Relative $L_2$ error (posterior mean vs. reference solution) of different
  methods for reconstructing KS equation solution.}
  \begin{tabular}{*{5}{C{6em}}}
    \hline \hline
    Kriging & PhIK & CoPhIK & BiPhIK & CoBiPhIK \\
    \hline
    0.1581  &  0.1145 & 0.0822 & 0.1439 & 0.2477 \\
    \hline \hline
  \end{tabular}
  \label{tab:ks}
\end{table}

\section{Conclusion}
\label{sec:concl}

In this work, we extend the PhIK/CoPhIK approach by combining two types of 
multifidelity methods in the BiPhIK/CoBiPhIK framework to reduce the 
computational cost of physical models simulations. Specifically, the 
approximation-theory-based bifidelity method is used to generate approximated
high-fidelity realizations of the physical model, which are used to construct GP
in PhIK and CoPhIK. The CoKriging approach utilizes an auxiliary GP to describe
the discrepancy between the model outputs and the sparse accurate observation 
data. We present the error estimate of the difference between the posterior mean
and variance in the resulting GPs by using BiPhIK/CoBiPhIK and PhIK/CoPhIK.
We also analyze the accuracy of preserving linear physical constraints in the
posterior mean of BiPhIK/CoBiPhIK.

The presented methods are nonintrusive, and can utilize existing domain codes 
to compute the necessary realizations. Therefore, these methods are suitable 
for large-scale complex applications for which physical models and codes are 
available. When the parametric dependence of the low-fidelity model
can well inform the structure of the high-fidelity model, the computational cost
can be reduced dramatically.

Our future work would include two directions. One is to use advanced
sampling strategies. For example, instead of MC, the probabilistic collocation 
method is used for the bifidelity method 
in~\cite{narayan2014stochastic, zhu2014computational}, which can further reduce 
the computational cost. The other direction is to directly approximate
high-fidelity mean and covariance without generating approximated high-fidelity
realizations as in~\cite{zhu2017multi}. This approach will save the cost of the
lifting procedure because it only requires solving a much simpler linear system.

\appendix

\section{Adaptive sampling}
\label{sec:append_act}
In this work, 
we use a greedy algorithm to add additional observations, i.e., to add new 
observations at the maxima of $s(\bm x)$, 
e.g.,~\cite{forrester2008engineering, raissi2017machine}. Then, we can make a 
new prediction $\hat y(\bm x)$ for $\bm x\in D$ and compute a new 
$\hat s^2(\bm x)$ to select the next location for additional observation
(see Algorithm~\ref{algo:act}). This selection criterion is based on the 
statistical interpretation of the
interpolation. More sophisticated sensor placement algorithms can be found in
literature, e.g.,~\cite{jones1998efficient, krause2008near, garnett2010bayesian}, 
and PhIK/BiPhIK or CoPhIK/CoBiPhIK are complementary to these methods.
\begin{algorithm}[!h]
  \caption{Active learning based on GPR}
  \label{algo:act}
  \begin{algorithmic}[1]
    \State Specify the locations $\bm X$, corresponding observations $\bm y$, and 
    the maximum number of observations $N_{\max}$ affordable. The number of 
    available observations is denoted as $N$.
    \While {$N_{\max}>N$}
    \State Compute the MSE $\hat s^2(\bm x)$ of prediction $\hat y(\bm x)$ for 
    $\bm x\in D$.
    \State Locate the location $\bm x_m$ for the maximum of $\hat s^2(\bm x)$ for
    $\bm x\in D$. 
    \State Obtain observation $y_m$ at $\bm x_m$, and set 
    $\bm X = \{\bm X, \bm x_m\}, \bm y = (\bm y^\trans, y_m)^\trans, N=N+1$.
    \EndWhile
    \State Construct the prediction of $\hat y(\bm x)$ on $D$ using $\bm X$ 
    and $\bm y$.
  \end{algorithmic}
\end{algorithm}

\section{Proofs of Theorems~\ref{thm:diff_mean} and~\ref{thm:diff_var}}
\label{sec:append_proof}

We present the following three lemmas.
\begin{lem}
  \label{lem:upp1}
For $1\leq n\leq N$,
\begin{equation}
  \begin{aligned}
    \Vert k_H(\bm x, \bm x^{(n)})\Vert 
    \leq \sigma_H(\Gamma)\sigma_H(\bm x^{(n)}), \\
    \Vert k_B(\bm x, \bm x^{(n)})\Vert 
      \leq \sigma_B(\Gamma)\sigma_B(\bm x^{(n)}).
  \end{aligned}
\end{equation}
\end{lem}
\begin{proof}
 According to Eq.~\eqref{eq:phik}, we have
 \[\begin{aligned} 
     \Vert k_H(\bm x, \bm x^{(n)})\Vert  
 = & \left\Vert\dfrac{1}{M-1} \sum_{m=1}^M \left(u^m_H(\bm x)-\mu_H(\bm x)\right) 
     \left(u^m_H(\bm x^{(n)})-\mu_H(\bm x^{(n)})\right) \right\Vert \\
 \leq & \dfrac{1}{M-1} \sum_{m=1}^M \left\Vert u^m_H(\bm x)-\mu_H(\bm x)\right\Vert
     \left\vert u^m_H(\bm x^{(n)})-\mu_H(\bm x^{(n)})\right\vert \\
 \leq & \dfrac{1}{M-1} \left(\sum_{m=1}^M 
        \left\Vert u^m_H(\bm x)-\mu_H(\bm x)\right\Vert^2\right)^{\half}
        \left(\sum_{m=1}^M\left\vert u^m_H(\bm x^{(n)})-\mu_H(\bm x^{(n)})\right\vert^2\right)^{\half} \\
 = &  \left(\dfrac{1}{M-1}\sum_{m=1}^M \left\Vert u^m_H(\bm x)-\mu_H(\bm x)\right\Vert^2\right)^{\half}
      \left(\dfrac{1}{M-1}\sum_{m=1}^M\left\vert u^m_H(\bm x^{(n)})-\mu_H(\bm x^{(n)})\right\vert^2\right)^{\half} \\ 
 = &  \sigma_H(\Gamma)\sigma_H(\bm x^{(n)}).
  \end{aligned}
  \]
Similarly, 
$\Vert k_B(\bm x, \bm x^{(n)})\Vert \leq \sigma_B(\Gamma)\sigma_B(\bm x^{(n)})$.
\end{proof}
\begin{lem}
  \label{lem:upp2}
For $1\leq n\leq N$,
\begin{equation}
  \left\Vert k_H(\bm x, \bm x^{(n)}) - k_B(\bm x, \bm x^{(n)}) \right\Vert \leq 
      \dfrac{2}{M-1} \sum_{m=1}^M
    \left(\delta_1 \left\vert u^m_B(\bm x^{(n)})-\mu_B^m(\bm x^{(n)})\right\vert 
        + \delta_2\left\Vert u_H^m(\bm x)-\mu_H(\bm x) \right\Vert \right).
\end{equation}
\end{lem}
\begin{proof} For any $x\in D$,
\begin{equation}
  \label{eq:diff_mean}
  \left\Vert \mu_H(\bm x)-\mu_B(\bm x)\right\Vert
  =  \left\Vert \dfrac{1}{M}\sum_{m=1}^{M} u_H^m(\bm x)  - 
   \dfrac{1}{M}\sum_{m=1}^{M} u_B^m(\bm x) \right\Vert 
   \leq \dfrac{1}{M}\sum_{m=1}^{M}\left\Vert u_H^m(\bm x) - u_B^m(\bm x)
 \right\Vert\leq \delta_1,
 \end{equation}
 and
\begin{equation}
    \left\vert \mu_H(\bm x)-\mu_B(\bm x)\right\vert
     =  \left\vert \dfrac{1}{M}\sum_{m=1}^{M} u_H^m(\bm x)  - 
    \dfrac{1}{M}\sum_{m=1}^{M} u_B^m(\bm x) \right\vert 
     \leq \dfrac{1}{M}\sum_{m=1}^{M}\left\vert u_H^m(\bm x) - u_B^m(\bm x)
 \right\vert\leq \delta_2.
 \end{equation}
Therefore,
  \begin{equation*}
  \begin{aligned}
    & \left\Vert k_H(\bm x, \bm x^{(n)}) - k_B(\bm x, \bm x^{(n)}) \right\Vert \\
 = & \dfrac{1}{M-1}
\bigg\Vert \sum_{m=1}^M \left(u^m_H(\bm x)-\mu_H(\bm x)\right) 
     \left(u^m_H(\bm x^{(n)})-\mu_H(\bm x^{(n)})\right)  
      - \sum_{m=1}^M \left(u^m_B(\bm x)-\mu_B(\bm x)\right) 
     \left(u^m_B(\bm x^{(n)})-\mu_B(\bm x^{(n)})\right) \bigg\Vert \\
     \leq & \dfrac{1}{M-1} \sum_{m=1}^M\bigg\{
     \bigg\Vert \left(u^m_H(\bm x)-\mu_H(\bm x)\right) 
     \left(u^m_H(\bm x^{(n)})-\mu_H(\bm x^{(n)})\right)   
      - \left(u^m_H(\bm x)-\mu_H(\bm x)\right)
   \left(u^m_B(\bm x^{(n)})-\mu_B(\bm x^{(n)})\right) \bigg\Vert\\
    & \qquad\qquad~~ + \bigg\Vert \left(u^m_H(\bm x)-\mu_H(\bm x)\right) 
     \left(u^m_B(\bm x^{(n)})-\mu_B(\bm x^{(n)})\right)  
      - \left(u^m_B(\bm x)-\mu_B(\bm x)\right)
   \left(u^m_B(\bm x^{(n)})-\mu_B(\bm x^{(n)})\right) \bigg\Vert\bigg\} \\
     \leq & \dfrac{1}{M-1} \sum_{m=1}^M\bigg\{ 
     \left\Vert u^m_H(\bm x)-\mu_H(\bm x)\right\Vert
     \left\vert u^m_H(\bm x^{(n)})-\mu_H(\bm x^{(n)}) -
     u^m_B(\bm x^{(n)})+\mu_B(\bm x^{(n)}) \right\vert \\
  &  \qquad\qquad\quad + \left\vert u^m_B(\bm x^{(n)})-\mu_B(\bm x^{(n)})\right\vert 
      \left\Vert u^m_H(\bm x)-\mu_H(\bm x)-u^m_B(\bm x)+\mu_B(\bm x)\right\Vert \bigg\} \\
     \leq & \dfrac{1}{M-1} \sum_{m=1}^M\bigg\{ 
     2\delta_2\left\Vert u^m_H(\bm x)-\mu_H(\bm x)\right\Vert
   + 2\delta_1\left\vert u^m_B(\bm x^{(n)})-\mu_B(\bm x^{(n)})\right\vert \bigg\}.
  \end{aligned}
  \end{equation*}
\end{proof}
\begin{lem}
  \label{lem:upp3}
  For $1\leq i,j\leq N$,
\begin{equation}
  \left\vert k_H(\bm x^{(i)}, \bm x^{(j)}) - k_B(\bm x^{(i)}, \bm x^{(j)})\right\vert \leq 
  2\delta_2\sqrt{\dfrac{M}{M-1}}\left[\sigma_H(\bm x^{(i)})+\sigma_B(\bm x^{(j)})\right].
\end{equation}
\end{lem}
\begin{proof}
\begin{align*}
    & \left\vert k_H(\bm x^{(i)}, \bm x^{(j)}) - k_B(\bm x^{(i)}, \bm x^{(j)}) \right\vert \\
 = & \dfrac{1}{M-1}
    \bigg\vert \sum_{m=1}^M \left(u^m_H(\bm x^{(i)})-\mu_H(\bm x^{(i)})\right) 
     \left(u^m_H(\bm x^{(j)})-\mu_H(\bm x^{(j)})\right)  \\
     & \qquad\quad - \sum_{m=1}^M \left(u^m_B(\bm x^{(i)})-\mu_B(\bm x^{(i)})\right) 
     \left(u^m_B(\bm x^{(j)})-\mu_B(\bm x^{(j)})\right) \bigg\vert \\
     \leq & \dfrac{1}{M-1} \sum_{m=1}^M\bigg\{
     \bigg\vert \left(u^m_H(\bm x^{(i)})-\mu_H(\bm x^{(i)})\right) 
     \left(u^m_H(\bm x^{(j)})-\mu_H(\bm x^{(j)})\right)  \\ 
     & \qquad\qquad\qquad - \left(u^m_H(\bm x^{(i)})-\mu_H(\bm x^{(i)})\right)
   \left(u^m_B(\bm x^{(j)})-\mu_B(\bm x^{(j)})\right) \bigg\vert\\
   & \qquad\qquad~~ + \bigg\vert \left(u^m_H(\bm x^{(i)})-\mu_H(\bm x^{(i)})\right) 
     \left(u^m_B(\bm x^{(j)})-\mu_B(\bm x^{(j)})\right)  \\ 
     & \qquad\qquad\qquad - \left(u^m_B(\bm x^{(i)})-\mu_B(\bm x^{(i)})\right)
   \left(u^m_B(\bm x^{(j)})-\mu_B(\bm x^{(j)})\right) \bigg\vert\bigg\} \\
     \leq & \dfrac{1}{M-1} \sum_{m=1}^M\bigg\{ 
   \left\vert u^m_H(\bm x^{(i)})-\mu_H(\bm x^{(i)})\right\vert
     \left\vert u^m_H(\bm x^{(j)})-\mu_H(\bm x^{(j)}) -
     u^m_B(\bm x^{(j)})+\mu_B(\bm x^{(j)}) \right\vert \\
  &  \qquad\qquad~~ + \left\vert u^m_B(\bm x^{(j)})-\mu_B(\bm x^{(j)})\right\vert 
   \left\vert u^m_H(\bm x^{(i)})-\mu_H(\bm x^{(i)})-u^m_B(\bm x^{(i)})+\mu_B(\bm x^{(i)})\right\vert \bigg\} \\
     \leq & \dfrac{2\delta_2}{M-1} \sum_{m=1}^M\bigg\{ 
   \left\vert u^m_H(\bm x^{(i)})-\mu_H(\bm x^{(i)})\right\vert
   + \left\vert u^m_B(\bm x^{(j)})-\mu_B(\bm x^{(j)})\right\vert \bigg\} \\
   \leq &\dfrac{2\delta_2\sqrt{M}}{M-1} \left\{ \left(\sum_{m=1}^M\left\vert
 u^m_H(\bm x^{(i)})-\mu_H(\bm x^{(i)})\right\vert^2\right)^{\half}
 + \left(\sum_{m=1}^M\left\vert u^m_B(\bm x^{(j)})-\mu_B(\bm
 x^{(j)})\right\vert^2 \right)^{\half}\right\} \\
 = & 2\delta_2\sqrt{\dfrac{M}{M-1}}\left[\sigma_H(\bm x^{(i)})+\sigma_B(\bm x^{(j)})\right].
\end{align*}
\end{proof}
Now we prove Theorem~\ref{thm:diff_mean} as follows.
\begin{proof}
  Rewriting Eq.~\eqref{eq:krig_pred2} in the functional form, we have
\begin{equation}
  \label{eq:func_phikh}
\hat y_H(\bm x) = \mu_H(\bm x) + \sum_{n=1}^N a_n^H k_H(\bm x, \bm
x^{(n)}),
\end{equation}
  where $a_n^H$ is the $n$-th entry of the vector 
  $\tensor C_H^{-1}(\bm y-\bm\mu_H)$, 
  $(\tensor C_H)_{ij}=k_H(\bm x^{(i)}, \bm x^{(j)})$ and 
  $(\bm\mu_H)_i=\mu_H(\bm x^{(i)})$.  Similarly, we have
\begin{equation}
  \label{eq:func_phikb}
  \hat y_B(\bm x) = \mu_B(\bm x) + \sum_{n=1}^N a_n^B k_B(\bm x, \bm x^{(n)}),
\end{equation}
  where $a_n^B$ is the $n$-th entry of the vector 
  $\tensor C_B^{-1}(\bm y-\bm\mu_B)$, 
  $(\tensor C_B)_{ij}=k_B(\bm x^{(i)}, \bm x^{(j)})$ and 
  $(\bm\mu_B)_i=\mu_B(\bm x^{(i)})$. In Eq.~\eqref{eq:diff_mean}, we show that
  \[\left\Vert \mu_H(\bm x)-\mu_B(\bm x)\right\Vert \leq \delta_1.\]
  Next,
 \begin{equation*}
   \begin{aligned}
     &  \left\Vert a^H_n k_H(\bm x, \bm x^{(n)})
            - a^B_n k_B(\bm x, \bm x^{(n)})\right\Vert \\ 
   \leq &
   \left\Vert a^H_n k_H(\bm x, \bm x^{(n)})
            - a^B_n k_H(\bm x, \bm x^{(n)})\right\Vert + 
   \left\Vert a^B_n k_H(\bm x, \bm x^{(n)})
            - a^B_n k_B(\bm x, \bm x^{(n)})\right\Vert  \\
   \leq & \left\vert a^H_n-a^B_n \right\vert \left\Vert k_H(\bm x, \bm x^{(n)}) \right\Vert 
        + \left\vert a^B_n\right\vert 
        \left\Vert k_H(\bm x, \bm x^{(n)})-k_B(\bm x, \bm x^{(n)})\right\Vert \\
   \leq & \left\vert a^H_n-a^B_n \right\vert \sigma_H(\Gamma)\sigma_H(\bm x^{(n)}) 
         + \dfrac{2\left\vert a^B_n\right\vert}{M-1} \sum_{m=1}^M
    \left(\delta_1 \left\vert u^m_B(\bm x^{(n)})-\mu_B^m(\bm x^{(n)})\right\vert 
        + \delta_2\left\Vert u_H^m(\bm x)-\mu_H(\bm x) \right\Vert \right),
          \end{aligned}
 \end{equation*}
where the last inequality utilizes Lemmas~\ref{lem:upp1} and \ref{lem:upp2}.
Therefore,
\begin{equation*}
\begin{aligned}
  &  \left\Vert \sum_{n=1}^N a^H_n k_H(\bm x, \bm x^{(n)})
            - \sum_{n=1}^Na^B_n k_B(\bm x, \bm x^{(n)})\right\Vert \\ 
 \leq &\underbrace{\sum_{n=1}^N \left\vert a^H_n-a^B_n \right\vert 
                   \sigma_H(\Gamma) \sigma_H(\bm x^{(n)})}_{J_1} 
     + \underbrace{\dfrac{2\delta_1}{M-1}\sum_{n=1}^N \left\vert a_n^B\right\vert
       \sum_{m=1}^M \left\vert u^m_B(\bm x^{(n)})-\mu_B^m(\bm x^{(n)})\right\vert}_{J_2}\\
 & + \underbrace{\dfrac{2\delta_2}{M-1}\sum_{n=1}^N\left\vert a_n^B\right\vert
     \sum_{m=1}^M \left\Vert u^m_H(\bm x)-\mu_H^m(\bm x)\right\Vert}_{J_3}.
\end{aligned}
\end{equation*}
Here,
\[\begin{aligned}
  J_1 & =\sigma_H(\Gamma)\sum_{n=1}^N \left\vert a^H_n-a^B_n \right\vert \sigma_H(\bm x^{(n)})\\
      & \leq \sigma_H(\Gamma)\left(\sum_{n=1}^N\sigma^2_H(\bm x^{(n)})) \right)^{\half}
    \left(\sum_{n=1}^N \left\vert a^H_n-a^B_n \right\vert^2\right)^{\half} \\ 
   & = S_H\sigma_H(\Gamma)
\left\Vert\tensor C_H^{-1}(\bm y-\bm\mu_H) - \tensor C_B^{-1}(\bm y-\bm\mu_B) \right\Vert_2 \\
   & = S_H\sigma_H(\Gamma)
\left\Vert\tensor C_H^{-1}(\bm y-\bm\mu_H) - \tensor C_H^{-1}(\bm y-\bm\mu_B) + \tensor C_H^{-1}(\bm y-\bm\mu_B) - \tensor C_B^{-1}(\bm y-\bm\mu_B) \right\Vert_2 \\
      & \leq S_H\sigma_H(\Gamma)
\left( 
  \left\Vert\tensor C_H^{-1}\right\Vert_2\left\Vert \bm\mu_H-\bm\mu_B \right\Vert_2
+\left\Vert \tensor C_H^{-1} - \tensor C_B^{-1} \right\Vert_2\left\Vert \bm y-\bm\mu_B \right\Vert_2 \right) \\
      & \leq S_H\sigma_H(\Gamma)
  \left\{\left\Vert\tensor C_H^{-1}\right\Vert_2\sqrt{N}\delta_2 + \left\Vert\tensor C_H^{-1}\right\Vert_2^2\left\Vert\tensor C_H-\tensor C_B\right\Vert_2 
\left\Vert \bm y-\bm\mu_B \right\Vert_2 \right\} \\
      & \leq S_H\sigma_H(\Gamma)
  \left\{\left\Vert\tensor C_H^{-1}\right\Vert_2\sqrt{N}\delta_2 + \left\Vert\tensor C_H^{-1}\right\Vert_2^2\left\Vert\tensor C_H-\tensor C_B\right\Vert_F 
\left\Vert \bm y-\bm\mu_B \right\Vert_2 \right\}, 
\end{aligned} \]
where we uses the well-known matrix perturbation conclusion (e.g.,~\cite{demmel1992componentwise}):
\begin{equation}
  \label{eq:mat_perturb}
  \left\Vert (\tensor A+\Delta\tensor A)^{-1} - \tensor A^{-1} \right\Vert\leq
\left\Vert \tensor A^{-1}\right\Vert^2 \Vert \Delta\tensor A\Vert
\end{equation}
for invertible matrices $\tensor A$ and $\tensor A+\Delta\tensor A$, and a
well-defined matrix norm $\Vert\cdot\Vert$. Further, using Lemma~\ref{lem:upp3},
we have
\begin{equation}
  \label{eq:mat_perturb2}
  \begin{aligned}
    \left\Vert\tensor C_H-\tensor C_B \right\Vert^2_F = &\sum_{i=1}^N\sum_{j=1}^N 
  \left\vert k_H(\bm x^{(i)}, \bm x^{(j)})-k_B(\bm x^{(i)}, \bm x^{(j)})\right\vert^2 \\
  \leq &\sum_{i=1}^N\sum_{j=1}^N 
  \dfrac{4\delta_2^2 M}{M-1}\left[\sigma(u_H^m(\bm x^{(i)}))+\sigma(u_B^m(\bm x^{(j)}))\right]^2 \\
  \leq &\dfrac{8\delta_2^2 M}{M-1}\sum_{i=1}^N\sum_{j=1}^N 
  \left[\sigma^2(u_H^m(\bm x^{(i)}))+\sigma^2(u_B^m(\bm x^{(j)}))\right] \\
  = & \dfrac{8\delta_2^2 MN}{M-1}\sum_{n=1}^N 
  \left[\sigma^2(u_H^m(\bm x^{(n)}))+\sigma^2(u_B^m(\bm x^{(n)}))\right] \\
  = & \dfrac{8\delta_2^2 MN}{M-1}(S^2_H+S^2_B)
\end{aligned} 
\end{equation}
which yields
\begin{equation*}
J_1 \leq \sqrt{N}S_H\delta_2\sigma_H(\Gamma)\left\Vert\tensor C_H^{-1}\right\Vert_2
    \left\{ 2\sqrt{\dfrac{2 M}{M-1}}\left(S^2_H+S^2_B \right)^{\half} 
    \left\Vert\tensor C_H^{-1}\right\Vert_2\Vert\bm y-\bm\mu_B\Vert_2 +1\right\}.
\end{equation*}
Also,
\begin{equation*}
\begin{aligned}
  J_2\leq & 2\delta_1\sqrt{\dfrac{M}{M-1}}\sum_{n=1}^N\left\vert a_n^B\right\vert
  \sigma(u_B^m(\bm x^{(n)})) \\
  \leq & 2\delta_1\sqrt{\dfrac{MN}{M-1}}\left\Vert\tensor C_B^{-1}(\bm y-\bm\mu_B) \right\Vert_2
  \left(\sum_{n=1}^N \sigma^2(u_B^m(\bm x^{(n)})) \right)^{\half} \\
  \leq & 2\delta_1\sqrt{\dfrac{MN}{M-1}}\left\Vert\tensor C_B^{-1}\right\Vert_2
  \left\Vert \bm y-\bm\mu_B \right\Vert_2 S_B
\end{aligned}
\end{equation*}
Similarly,
\begin{equation*}
\begin{aligned}
  J_3\leq & 2\delta_2\sqrt{\dfrac{M}{M-1}}\sum_{n=1}^N\left\vert a_n^B\right\vert \sigma_H(\Gamma) \\
  \leq & 2\delta_2\sqrt{\dfrac{MN}{M-1}}\sigma_H(\Gamma)
  \left\Vert\tensor C_B^{-1}(\bm y-\bm\mu_B) \right\Vert_2 \\
\leq & 2\delta_2\sqrt{\dfrac{MN}{M-1}}\sigma_H(\Gamma)\left\Vert\tensor C_B^{-1}\right\Vert_2
  \left\Vert \bm y-\bm\mu_B \right\Vert_2.
\end{aligned}
\end{equation*}
Therefore, the conclusion holds.
\end{proof}
Next, we present the proof of Theorem~\ref{thm:diff_var}.
\begin{proof}
For any $\bm x^*\in D$, we use the following concise notations
\begin{equation}
\begin{aligned}
\bm c_H & =
\left(k_H(\bm x^{(1)},\bm x^*), \dotsc,k_H(\bm x^{(N)},\bm x^*)\right)^\trans, \\
\bm c_B & =
\left(k_B(\bm x^{(1)},\bm x^*), \dotsc,k_B(\bm x^{(N)},\bm x^*)\right)^\trans.
\end{aligned}
\end{equation}
Following the same procedure in the proof of Lemma~\ref{lem:upp3}, we have
\begin{equation*}
\left\vert k_H(\bm x^*, \bm x^*) - k_B(\bm x^*, \bm x^*) \right\vert 
  \leq 2\delta_2\sqrt{\dfrac{M}{M-1}} 
    \left[\sigma(u_H^m(\bm x^*))+\sigma(u_B^m(\bm x^*))\right] 
  \leq 2\delta_2\sqrt{\dfrac{2M}{M-1}} 
    (\Delta_H^2+\Delta_B^2)^{\half}.
\end{equation*}
According to Lemma~\ref{lem:upp3}, 
\begin{equation*}
  \begin{aligned}
  \Vert\bm c_H-\bm c_B \Vert_2 = & \left(\sum_{n=1}^N 
   \left\vert k_H(\bm x^{(n)}, \bm x^*) - k_B(\bm x^{(n)}, \bm
 x^*)\right\vert^2\right)^{\half} \\
 \leq & 2\delta_2\sqrt{\dfrac{M}{M-1}}\left(\sum_{n=1}^N 
\left[\sigma(u_H^m(\bm x^{(n)}))+\sigma(u_B^m(\bm x^*))\right]^2\right)^{\half}\\
\leq & 2\delta_2\sqrt{\dfrac{2M}{M-1}}\left(\sum_{n=1}^N 
\left[\sigma^2(u_H^m(\bm x^{(n)}))+\sigma^2(u_B^m(\bm x^*))\right]\right)^{\half} \\
= &2\delta_2\sqrt{\dfrac{2M}{M-1}} \left(S_H^2+S_B^2\right)^{\half}.
  \end{aligned}
\end{equation*}
Also,
\begin{equation}
  \begin{aligned}
    \Vert\bm c_H \Vert_2 = & \left\{\sum_{n=1}^N 
  \left\vert k_H(\bm x^{(n)}, \bm x^*)\right\vert^2\right\}^{\half} \\
  \leq & \Bigg\{\sum_{n=1}^N \left\vert\dfrac{1}{M-1}\sum_{m=1}^M
  \left(u^m_H(\bm x^{(n)})-\mu_H(\bm x^{(n)})\right) 
\left(u^m_H(\bm x^*)-\mu_H(\bm x^*)\right)\right\vert^2 \Bigg\}^{\half}\\
\leq & \Bigg\{\sum_{n=1}^N \left(\dfrac{1}{M-1}
   \sum_{m=1}^M \left(u^m_H(\bm x^{(n)})-\mu_H(\bm x^{(n)})\right)^2\right)\cdot
    \left(\dfrac{1}{M-1} \sum_{m=1}^M \left(u^m_H(\bm x^*)-\mu_H(\bm
 x^*)\right)^2\right)\Bigg\}^{\half} \\
 = &\sigma_H(\bm x^*) S_H.
  \end{aligned}
\end{equation}
Similarly,
\begin{equation}
  \Vert\bm c_B\Vert_2 \leq \sigma_B(\bm x^*) S_B.
\end{equation}
Thus,
\begin{equation}
\begin{aligned}
  & \left\vert \bm c_H^\trans\tensor C_H^{-1}\bm c_H 
- \bm c_B^\trans\tensor C_B^{-1}\bm c_B\right\vert \\
\leq & \left\vert \bm c_H^\trans\tensor C_H^{-1}\bm c_H 
- \bm c_H^\trans\tensor C_H^{-1}\bm c_B \right\vert+
\left\vert \bm c_H^\trans\tensor C_H^{-1}\bm c_B 
- \bm c_H^\trans\tensor C_B^{-1}\bm c_B \right\vert+
\left\vert \bm c_H^\trans\tensor C_B^{-1}\bm c_B 
- \bm c_B^\trans\tensor C_B^{-1}\bm c_B \right\vert \\
\leq & \underbrace{\Vert \bm c_H \Vert_2\Vert\tensor C_H^{-1}\Vert_2
                   \Vert\bm c_H-\bm c_B\Vert_2}_{J_1}
+ \underbrace{\Vert \bm c_H \Vert_2\Vert\tensor C_H^{-1}-\tensor C_B^{-1}\Vert_F
              \Vert\bm c_B\Vert_2}_{J_2}  
 + \underbrace{\Vert \bm c_B \Vert_2\Vert\tensor C_B^{-1}\Vert_2
                \Vert\bm c_H-\bm c_B\Vert_2}_{J_3},
  \end{aligned}
\end{equation}
and
\[\begin{aligned}
    J_1 & \leq 2\delta_2\sqrt{\dfrac{2M}{M-1}}\sigma_H(\bm x^*)S_H(S_B^2+S_H^2)^{\half}
  \Vert\tensor C_H^{-1}\Vert_2, \\
  J_2 & \leq 2\delta_2 \sqrt{\dfrac{2MN}{M-1}}\sigma_H(\bm x^*)\sigma_B(\bm x^*)
  S_HS_B(S_H^2+S_B^2)^{\half}\Vert\tensor C_H^{-1}\Vert^2_2, \\
  J_3 &\leq 2\delta_2\sqrt{\dfrac{2M}{M-1}}
\sigma_B(\bm x^*) S_B(S_B^2+S_H^2)^{\half}\Vert\tensor C_B^{-1}\Vert_2,
\end{aligned}\]
where in the inequality of $J_2$ we use Eqs.~\eqref{eq:mat_perturb} and
\eqref{eq:mat_perturb2}. Finally, using the fact $\sigma_H(\bm x^*)\leq\Delta_H$
and $\sigma_B(\bm x^*)\leq\Delta_B$, we have 
\begin{equation}
  S_H \leq \sqrt{N}\Delta_H,\quad S_B\leq \sqrt{N}\Delta_B,
\end{equation}
which indicates
\[\begin{aligned}
 J_1 & \leq 2N\delta_2\sqrt{\dfrac{2M}{M-1}}\Delta^2_H(\Delta_B^2+\Delta_H^2)^{\half}
  \Vert\tensor C_H^{-1}\Vert_2, \\
J_2 & \leq 2N\delta_2 \sqrt{\dfrac{2MN}{M-1}}\Delta^2_H\Delta^2_B
  (\Delta_H^2+\Delta_B^2)^{\half}\Vert\tensor C_H^{-1}\Vert^2_2, \\
J_3 &\leq 2N\delta_2\sqrt{\dfrac{2M}{M-1}}
\Delta^2_B (\Delta_B^2+\Delta_H^2)^{\half}\Vert\tensor C_B^{-1}\Vert_2.
\end{aligned}\]
Therefore, the conclusion holds.
\end{proof}

\section*{Acknowledgments}
Xueyu Zhu's work was supported by Simons Foundation.
Xiu Yang and Jing Li's work was supported by the U.S. Department of Energy 
(DOE), Office of Science, Office of Advanced Scientific Computing Research 
(ASCR) as part of Uncertainty Quantification in Advection-Diffusion-Reaction 
Systems. Pacific Northwest National Laboratory is operated by Battelle for the
DOE under Contract DE-AC05-76RL01830.

\bibliographystyle{plain}
\bibliography{ref}
\end{document}